%% file: main.tex
\documentclass{article}

\usepackage[numbers,sort&compress]{natbib}
\usepackage[preprint]{neurips_2026}
\usepackage{geometry}
\usepackage{tcolorbox}
\usepackage{upgreek}
\usepackage{mathrsfs}
\input{header.tex}

\input{def}

\title{Uniform Diffusion Models Revisited: Leave-One-Out Denoiser and Absorbing State Reformulation}

\author{Samson Gourevitch$^{*, 1}$
\quad
Yazid Janati$^{*, 2, 4}$ 
\quad 
Dario Shariatian$^3$\\
\textbf{
Umut Simsekli$^3$ \quad Eric Moulines$^{4, 5}$ \quad Eric P. Xing$^{2,4}$ 
\quad 
Alain Durmus$^1$} \\
$^1$ CMAP, Ecole polytechnique \quad $^2$ Institute of Foundation Models\\ $^3$ Inria, PSL Research University \quad $^4$ MBZUAI \quad $^5$ EPITA, LRE}
\date{}

\begin{document}

\maketitle

\begin{abstract}

Discrete diffusion models are often trained through clean-data prediction, but the prediction can be used in different ways to define the reverse dynamics. In Masked Diffusion Models (MDM) these choices largely coincide, whereas in Uniform Diffusion Models (UDM) they do not. We show that the standard plug-in bridge parameterization for UDM is not optimized by the denoising posterior, but by a \emph{leave-one-out} posterior that predicts each clean token without using its own noisy observation. This identifies a mismatch between the plug-in ELBO and the usual cross-entropy denoising objective.
We characterize the leave-one-out target and derive exact conversions between the denoiser, the leave-one-out posterior, and the score. These conversions allow us to disentangle parameterization and training objective. Our results also lead to inference improvements without any additional training through an informed predictor-corrector sampler and improved temperature sampling based on the leave-one-out predictor.
 We further introduce an absorbing-state reformulation of uniform diffusion that preserves the UDM joint law while decomposing it into masked-diffusion-like sampling operations, with simpler denoising posteriors, carry-over unmasking, and a natural remasking mechanism. On language modeling, leave-one-out parameterizations consistently improve UDM generation, while the absorbing construction matches or surpasses masked diffusion. These results suggest that the empirical gap between masked and uniform diffusion is driven less by the choice of marginals themselves than by parameterization and sampling design. The code and models can be found at \url{https://github.com/samsongourevitch/rev_udm}.

\end{abstract}

\blfootnote{* Equal contribution}
\blfootnote{Corresponding authors: \texttt{samson.gourevitch@polytechnique.edu}, \quad \texttt{yazid.janati@mbzuai.ac.ae} \quad and \\
\texttt{dario.shariatian@inria.fr}}
\input{body/introduction.tex}
\input{body/background.tex}
\input{body/method.tex}
\input{body/absorbing_uniform.tex}
\input{body/related_works.tex}
\input{body/experiments.tex}

\input{body/conclusion.tex}
\bibliographystyle{plainnat}
\bibliography{bibliography}

\newpage
\appendix
\crefalias{section}{appendix}
\crefalias{subsection}{appendix}
\crefalias{subsubsection}{appendix}
\section*{Appendix Outline}
The appendix is organized as follows. \Cref{appendix:loo} proves the leave-one-out optimality result, gives the conversion formulas between the leave-one-out denoiser, denoiser and score, and derives the Gibbs conditional used by the predictor-corrector sampler. \Cref{apdx:audm} gives the full derivation of AUDM and ReAUDM, including the continuous-time AUDM objective and the lifted resampling construction. \Cref{app:mudm} gives the masked-diffusion lifting of UDM and proves \Cref{prop:mudm-joint}. \Cref{app:bridge-extension} discusses the bridge extension underlying the plug-in parameterization. \Cref{app:predictor-corrector} details the predictor-corrector sampler. In \Cref{sec:ctmc} we provide the continuous-time CTMC perspective of discrete diffusion models. In \Cref{app:max-coupling} we detail Maximal Coupling Uniform Diffusion and its ELBO. \Cref{app:loo-sensitivity} discusses the leave-one-out invariance diagnostic. \Cref{app:exp-details} gives experimental details and additional results.
\input{appendix/loo_denoiser.tex}

\input{appendix/audm.tex}
\input{appendix/mudm.tex}
\input{appendix/bridge_linearization.tex}
\input{appendix/predictor_corrector.tex}
\input{appendix/ctmc.tex}
\input{appendix/max_coupling.tex}

\input{appendix/loo_sensitivity.tex}
\input{appendix/experiments.tex}

\end{document}

%% file: header.tex
\usepackage{ifthen}
\usepackage{xargs}
\usepackage{amsmath}
\usepackage{bm}
\usepackage{amssymb}
\usepackage{dsfont}
\usepackage{mathtools}
\usepackage{amsthm}
\usepackage{booktabs}
\usepackage{colortbl}
\usepackage{comment}
\usepackage{stmaryrd}
\usepackage[colorlinks=true,linkcolor=blue,citecolor=blue,urlcolor=blue]{hyperref}
\usepackage{todonotes}
\usepackage{wrapfig}
\usepackage{caption}
\usepackage{subcaption}
\usepackage{float}
\usepackage{algorithm}
\usepackage{algpseudocode}
\usepackage{enumitem}
\definecolor{crimson}{RGB}{220,20,60}

\newcommand\blfootnote[1]{%
  \begingroup
  \renewcommand\thefootnote{}\footnote{#1}%
  \addtocounter{footnote}{-1}%
  \endgroup
}
\usepackage[framemethod=TikZ]{mdframed}
\mdfdefinestyle{propFrame}{%
    linecolor=black!8!white,
    outerlinewidth=.6pt,
    roundcorner=5pt,
    innertopmargin=7pt,
    innerbottommargin=7pt,
    innerrightmargin=8pt,
    innerleftmargin=8pt,
    backgroundcolor=black!3!white
}

\mdfsetup{
  skipabove=.2cm,
  skipbelow=0pt
}

\usepackage[capitalize,noabbrev]{cleveref}

\theoremstyle{plain}

\newtheorem{proposition}{Proposition}
\newtheorem{lemma}{Lemma}

\theoremstyle{definition}

\theoremstyle{plain}
\newtheorem{remark}{Remark}

\crefname{appendix}{appendix}{appendices}
\Crefname{appendix}{Appendix}{Appendices}
\crefname{algorithm}{algorithm}{algorithms}
\Crefname{algorithm}{Algorithm}{Algorithms}
\algrenewcommand\algorithmicrequire{\textbf{Require:}}
\algrenewcommand\algorithmicensure{\textbf{Ensure:}}

\newcounter{hypA}
\newenvironment{hypA}{\refstepcounter{hypA}\begin{itemize}
  \item[({\bf A\arabic{hypA}})]}{\end{itemize}}       

%% file: def.tex
\def\vocab{K}
\def\loo{\mathrm{loo}}
\def\audm{\mathrm{AUDM}}

\def\maxcoupling{\mathrm{MCUD}}
\def\len{L}
\def\noise{\mathbf{u}}
\def\msv{\mathsf{V}}

\def\msx{\mathsf{X}}

\def\rset{\mathbb{R}}

\def\rmd{\mathrm{d}}

\def\eqsp{\,}

\def\one{\mathbf{1}}

\def\pE{\mathbb{E}}
\def\softmax{\mathrm{softmax}}
\def\pP{\mathbb{P}}

\def\indic{\mathds{1}}

\def\eqdef{\vcentcolon=}

\def\wrt{w.r.t.}

\def\law{\text{Law}}

\def\last{n}
\newcommand{\tk}[1]{{t_{#1}}}
\newcommand{\sk}[1]{{s_{#1}}}
\newcommand{\onehot}[1]{e_{#1}}

\def\mask{\mathsf{m}}

\newcommand{\dotp}[2]{\langle #1, #2 \rangle}

\newcommand{\assp}[1]{(\textbf{A}\textbf{#1})}

\def\txts{\textstyle}

\newcommand{\intset}[2]{\llbracket #1, #2 \rrbracket}
\newcommand{\kldivergence}[2]{\mathrm{KL}(#1\|#2)}

 \def\rset{{\mathbb{R}}}

 \def\rmd{\mathrm{d}}
 \def\eqsp{\;}
 \def\eqdef{\coloneqq}

\newcommand{\tbar}{\mathrel{\raisebox{0.15ex}{$\scriptscriptstyle\mid$}}}

\newcommandx\vartarg[4][4=]{
    \ifthenelse{\equal{#2}{}}{
        \ifthenelse{\equal{#3}{}}{
            \varpi^{#4} _{\!#1}
        }{
            \varpi^{#4} _{\!#1}(#3)
            }
    }{
        \ifthenelse{\equal{#3}{}}{
            \varpi^{#4} _{\!#1}(\cdot|#2)
        }{
            \varpi^{#4} _{\!#1}(#3|#2)
        }
    }}

\newcommandx\targ[4][4=]{
    \ifthenelse{\equal{#2}{}}{
        \ifthenelse{\equal{#3}{}}{
            \pi^{#4} _{\!#1}
        }{
            \pi^{#4} _{\!#1}(#3)
            }
    }{
        \ifthenelse{\equal{#3}{}}{
            \pi^{#4} _{\!#1}(\cdot|#2)
        }{
            \pi^{#4} _{\!#1}(#3|#2)
        }
    }}

\newcommandx\interp[4][4=\interpscale]{
    \ifthenelse{\equal{#2}{}}{
        \ifthenelse{\equal{#3}{}}{
            \pi^{#4} _{#1}
        }{
            \pi^{#4} _{#1}(#3)
            }
    }{
        \ifthenelse{\equal{#3}{}}{
            \pi^{#4} _{#1}(\cdot|#2)
        }{
            \pi^{#4} _{#1}(#3|#2)
        }
    }}

\newcommandx\dpdata[4][4=]{
\ifthenelse{\equal{#2}{}}{
    \ifthenelse{\equal{#3}{}}{
        p^{d, #4} _{#1}
    }{
        p^{d, #4} _{#1}(#3)
    }
}{
    \ifthenelse{\equal{#3}{}}{
        p^{d, #4} _{#1}(\cdot|#2)
    }{
        p^{d, #4} _{#1}(#3|#2)
    }
}}

\newcommandx\hpdata[4][4=]{
\ifthenelse{\equal{#2}{}}{
    \ifthenelse{\equal{#3}{}}{
        \hat{p}^{#4} _{#1}
    }{
        \hat{p}^{#4} _{#1}(#3)
    }
}{
    \ifthenelse{\equal{#3}{}}{
        \hat{p}^{#4} _{#1}(\cdot|#2)
    }{
        \hat{p}^{#4} _{#1}(#3|#2)
    }
}}
\newcommandx\pdata[4][4=]{
    \ifthenelse{\equal{#2}{}}{
        \ifthenelse{\equal{#3}{}}{
            p^{#4} _{#1}
        }{
            p^{#4} _{#1}(#3)
        }
    }{
        \ifthenelse{\equal{#3}{}}{
            p^{#4} _{#1}(\cdot|#2 )
        }{
            p^{#4} _{#1}(#3|#2)
        }
    }}

\newcommandx\barpdata[4][4=]{
    \ifthenelse{\equal{#2}{}}{
        \ifthenelse{\equal{#3}{}}{
            \bar{p}^{#4} _{#1}
        }{
            \bar{p}^{#4} _{#1}(#3)
        }
    }{
        \ifthenelse{\equal{#3}{}}{
            \bar{p}^{#4} _{#1}(\cdot|#2 )
        }{
            \bar{p}^{#4} _{#1}(#3|#2)
        }
    }}

\newcommandx\pnoise[4][4=]{
\ifthenelse{\equal{#2}{}}{
    \ifthenelse{\equal{#3}{}}{
        \upsilon^{#4} _{#1}
    }{
        \upsilon^{#4} _{#1}(#3)
    }
}{
    \ifthenelse{\equal{#3}{}}{
        \upsilon^{#4} _{#1}(\cdot|#2 )
    }{
        \upsilon^{#4} _{#1}(#3|#2)
    }
}}

\newcommandx\ipdata[4][4=]{
\ifthenelse{\equal{#2}{}}{
    \ifthenelse{\equal{#3}{}}{
        \bar{p}^{#4} _{#1}
    }{
        \bar{p}^{#4} _{#1}(#3)
    }
}{
    \ifthenelse{\equal{#3}{}}{
        \bar{p}^{#4} _{#1}(\cdot|#2 )
    }{
        \bar{p}^{#4} _{#1}(#3|#2)
    }
}}

\newcommandx\ratemat[4][4=]{
\ifthenelse{\equal{#2}{}}{
    \ifthenelse{\equal{#3}{}}{
        R^{#4} _{#1}
    }{
        \lambda^{#4} _{#1}(#3)
    }
}{
    \ifthenelse{\equal{#3}{}}{
        R^{#4} _{#1}(\cdot | #2)
    }{
        R^{#4} _{#1}(#3 | #2)
    }
}}

\newcommandx\rratemat[4][4=]{
\ifthenelse{\equal{#2}{}}{
    \ifthenelse{\equal{#3}{}}{
        \bar R^{#4} _{#1}
    }{
        \bar\lambda^{#4} _{#1}(#3)
    }
}{
    \ifthenelse{\equal{#3}{}}{
        \bar R^{#4} _{#1}(\cdot | #2)
    }{
        \bar R^{#4} _{#1}(#3 | #2)
    }
}}

\newcommandx\mdm[4][4=]{
\ifthenelse{\equal{#2}{}}{
    \ifthenelse{\equal{#3}{}}{
        p^{\mathsf{d}, #4} _{#1}
    }{
        p^{\mathsf{d}, #4} _{#1}(#3)
    }
}{
    \ifthenelse{\equal{#3}{}}{
        p^{\mathsf{d}, #4} _{#1}(\cdot|#2)
    }{
        p^{\mathsf{d}, #4} _{#1}(#3|#2)
    }
}}

\newcommandx\jpdata[4][4=]{
\ifthenelse{\equal{#2}{}}{
    \ifthenelse{\equal{#3}{}}{
        \tilde{p}^{#4} _{#1}
    }{
        \tilde{p}^{#4} _{#1}(#3)
    }
}{
    \ifthenelse{\equal{#3}{}}{
        \tilde{p}^{#4} _{#1}(\cdot|#2)
    }{
        \tilde{p}^{#4} _{#1}(#3|#2)
    }
}}

\newcommandx\cpdata[4][4=]{
    \ifthenelse{\equal{#2}{}}{
        \ifthenelse{\equal{#3}{}}{
            \pi^{#4} _{#1}
        }{
            \pi^{#4} _{#1}(#3)
        }
    }{
        \ifthenelse{\equal{#3}{}}{
            \pi^{#4} _{#1}(\cdot|#2)
        }{
            \pi^{#4} _{#1}(#3|#2)
        }
    }}

  \newcommandx\potg[4][4=]{
    \ifthenelse{\equal{#2}{}}{
        \ifthenelse{\equal{#3}{}}{
            g^{#4} _{#1}
        }{
            g^{#4} _{#1}(#3)
        }
    }{
        \ifthenelse{\equal{#3}{}}{
            g^{#4} _{#1}(\cdot|#2)
        }{
            g^{#4} _{#1}(#3|#2)
        }
    }}

\newcommandx\jcpdata[4][4=]{
    \ifthenelse{\equal{#2}{}}{
        \ifthenelse{\equal{#3}{}}{
            \bar\pi^{#4} _{#1}
        }{
            \bar\pi^{#4} _{#1}(#3)
        }
    }{
        \ifthenelse{\equal{#3}{}}{
            \bar\pi^{#4} _{#1}(\cdot|#2)
        }{
            \bar\pi^{#4} _{#1}(#3|#2)
        }
    }}

\newcommandx\dtilt[4][4=]{
    \ifthenelse{\equal{#2}{}}{
        \ifthenelse{\equal{#3}{}}{
            \pi^{d, #4} _{#1}
        }{
            \pi^{d, #4} _{#1}(#3)
        }
    }{
        \ifthenelse{\equal{#3}{}}{
            \pi^{d, #4} _{#1}(\cdot|#2)
        }{
            \pi^{d, #4} _{#1}(#3|#2)
        }
    }}

\newcommandx\tilt[4][4=]{
    \ifthenelse{\equal{#2}{}}{
        \ifthenelse{\equal{#3}{}}{
            \pi^{#4} _{#1}
        }{
            \pi^{#4} _{#1}(#3)
        }
    }{
        \ifthenelse{\equal{#3}{}}{
            \pi^{#4} _{#1}(\cdot|#2)
        }{
            \pi^{#4} _{#1}(#3|#2)
        }
    }}

\newcommand\updata[3]{
    \ifthenelse{\equal{#2}{}}{
        \ifthenelse{\equal{#3}{}}{
            \bar{p}_{#1}
        }{
            \bar{p}_{#1}(#3)
        }
    }{
        \ifthenelse{\equal{#3}{}}{
            \bar{p}_{#1}(\cdot|#2)
        }{
            \bar{p}_{#1}(#3|#2)
        }
    }}

\newcommand\phat[3]{
    \ifthenelse{\equal{#2}{}}{
        \ifthenelse{\equal{#3}{}}{
            \hat{p}_{#1}
        }{
            \hat{p}_{#1}(#3)
        }
    }{
        \ifthenelse{\equal{#3}{}}{
            \hat{p}_{#1}(\cdot|#2)
        }{
            \hat{p}_{#1}(#3|#2)
        }
    }}

\newcommandx\dfw[4][4=]{
        \ifthenelse{\equal{#3}{}}{
            q^{d, #4} _{\smash{#1}}(\cdot|#2)
        }{
            q^{d, #4} _{\smash{#1}}(#3|#2)
        }
    }

\newcommandx\fw[4][4=]{
    \ifthenelse{\equal{#2}{}}{
        \ifthenelse{\equal{#3}{}}{
            q^{#4} _{#1}
        }{
            q^{#4} _{#1}(#3)
        }
        }{
        \ifthenelse{\equal{#3}{}}{
            q^{#4} _{\smash{#1}}(\cdot | #2 )
        }{
            q^{#4} _{\smash{#1}}(#3 | #2 )
        }
        }
    }

\newcommandx\barfw[4][4=]{
    \ifthenelse{\equal{#2}{}}{
        \ifthenelse{\equal{#3}{}}{
            \bar{q}^{#4} _{#1}
        }{
            \bar{q}^{#4} _{#1}(#3)
        }
        }{
        \ifthenelse{\equal{#3}{}}{
            \bar{q}^{#4} _{\smash{#1}}(\cdot | #2 )
        }{
            \bar{q}^{#4} _{\smash{#1}}(#3 | #2 )
        }
        }
    }
\newcommandx\fwp[3][3=]{
        \ifthenelse{\equal{#3}{}}{
            \bq_{\smash{#1}}(#2)
        }{
            \bq^{#3}_{\smash{#1}}(#2)
        }
    }

\newcommandx\sfwp[3][3=]{
    \ifthenelse{\equal{#3}{}}{
        \tilde\bq_{\smash{#1}}(#2)
    }{
        \tilde\bq^{#3}_{\smash{#1}}(#2)
    }
}
\newcommand{\Simplex}[1]{\Delta_{#1}}

\newcommandx\denoiser[5][4=, 5=0]{
    \ifthenelse{\equal{#2}{}}{
        \ifthenelse{\equal{#3}{}}{
            \hat\bx^{#4}_{#5}(\cdot, #1)
        }{
            \hat\bx^{#4} _{#5}(#3, #1)
        }
    }{
        \ifthenelse{\equal{#3}{}}{
            \hat\bx^{#4} _{#5}(\cdot, #1|#2)
        }{
            \hat\bx^{#4} _{#5}(#3, #1|#2)
        }
    }
}

\newcommandx\matrixdenoiser[5][4=, 5=0]{
    \ifthenelse{\equal{#2}{}}{
        \ifthenelse{\equal{#3}{}}{
            \hat\bx_{0\tbar #1}(\cdot)
        }{
            \hat\bx_{0\tbar #1}(#3)
        }
    }{
        \ifthenelse{\equal{#3}{}}{
            \hat\bx_{0\tbar #1}(\cdot|#2)
        }{
            \hat\bx_{0\tbar #1}(#3|#2)
        }
    }
}

\newcommandx\loodenoiser[5][4=, 5=0]{
    \ifthenelse{\equal{#2}{}}{
        \ifthenelse{\equal{#3}{}}{
            \hat\bx^{\loo}_{0\tbar #1}(\cdot)
        }{
            \hat\bx^{\loo}_{0\tbar #1}(#3)
        }
    }{
        \ifthenelse{\equal{#3}{}}{
            \hat\bx^{\loo}_{0\tbar #1}(\cdot|#2)
        }{
            \hat\bx^{\loo}_{0\tbar #1}(#3|#2)
        }
    }
}

    \newcommandx\hdenoiser[4][4=]{
    \ifthenelse{\equal{#2}{}}{
        \ifthenelse{\equal{#3}{}}{
            \hat{D}^{#4}_{#1}
        }{
            \hat{D}^{#4} _{#1}(#3)
        }
    }{
        \ifthenelse{\equal{#3}{}}{
            \hat{D}^{#4} _{#1}(\cdot|#2)
        }{
            \hat{D}^{#4} _{#1}(#3|#2)
        }
    }}

\newcommandx\epspred[4][4=]{
    \ifthenelse{\equal{#2}{}}{
        \ifthenelse{\equal{#3}{}}{
            \varepsilon^{#4}_{#1}
        }{
            \varepsilon^{#4} _{#1}(#3)
        }
    }{
        \ifthenelse{\equal{#3}{}}{
            \varepsilon^{#4} _{#1}(\cdot|#2)
        }{
            \varepsilon^{#4} _{#1}(#3|#2)
        }
    }}

\newcommand\clf[3]{
    \ifthenelse{\equal{#2}{}}{
        g_{#1}(#3|\cdot)
    }{
        g _{#1}(#3|#2)
    }
}

\newcommand\cfgdist[3]{
    \ifthenelse{\equal{#2}{}}{
        p^w _{#1}(#3|\cdot)
    }{
        p^w _{#1}(#3|#2)
    }
}

\newcommand\cscore[3]{
    \ifthenelse{\equal{#2}{}}{
        \ifthenelse{\equal{#3}{}}{
            s _{#1}
        }{
            s _{#1}(#3)
        }
    }{
        \ifthenelse{\equal{#3}{}}{
            s _{#1}(\cdot|#2)
        }{
            s _{#1}(#3|#2)
        }
    }}

 \def\bpi{\bm{\pi}}
 \def\bpiu{\bm{\upsilon}}

\def\interpscale{\eta}

\def\categorical{\mathrm{Cat}}

\def\bx{\mathbf{x}}
\def\btau{\bm{\tau}}
\def\br{\mathbf{r}}
\def\bn{\mathbf{n}}

\def\bp{\mathbf{p}}

\def\by{\mathbf{y}}

\def\bz{\mathbf{z}}
\def\bq{\mathbf{q}}

\def\param{\theta}

\newcommandx{\hpredx}[3][2=0,3=\param]{\smash{m^{#3} _{#2|#1}}}
\newcommandx{\predx}[2][2=0]{\smash{m _{#2|#1}}}
\newcommandx{\prednoise}[2][2=\param]{\smash{\epsilon^{#2} _{#1}}}

\newcommandx\score[4][4=]{
\ifthenelse{\equal{#2}{}}{
    \ifthenelse{\equal{#3}{}}{
        \mathbf{s}^{#4} _{#1}
    }{
        \mathbf{s}^{#4} _{#1}(#3)
    }
}{
    \ifthenelse{\equal{#3}{}}{
        s^{#4} _{#1}(\cdot|#2)
    }{
        s^{#4} _{#1}(#3|#2)
    }
}}

\newcommand\scoreD[3]{\mathbf{s}_{#1}(#2,#3)}

\newcommand{\alain}[1]{\todo[linecolor=cyan,backgroundcolor=cyan!25,bordercolor=cyan]{\tiny{\textbf{AD}: #1}}}

\newcommand{\ps}[2]{\langle #1, #2 \rangle}

%% file: body/introduction.tex
\section{Introduction}
\label{sec:introduction}

Discrete diffusion models have emerged as a strong alternative to autoregressive generation for language and other structured discrete domains \citep{labs2025mercuryultrafastlanguagemodels, geminidiffusion2025blog, nie2025largelanguagediffusionmodels}. Unlike left-to-right models, they support parallel token updates, controllable infilling, and a flexible trade-off between compute and sample quality at inference time. Recent models show that they can scale competitively on modern text benchmarks \citep{lou2024discrete,sahoo2024simple,peebles2023dit,nie2025scaling}.
Although early work emphasized Uniform Diffusion Models (UDM), in which each position is replaced by a uniformly sampled token \citep{austin2021structured,hoogeboom2021argmax,schiff2024simple}, more recent approaches have relied on Masked Diffusion Models (MDM), where corruption replaces tokens by a separate mask symbol, motivated by their better perplexity at large vocabulary size \citep{lou2024discrete,sahoo2024simple,sahoo2025diffusion}. 
This empirical shift left open fundamental questions regarding parameterization choices in UDMs. To make this question precise, we  distinguish three common parameterizations: the network prediction can be used as (i) a score parameterizing a rate matrix \citep{lou2024discrete, sun2023scorebased, pham2025discrete}, 
(ii) a denoiser, where the network assigns probabilities to possible clean  $\bx_0$, and the reverse transition is computed as the mixture of the $\bx_0$-conditioned reverse transitions $\fw{s\tbar 0, t}{\bx_0, \bx_t}{}$, weighted by these probabilities
\citep{austin2021structured,gu2022vectorquantizeddiffusionmodel}, or (iii) a direct clean-token proxy $\smash{\denoiser{t}{}{\bx_t}[\param]}$, plugged instead of $\bx_0$ in $\fw{s\tbar 0, t}{\bx_0, \bx_t}{}$ \citep{hoogeboom2021argmax,sahoo2024simple,schiff2024simple,vonrutte2025gidd}.

\paragraph{Contributions.} In this paper, we first show that for uniform diffusion the plug-in bridge parameterization, the third option above, does not reach its optimum for the standard ELBO loss at the denoising posterior, which has been the central quantity in diffusion models. We show that for this plug-in parameterization, the ELBO is instead minimized by the leave-one-out posterior: to predict the clean token at position $\ell$, the model uses the noisy tokens at all positions except $\ell$. This distinction is mostly hidden in masked diffusion because the parameterizations (ii) and (iii) coincide.

This observation exposes a mismatch between two standard objectives for UDMs: the plug-in ELBO targets a leave-one-out posterior,
whereas the usual cross-entropy objective targets the plain denoising posterior. We use the conversion formulas between these two
quantities to train either parameterization with either objective, namely a denoiser through the ELBO or a leave-one-out predictor
through cross-entropy. Empirically, across objectives and samplers, we find that parameterizing the leave-one-out posterior
consistently improves generative performance. These conversions are also useful at inference time. They turn a trained denoiser into the leave-one-out quantities needed for an informed predictor-corrector sampler, in the spirit of \citep{zhao2025informed}, without training an auxiliary model. They also indicate where to apply sampling heuristics: applying top-$p$ or temperature to the leave-one-out predictor, rather than to the denoiser itself, improves generation at no additional training cost.

  We finally revisit the gap between uniform and masked diffusion processes \citep{sahoo2024simple,vonrutte2025gidd,rutte2026scaling}. We ask whether this gap is caused by the difference in marginals themselves, or by the parameterization and sampling scheme usually paired with each process. To separate these effects, we propose two absorbing-state constructions of uniform diffusion that preserves the UDM joint law while
  inheriting structural simplifications of MDM. We coin these Absorbing State Uniform Diffusion Model (AUDM) and Masked Uniform Diffusion Model (MUDM). We show that MUDM makes it posibble to reuse the MDM denoiser and obtain a joint process with UDM joint law. As for AUDM, it reaches perplexity and generation quality comparable to masked diffusion. The empirical advantage of masked diffusion therefore appears to stem not from masked versus uniform marginals, but from parameterization choices that can be
  transferred to a model with UDM marginals.

We validate our theoretical insights and methods through extensive experiments on large scale language modeling tasks as well as a smaller scale Sudoku task. In all the cases considered, the leave-one-out parameterization performs better than its denoiser counterpart, and our absorbing construction is on par or slightly better than masked diffusion models.

%% file: body/background.tex
\paragraph{Notation.}
We consider a discrete space $\msx=\msv^L$ with $\msv=\intset{1}{\vocab}$, and write $\Simplex{\vocab}=\{p\in\rset^K:\ p_i\ge 0,\ \sum_{i} p_i=1\}$ for the probability simplex. Tokens are identified with their one-hot vectors in $\rset^K$ via the canonical basis $(e_1,\ldots,e_K)$, so categorical distributions are points in $\Simplex{K}$ (or $\Simplex{K}^L$ in the multi-token case). For masked diffusion, the mask token is denoted $\mask$. We write $\bx\in\msx$, $x\in\msv$, $\bx^\ell$ for the $\ell$-th token of $\bx$, and $\bx^{-\ell} \in \msv^{\len - 1}$ for $\bx$ with the $\ell$-th token removed.
For $\pi \in \Delta_K$, we  denote by $\categorical(\pi)$ the categorical distribution with parameter $\pi$. For ease of notation, we also denote by $x \mapsto \categorical(x;\pi)= \ps{x}{\pi}$ the associated density with respect to the counting measure.

\section{Preliminaries on discrete diffusion models}
\label{sec:preliminaries}

We first briefly introduce discrete diffusion models \citep{sahoo2024simple, shi2024simplified, austin2021structured, hoogeboom2021argmax, schiff2024simple}. Denoting by $p_0 $ a data distribution on $\msx$, this class of models consists in transporting a reference distribution $\smash{\pdata{1}{}{\cdot} \eqdef \prod_{\ell = 1}^\len \categorical(\cdot;\bpi^\ell)}$, with $\smash{\bpi \in (\Simplex{\vocab})^{\len}}$, to the data distribution $\pdata{0}{}{}$ through successive transitions.
We consider a Markov process $(X_t)_{t \in [0,1]}$ started at $\pdata{0}{}{}$ with transitions 
$\fw{t\tbar s}{\bx_s}{\bx_t}[] = \prod_{\ell=1}^L \fw{t\tbar s}{\bx_s^{\ell}}{\bx_t^{\ell}} [\ell]$, and
\begin{equation}
    \label{eq:transition}
    \fw{t\tbar s}{\bx_s^{\ell}}{\bx_t}[\ell]
    =
    \categorical\!\left(\bx_t;\alpha_{t\tbar s} \bx_s^{\ell}+(1-\alpha_{t\tbar s})\bpi^{\ell}\right)\eqsp, \quad 0\leq s < t \leq 1 \eqsp.
\end{equation}
Here $(\alpha_t)_{t\in[0,1]}$ is a monotone noise schedule with $\alpha_0=1$, $\alpha_1 \approx 0$, and $\alpha_{t\tbar s}\eqdef \alpha_t/\alpha_s$. We denote by $\pdata{t}{}{}$ the marginal of $X_t$. By reversing the Markov process in time, we obtain a generative model for $\pdata{0}{}{}$ with source distribution $\pdata{1}{}{}$ and transitions given by
\begin{equation}
    \label{eq:reverse}
    \txts \pdata{s\tbar t}{\bx_t}{\bx_s}
    = 
    \sum_{\bx_0}
    \fw{s\tbar 0,t}{\bx_0,\bx_t}{\bx_s}\,
    \pdata{0\tbar t}{\bx_t}{\bx_0} \eqsp,
\end{equation}
with $\pdata{0 \tbar t}{\bx_t}{\cdot}$ the denoising posterior and $\fw{s\tbar 0, t}{}{}$ being the conditional distribution of $X_s$ given $(X_0, X_t)$. We refer to this quantity, which is central to the paper, as the \emph{bridge}. 
Note that for $s <t$, $\fw{t \tbar s}{}{}$ and $\fw{s\tbar 0,t}{}{}$ factorize across tokens, \emph{e.g.}, $\smash{\fw{s \tbar 0, t}{\bx_0, \bx_t}{\bx_s} = \prod_{\ell = 1}^L \fw{s\tbar 0,t}{\bx_0^\ell, \bx_t^\ell}{\bx_s^\ell}[\ell]}$ with 
\begin{align}
    \fw{s\tbar 0,t}{x_0, x_t}{\cdot}[\ell]
    & = \categorical\bigg(\cdot; \frac{\big[ \alpha_{t\tbar s} x_t + (1 - \alpha_{t\tbar s})\dotp{x_t}{\bpi^\ell}\one_\vocab\big] \odot \big[\alpha_s x_0 + (1 - \alpha_s) \bpi^{\ell} \big]}{\alpha_t \dotp{x_t}{x_0} + (1 - \alpha_t) \dotp{x_t}{\bpi^{\ell}}} \bigg)\eqsp,
    \label{eq:bridge-prob}
\end{align}
 when $\fw{t\tbar 0}{x_0}{x_t}[\ell] > 0$. To define a practical model, we approximate the reverse transitions. To this end, we consider the following joint law with factorized transitions parameterized by $\param$:
\begin{equation}
    \label{eq:plugin-reverse}
     \pdata{0:n}{}{\bx_{\tk{0:n}}}[\param] = \pdata{1}{}{\bx_\tk{\last}} \prod_{i = 1} ^\last \pdata{\sk{i}\tbar \tk{i}}{\bx_{\tk{i}}}{\bx_{\sk{i}}}[\param] \quad \text{where \quad $\pdata{s\tbar t}{\bx_t}{\bx_s}[\param] \eqdef \prod_{\ell = 1} ^\len \pdata{s\tbar t}{\bx_t}{\bx^\ell _s}[\param, \ell] \eqsp,$}
\end{equation}
for a choice of timesteps $0 = t_0 < \ldots < t_n = 1$, setting $s_i = t_{i - 1}$ for notational convenience.
Discrete diffusion models are then trained by minimizing the \emph{expected} NELBO $\pE[\mathrm{L}_{\last}(X_0; \param)]$. For a fixed $\bx_0$, $\mathrm{L}_{\last}(\bx_0;\param)$ is given, up to an additive constant independent of $\param$, by:
\begin{equation}
    \label{eq:elbo-discrete}
    \mathrm{L}_\last(\bx_0;\theta)
    \!=\! \mathbb{E}\bigg[ -\!\log \pdata{0\tbar \tk{1}}{X_{\tk{1}}}{\bx_0}[\param] + \sum_{i=2}^n \kldivergence{\fw{\sk{i}|0, \tk{i}}{X_0, X_{\tk{i}}}{}}{\pdata{\sk{i}|\tk{i}}{X_\tk{i}}{}[\param]} \bigg| X_0 = \bx_0 \bigg] \eqsp.
\end{equation}
\paragraph{Uniform and masked diffusion models.} In this paper we focus on Masked Diffusion Models (MDM) and Uniform Diffusion Models (UDM). For MDM \cite{sahoo2024simple,shi2024simplified}, we take $\bpi^\ell = \mask \eqdef \onehot{\vocab}$ where $\mask$ denotes the mask token and we assume that any state containing this token has zero data probability, \emph{i.e.}, $\pdata{0}{}{\bx}=0$ whenever $\bx^\ell=\mask$ for some $\ell$. 
For UDM \cite{austin2021structured,hoogeboom2021argmax, schiff2024simple}, $\bpi^\ell \eqdef \mathbf{1} / K$ where $\mathbf{1}$ is the vector with all ones, \emph{i.e.}, the uniform distribution.
The MDM bridge is given, for $\ell \in \intset{1}{\len}$, by:
\begin{equation}
\label{eq:mdm-og-bridge}
\fw{s\tbar 0, t}{x_0, x_t}{x_s}[\ell] = \begin{cases}
\categorical \left(x_s; x_0 \right) & \text{if $x_t = x_0$} \\
\categorical \left(x_s; \frac{\alpha_s - \alpha_t}{1 - \alpha_t} x_0 + \frac{1 - \alpha_s}{1 - \alpha_t} \mask \right) \quad & \text{if $x_t = \mask$}
\end{cases}
\end{equation}
and is undefined for $x_t \notin \{\mask, x_0\}$. 
Regarding UDM, the bridge has the explicit form 
\begin{equation}
\label{eq:usdm-bridge}
\fw{s\tbar 0,t}{x_0,x_t}{x_s}[\ell]\!=\!\categorical \left(x_s; \frac{
        K\alpha_t \dotp{x_t}{x_0} x_t
        \!+\!(\alpha_{t\tbar s}-\alpha_t) x_t
        \!+\!(\alpha_s-\alpha_t) x_0
        \!+\!D_{s, t}\one_\vocab / \vocab
    }{
        K\alpha_t \dotp{x_t}{x_0} + 1-\alpha_t
    } \right),
\end{equation}
with
$
    D_{s,t}\eqdef (1-\alpha_{t\tbar s})(1-\alpha_s).
$

%% file: body/method.tex
\section{Elucidating the parameterizations in discrete diffusion models}
\subsection{The different parameterizations of the reverse transitions}
A central design choice is how to parameterize the family of reverse conditionals $\{\pdata{s\tbar t}{}{}[\param,\ell]\}_{\ell=1}^{\len}$. A natural approach is to let it be the direct output of a neural network with input $(\bx_t,s,t)$. As observed in \cite{austin2021structured}, this approach does not exploit the inductive bias associated with the forward process. Another option is to use a neural network that takes $(\bx_t,t)$ as input and produces a simplex-valued prediction, denoted by $\smash{\denoiser{t}{}{\bx_t}[\param]}$, from which the reverse conditionals are constructed through the bridge. There are two common instantiations of this idea. The \textit{marginalization} parameterization \cite{austin2021structured,gu2022vectorquantizeddiffusionmodel} interprets this prediction as the factorized distribution $\smash{\pdata{0\tbar t}{\bx_t}{}[\param] \eqdef \prod_{\ell = 1}^\len \categorical(\cdot; \denoiser{t}{}{\bx_t}[\param]^\ell)}$ and defines
\begin{equation}
\label{eq:param_marginalization}
\hspace{2.5cm} \txts \jpdata{s\tbar t}{\bx_t}{\bx_s}[\param] = \sum_{\bx_0} \fw{s\tbar 0,t}{\bx_0,\bx_t}{\bx_s}\, \pdata{0\tbar t}{\bx_t}{\bx_0}[\param] \tag{Marginalization}\eqsp,
\end{equation}
whereas the \textit{bridge plug-in} parameterization \cite{hoogeboom2021argmax, sahoo2024simple, shi2024simplified, schiff2024simple, vonrutte2025gidd} uses the same simplex-valued prediction directly inside the bridge \eqref{eq:bridge-prob}:
\begin{equation}
\label{eq:param_plug_in}
\hspace{2cm} \hpdata{s\tbar t}{\bx_t}{\bx_s}[\param] =
\fw{s\tbar 0,t}{\denoiser{t}{}{\bx_t}[\param],\bx_t}{\bx_s}\eqsp. \tag{Bridge Plug-in}
\end{equation}
Note that these two parameterizations coincide if $ \Simplex{\vocab} \ni \bx_0^\ell \mapsto \fw{s\tbar 0,t}{\bx_0^\ell,\bx_t^\ell}{\bx_s^\ell}[\ell]$ is affine in $\bx_0^\ell$. 
\begin{remark}[Bridge extensions]
There are two extension choices required for our analysis. First, the bridge is determined by the Bayes' formula only on pairs $(x_0,x_t)$ such that $\fw{t\tbar 0}{x_0}{x_t}[\ell]>0$. Outside this support, its value is fixed by convention. For UDM no such completion is needed, while for MDM we use the completion (and keep the same notation) of \cite{shi2024simplified} which agrees with \eqref{eq:mdm-og-bridge} on supported pairs:
\begin{equation}
\label{eq:mdm-bridge}
\fw{s\tbar 0, t}{x_0, x_t}{x_s}[\ell] = \begin{cases}
\categorical \left(x_s; x_t \right) & \text{if $x_t \neq \mask$} \\
\categorical \left(x_s; \frac{\alpha_s - \alpha_t}{1 - \alpha_t} x_0 + \frac{1 - \alpha_s}{1 - \alpha_t} \mask \right) \quad & \text{otherwise}.
\end{cases}
\end{equation}
 Second, the plug-in parameterization evaluates the first bridge argument at the simplex-valued prediction $\denoiser{t}{}{\bx_t}[\param]^\ell$, which is not necessarily a one-hot vector. The extensions we consider next satisfy 
$\smash{\fw{s\tbar 0, t}{\nu, x_t}{x_s}[\ell] = \fw{t \tbar s}{x_s}{x_t}[\ell]\fw{s\tbar 0}{\nu}{x_s}[\ell] / \fw{t\tbar 0}{\nu}{x_t}[\ell]}$ for $\nu \in \Simplex{\vocab}$. 
We refer to this property as Assumption \assp{\ref{ass:simplex-bayes-bridge}} in \Cref{appendix:loo}. It holds by construction for the UDM bridge \eqref{eq:usdm-bridge}, and \eqref{eq:mdm-bridge} also satisfies it, as shown in \Cref{appendix:loo}. 
\end{remark}
A recurring argument \cite{sahoo2025diffusion, sahoo2024simple,vonrutte2025gidd} used to motivate \ref{eq:param_plug_in} is that, for a fixed $\bx_0$, the minimizer of \eqref{eq:elbo-discrete} is given for any $t$ by $\denoiser{t}{}{}[\param] \equiv \bx_0$.
However, this argument only characterizes the minimizer of the conditional objective for a fixed $\bx_0$. The actual training objective  of discrete diffusion models is instead the expected NELBO given by
\begin{equation}
    \label{eq:kl-sum}
    \pE[\mathrm{L}_{n}(X_0; \param)]
    =
    \sum_{i = 1}^\last \sum_{\ell=1}^{\len}
    \pE \big[\kldivergence{
        \pdata{\sk{i} \tbar \tk{i}}{X_\tk{i}}{}[\ell]
    }{
        \pdata{\sk{i}\tbar \tk{i}}{X_\tk{i}}{}[\param,\ell]
    }
    \big]
    +
    \mathrm{Const} \eqsp,
\end{equation}
with $\smash{\pdata{s\tbar t}{\bx_t}{\bx^\ell _s}[\ell] \eqdef \sum_{\bx^{-\ell} _s} \pdata{s\tbar t}{\bx_t}{\bx_s}}$ the $\ell$-th token marginal transition. Hence, we can see that the minimum is attained when $\smash{\pdata{\sk{i}\tbar \tk{i}}{\bx_\tk{i}}{}[\param] = \prod_{\ell = 1}^\len \pdata{\sk{i} \tbar \tk{i}}{\bx_\tk{i}}{}[\ell]}$. Moreover, from \eqref{eq:reverse} and \eqref{eq:bridge-prob} we have:
\[
 \txts
 \pdata{s \tbar t}{\bx_t}{\bx_s^\ell}[\ell] = \sum_{\bx^\ell _0} \fw{s\tbar 0, t}{\bx_0^\ell, \bx_t^\ell}{\bx_s^\ell}[\ell]\, \pdata{0 \tbar t}{\bx_t}{\bx^\ell _0}[\ell]\eqsp,
\]
we conclude that the \textit{marginalization} parameterization reaches the minimum of the loss by taking $\smash{\pdata{0\tbar t}{\bx_t}{\bx_0}[\param] = \prod_{\ell = 1}^\len \pdata{0\tbar t}{\bx_t}{\bx^\ell_0}[\ell]}$, or equivalently, $\denoiser{t}{}{\bx_t}[\param] = \pE[X_0 | X_t = \bx_t]$. By contrast, the same conclusion is more delicate for \ref{eq:param_plug_in} since the canonical extension of \eqref{eq:bridge-prob}, as defined after \eqref{eq:usdm-bridge}, is not necessarily linear with respect to $\bx_0^{\ell}$ so it cannot simply be expressed as the marginalization of a denoising posterior\footnote{Interestingly, with an affine extension of the bridge, which always exists, the two parameterizations are equivalent by construction (see \Cref{app:bridge-extension} for a more thorough discussion).}. For MDM this distinction makes no difference, since \eqref{eq:mdm-bridge} is token-wise affine in $x_0$ and thus
\[
\jpdata{s\tbar t}{\bx_t}{\bx_s}[\param]
=
\fw{s\tbar 0,t}{\denoiser{t}{}{\bx_t}[\param],\bx_t}{\bx_s},
\]
so the \textit{marginalization} parameterization is itself the \textit{bridge plug-in} parameterization, and a minimizer is the denoising distribution. For UDM however, the normalization term in \eqref{eq:usdm-bridge} makes the bridge nonlinear as a function of $x_0$, so the equivalence between the two parameterizations does not hold. In that setting, it is no longer obvious which object the \ref{eq:param_plug_in} parameterization recovers at optimum, nor whether it can still attain the minimum of the expected NELBO. The next subsection answers both questions.



\subsection{Leave-one-out parameterization}
\label{sec:loo}
We first recall the factorized \emph{leave-one-out} (LOO) denoising distribution as introduced in \cite{sun2023scorebased, zhao2025informed}:
\begin{equation}
\label{eq:def_LOO-Denoiser}
\pdata{0\tbar t}{\bx_t}{\bx_0}[\loo]
:=
\prod_{\ell=1}^{\len}
\pdata{0\tbar t}{\bx_t^{-\ell}}{\bx_0^\ell}[\mathrm{loo}, \ell]  \quad \text{ where  } \quad \pdata{0\tbar t}{\bx^{- \ell} _t}{\bx^\ell _0}[\loo, \ell] \propto \sum_{\bx_0^{-\ell}} \pdata{0}{}{\bx_0} \prod_{j \neq \ell} ^L \fw{t\tbar 0}{\bx^j _0}{\bx^j _t}[j] \eqsp.
\vspace{-.3cm}
\end{equation}
Equivalently, $\pdata{0\tbar t}{\bx_t}{}[\loo] = \prod_{\ell = 1}^\len \categorical(\loodenoiser{t}{}{\bx_t}^\ell)$ with $\loodenoiser{t}{}{\bx_t}^\ell = \pE[X^\ell _0 | X^{-\ell} _t= \bx^{-\ell} _t]$. This quantity is relevant for the \ref{eq:param_plug_in} parameterization since it allows to express the minimizer of the expected NELBO under such parameterization, as made precise in the following Proposition.

\begin{mdframed}[style=propFrame,nobreak=true]
\begin{proposition}
\label{prop:loo}
For the \ref{eq:param_plug_in} parameterization, under Assumption \assp{\ref{ass:simplex-bayes-bridge}}, a minimizer of \eqref{eq:kl-sum} is $\denoiser{t}{}{\bx_t}[\param_\star] = \loodenoiser{t}{}{\bx_t}$.
In addition,
\[
 \txts \pdata{s\tbar t}{\bx_t}{}[\ell]
= \fw{s\tbar 0,t}{\loodenoiser{t}{}{\bx_t}^\ell,\bx_t^\ell}{}[\ell] \eqsp.
\]
Furthermore, this minimizer is unique in the case of UDM.
\end{proposition}
\end{mdframed}

The proof is given in \Cref{sec:proof-crefprop:loo}. Thus, the \ref{eq:param_plug_in} parameterization can attain the minimum of the ELBO, but the optimal prediction is, in general, the leave-one-out posterior rather than the denoising posterior, which can be substantially different quantities, see \Cref{app:loo-vs-denoiser}.
Masked diffusion is a special case in which the distinction between the leave-one-out posterior and the denoising posterior disappears on masked positions: if $\bx_t^\ell=\mask$, then $\smash{\pdata{0\tbar t}{\bx_t^{-\ell}}{\bx_0^\ell}[\loo,\ell]=\pdata{0\tbar t}{\bx_t}{\bx_0^\ell}[\ell]}$. On unmasked positions, by contrast, the denoiser is the Dirac mass at $\bx_t^\ell$ whereas the leave-one-out posterior need not be.

The difference between the leave-one-out and denoiser is not an obstacle to using the usual training and sampling tools, because for UDM the leave-one-out posterior, the denoising posterior, and the score can be converted into one another explicitly, \emph{i.e.}
\begin{equation}
    \label{eq:udm-loo-to-denoiser}
    \pdata{0\tbar t}{\bx_t}{\bx_0^\ell}[\ell]
    \propto
    \pdata{0\tbar t}{\bx^{-\ell} _t}{\bx^\ell _0}[\loo,\ell]
    \fw{t\tbar 0}{\bx^\ell _0}{\bx^\ell _t}[\ell]
\end{equation}
and conversely this provides a conversion from denoiser to LOO denoiser since for UDM, $\fw{t\tbar 0}{x_0}{x_t}[\ell] > 0$ for all $x_0$, so we have $\pdata{0\tbar t}{\bx_t^{-\ell}}{\bx_0^\ell}[\loo, \ell]
\propto \pdata{0\tbar t}{\bx_t}{\bx^\ell _0}[\ell] / \fw{t\tbar 0}{\bx^\ell _0}{\bx^\ell _t}[\ell]$.
These formulas hold more generally for other corruption processes; see \Cref{app:loo-to-denoiser}. We can also convert between the leave-one-out posterior and the concrete score \cite{campbell2022continuous, lou2024discrete} $\score{t}{}{\bx_t} \in \rset^{\len \times \vocab}$ defined by $\langle \by^\ell, \score{t}{}{\bx_t}^\ell \rangle = \pdata{t}{}{\by} / \pdata{t}{}{\bx_t}$ with $\by^\ell \neq \bx^\ell _t$ and $\smash{\by^{-\ell} = \bx^{-\ell} _t}$. We then have that
    \begin{equation} 
      \label{eq:score-ratio-loo}
    \langle \by^\ell , \score{t}{}{\bx_t}^\ell \rangle = \fw{t\tbar 0}{\loodenoiser{t}{}{\bx_t}^\ell}{\by^\ell} \big/  \fw{t\tbar 0}{\loodenoiser{t}{}{\bx_t}^\ell}{\bx^\ell _t} \eqsp.
    \end{equation}

We now focus on the practical consequences of these conversion formulas: 

\paragraph{Training with the leave-one-out.} A first consequence is that recent works on UDMs that use the \ref{eq:param_plug_in} parameterization \cite{hoogeboom2021argmax,schiff2024simple,sahoo2025diffusion} therefore train a model whose target is the leave-one-out posterior rather than the denoising posterior.
A structural property of this target is that $\loodenoiser{t}{}{\bx_t}^\ell$ does not depend on $\bx_t^\ell$. This invariance can be enforced architecturally with a Hollow Transformer \cite{sun2023scorebased,zhao2025informed}, whose $\ell$-th output is prevented from attending to the $\ell$-th input position. However, for text modeling this architecture is harder to train and performs worse than standard attention \cite{zhao2025informed}, which leaves the invariance to be learned. In that case, the sensitivity of the $\ell$-th prediction to $\bx_t^\ell$ provides a direct diagnostic for suboptimality of the learned leave-one-out model or a trained denoiser, provided that we convert it first to a leave-one-out using \eqref{eq:udm-loo-to-denoiser}; see \Cref{app:loo-sensitivity}.

The conversion \eqref{eq:udm-loo-to-denoiser} allows us to train a leave-one-out model with the usual denoising cross-entropy. Indeed, if a network outputs logits $\bm f_\param(\bx_t,t)$ and we set the parameterized prediction as $\denoiser{t}{}{\bx_t}[\param]^\ell=\softmax(\bm f_\param(\bx_t,t)^\ell)$, then the associated denoiser is
\begin{equation}
    \label{eq:udm-loo-to-denoiser-matrix}
    \txts
    \bm{d}_\param(\bx_t, t)^\ell
    =
    \softmax\!\left(
        \bm{f}_\param(\bx_t,t)^\ell 
        +
        \log\left(1+\frac{\vocab\alpha_t}{1-\alpha_t}\right)\cdot \bx_t^\ell
    \right)
    \eqsp .
\end{equation}
Thus, for UDM, one may optimize the standard cross-entropy loss 
\begin{equation}
\label{eq:cross-entropy-training}
\txts - \int^1 _{0}\pE\big[\sum_{\ell=1}^{\len}\log\,\dotp{X_0^\ell}{\bm{d}_\param(X_t, t)^\ell}\big]\, \rmd t \eqsp,
\end{equation}
while the underlying prediction $\denoiser{t}{}{\bx_t}[\param]$ parameterizes the leave-one-out posterior. At optimality $\bm{d}_\param(\bx_t, t) = \pE[X_0 | X_t = \bx_t]$ and since the conversion formulas define a one-to-one correspondence between the denoiser and leave-one-out denoiser \emph{in the case of UDM}, we have $\denoiser{t}{}{\bx_t}[\param] = \loodenoiser{t}{}{\bx_t}$. Experimentally we find that this loss outperforms the standard cross-entropy loss in which the model parameterizes a plain denoiser and performs as well as the continuous-time UDM plug-in NELBO used in \cite{schiff2024simple} while being simpler.

\paragraph{Predictor-corrector sampling.} The leave-one-out denoiser also gives access to the one-coordinate conditional of $\pdata{t}{}{}$ given by $\pdata{t}{\bx^{-\ell} _t}{\bx^\ell _t}[\ell] \eqdef \pdata{t}{}{\bx_t} / \sum_{\tilde\bx^\ell _t}  \pdata{t}{}{\tilde\bx_t}$. Following \cite{sun2023scorebased} it holds that: 
\begin{equation}
    \label{eq:cond-distr-loo}
    \pdata{t}{\bx_t^{-\ell}}{}[\ell] 
    = \fw{t\tbar 0}{\loodenoiser{t}{}{\bx_t}^\ell}{}[\ell] 
    = \categorical(\alpha_t \loodenoiser{t}{}{\bx_t}^\ell + (1 - \alpha_t) \bpi^\ell) \eqsp.
\end{equation}
See \Cref{app:predictor-corrector} for a proof. This observation yields a predictor-corrector \cite{campbell2022continuous,wang2026remaskingdiscretediffusionmodels} in the same spirit as \cite{zhao2025informed} without needing to train an additional model. Concretely, iterating the corresponding Gibbs updates gives a corrector kernel that preserves $\pdata{t}{}{}$, making it a valid corrector step. If one trains a denoiser or score function instead, the conversion formulas above recover the leave-one-out posterior needed to evaluate the same corrector. In practice, we use the margin-based heuristic of \cite{zhao2025informed}, with parallel updates on low-confidence positions; the full sampler is given in \Cref{app:predictor-corrector}.

%% file: body/absorbing_uniform.tex
\section{Bridging Masked and Uniform Diffusion}
\label{sec:audm}
In the previous section we clarified the difference between existing parameterizations of UDMs. We now give an alternative view: uniform diffusion can be represented as an absorbing-state diffusion after conditioning on a random sequence of absorbing states. This perspective is useful since it bridges uniform and masked diffusion. This bridge can be made exact and we show in \Cref{sec:mudm} a construction that preserves UDM joint law while using a masked diffusion denoiser.

\subsection{Absorbing state uniform diffusion}
\label{subsec:audm}
Let $U\!\sim\!\pnoise{}{}{} \eqdef \mathrm{Uniform}(\intset{1}{\vocab})^\len$ be an auxiliary noise variable. 
Conditionally on a realization $U = \noise$, define a Markov process \((X_t)_{t \in [0,1]}\) with transition kernel
\[ 
   \txts  \fw{t\tbar s}{\bx_s, \noise}{\bx_t} = \prod_{\ell = 1}^\len \categorical(\bx^\ell _t; \alpha_{t\tbar s} \bx^\ell _s + (1 - \alpha_{t\tbar s}) \noise^\ell ) \eqsp.
\]
Thus, conditioned on \(U=\mathbf u\), each coordinate evolves as an absorbing-state diffusion whose absorbing state is fixed to \(\noise^\ell\). In contrast with the usual absorbing diffusion, the absorbing state is not a distinguished mask token, but a random token sampled independently at each position. 

Conditionally on $U = \noise$, the law of $X_t$ is
\[ 
\txts \pdata{t}{\noise}{\bx_t} \eqdef \sum_{\bx_0} \fw{t\tbar 0}{\bx_0, \noise}{\bx_t} \, \pdata{0}{}{\bx_0}
\] 
Averaging over the auxiliary variable recovers the usual uniform-diffusion marginal:
\[ 
    \txts \sum_\noise \pdata{t}{\noise}{\bx_t} \, \pnoise{}{}{\noise} = \pdata{t}{}{\bx_t} \eqsp,
\]
which follows from the identity $\sum_\noise \fw{t\tbar s}{\bx_s, \noise}{\bx_t} \, \pnoise{}{}{\noise} = \fw{t\tbar s}{\bx_s}{\bx_t}$ which states that the UDM forward transition is a mixture of absorbing-state transitions. Therefore, the joint process $(U, X_t)$ has the same marginal law for $X_t$ as the original UDM, while its conditional dynamics given $U$ are those of an
absorbing-state process with position-dependent absorbing states.

This conditional viewpoint gives a direct analogue of the masked-diffusion carry-over structure. Indeed, since \(U\) is fixed in the conditional process, the bridge $\fw{s\tbar 0, t}{\bx_0, \bx_t, \noise}{\bx_s}$ is token-wise and, for each coordinate \(\ell\), is given by
\begin{equation}
\label{eq:audm-bridge}
\fw{s\tbar 0, t}{\bx^\ell _0, \bx^\ell _t, \noise^\ell}{\bx^\ell _s}[\ell] = \begin{cases}
\categorical \left(\bx^\ell _s; \bx^\ell _t \right) & \text{if $\bx^\ell _t \neq \noise^\ell$} \\
\categorical \left(\bx^\ell _s; \frac{\alpha_s - \alpha_t}{1 - \alpha_t} \bx^\ell _0 + \frac{1 - \alpha_s}{1 - \alpha_t} \noise^\ell \right) \quad & \text{otherwise}.
\end{cases}
\end{equation}
The same dichotomy appears in the conditional denoiser. By definition: 
\begin{equation}
    \label{eq:audm-posterior}
   \pdata{0\tbar t}{\bx_t, \noise}{\bx_0}
   \propto
   \pdata{0}{}{\bx_0}
   \prod_{\ell: \bx_t^\ell = \noise^\ell}
   \big[\alpha_t \indic_{\bx_0^\ell = \bx_t^\ell} + (1 - \alpha_t)\big]
   \prod_{\ell: \bx_t^\ell \neq \noise^\ell}
   \indic_{\bx_0^\ell = \bx_t^\ell}.
\end{equation}
Consequently, for every coordinate $\ell$ with $\bx_t^\ell\neq \noise^\ell$, the posterior is degenerate at $\bx_t^\ell$; meaning that $
\pdata{0\tbar t}{\bx_t; \noise}{}[\ell] = \categorical(\bx_t^\ell)
$.
In contrast, the coordinates with $\bx_t^\ell=\noise^\ell$ remain ambiguous and must be predicted from the full noisy sequence. However, not all convenient properties of MDM carry over to this setting. In particular, unlike in MDMs \cite{ou2024your,zheng2024masked}, the noise-conditioned denoiser remains explicitly time-dependent. For example, if we consider the single token case $\len = 1$, then when $\bx^\ell _t = \noise$, we have that $\pdata{0\tbar t}{\bx_t; \noise}{\bx_t} = \pdata{0}{}{\bx_t} / (1 - \alpha_t \sum_{\tilde\bx_0 \neq \bx_t} \pdata{0}{}{\tilde\bx_0})$, showing that the noise-conditioned denoiser is indeed time-dependent. 

\paragraph{Model and objective.} The structure above suggests the following parameterization:
\begin{equation}
    \label{eq:carry-over-uniform}
    \hat{\bx}^\param_0(\bx_t,t;\noise)^\ell
    \eqdef
    \begin{cases}
        \softmax(\bm{f}_\param(\bx_t, t; \noise)^\ell) & \text{if $\bx_t^\ell = \noise^\ell$}\\
        \bx_t^\ell & \text{if $\bx_t^\ell \neq \noise^\ell$}\,.
    \end{cases}
\end{equation}
We now define the corresponding generative model. The sampler first draws a latent absorbing sequence $\mathbf u \sim \pnoise{}{}{}$, initializes the terminal state from the point mass at $\noise$, and then runs the learned reverse process. The joint model is 
\[
\txts \pdata{0:n}{}{\bx_\tk{0:n}}[\param] \eqdef \sum_{\noise} \pdata{\tk{n}}{\noise}{\bx_\tk{n}} \prod_{i=1}^{n} \pdata{\sk{i} \tbar \tk{i}}{\bx_\tk{i}, \noise}{\bx_\sk{i}}[\param]\, \pnoise{}{}{\noise} \eqsp.
\]
with $\pdata{s\tbar t}{\bx_t, \noise}{\bx_s}[\param] \eqdef \fw{s\tbar 0, t}{\hat{\mathbf{x}}^\param _0(\bx_t, t; \noise), \bx_t, \noise}{\bx_s}$; \emph{i.e.} we use the \ref{eq:param_plug_in} parameterization which coincides here with \ref{eq:param_marginalization}.  
Applying Jensen's inequality to $-\log \pdata{0}{}{\bx_0}[\param]$ where $\pdata{0}{}{}[\param]$ is the time-zero marginal of the joint model above and taking the limit $n \to \infty$ in the conditional NELBO \eqref{eq:elbo-discrete} appearing in the upper bound gives the following continuous-time objective. A proof of the next Proposition is provided in \Cref{apdx:audm}.
\begin{mdframed}[style=propFrame]
\begin{proposition}
    \label{prop:absorbing-uniform-cont-loss}
    A continuous time NELBO for $- \log \pdata{0}{}{\bx_0}[\param]$ is 
\begin{multline}
    \label{eq:absorbing-ctmc-loss}
        \txts \mathrm{L}^{\audm} _{\infty}(\bx_0; \param)
        \eqdef
        - \int_0^1
         \frac{\alpha^\prime _t}{1-\alpha_t}\,
        \pE\big[
            \sum_{\ell=1}^{\len}
            \indic_{X^\ell _t = {\color{crimson}U^\ell}}
            \big\{
                1-\dotp{{\color{crimson}U^\ell}}{\hat{\bx}^\param _0(X_t,t,{\color{crimson}U})^\ell}
                \\
                -
                \indic_{\bx^\ell _0 \neq {\color{crimson}U^\ell}} \cdot
                \big(
                    1+ \log \, \dotp{\bx^\ell _0}{ \hat{\bx}^\param _0(X_t,t,{\color{crimson}U})^\ell}
                \big)
            \big\}
        \big]\rmd t,
\end{multline}
where the expectation is \wrt\ $U \sim \text{Uniform}(\intset{1}{\vocab})^{\otimes \len}$, and $X_t \sim \fw{t\tbar 0}{\bx_0,U}{}$.
\end{proposition}
\end{mdframed}
This loss can be viewed as a generalization of the continuous-time ELBO for MDMs. Indeed, upon taking $U^\ell \sim \categorical(\mask)$ for all $\ell \in \intset{1}{\len}$ we recover the MDM ELBO by further imposing that $\dotp{\mask}{\hat{\bx}^\param _0(\bx,t)^\ell} = 0$ for any $\bx$, which is referred to as the zero-remasking property as in \cite{sahoo2024simple}.

\subsection{AUDM with resampling recovers UDM}
\label{sec:reaudm}
\begin{algorithm}[t]
\caption{Remasked AUDM sampler}
\label{alg:clean-conditioned-remasking}
\footnotesize
\begin{algorithmic}
\Require time grid $0 = \tk{0} < \tk{1} < \cdots < \tk{n} = 1$
\State Draw $U_{\tk{n}} \sim \mathrm{Uniform}(\intset{1}{\vocab})^{\otimes \len}$ and $X_{\tk{n}} \sim \pdata{\tk{n}}{U_{\tk{n}}}{\,\cdot}$
\For{$i=n-1,n-2,\dots,0$}
    \State Draw $X_0 \sim \pdata{0\tbar \tk{i+1}}{X_{\tk{i+1}}, U_{\tk{i+1}}}{}[\param]$
    \State Draw $X_{\tk{i}} \sim \fw{\tk{i}\tbar 0,\tk{i+1}}{X_0, X_{\tk{i+1}}}{}$
    \State Draw $U^\ell_{\tk{i}} \sim
    \begin{cases}
        \updelta_{X^\ell_{\tk{i}}}, & X^\ell_{\tk{i}} \neq X^\ell_0,\\
        \categorical\!\left(\frac{X^\ell_0+\alpha_{\tk{i}}(\one-X^\ell_0)}{1+(\vocab-1)\alpha_{\tk{i}}}\right), & X^\ell_{\tk{i}} = X^\ell_0,
    \end{cases}$
    \Comment{in parallel over $\ell$}
\EndFor
\State \Return $X_{\tk{0}}$
\end{algorithmic}
\end{algorithm}

The construction above matches the single-time marginals of uniform diffusion. We now show that, by resampling the absorbing sequence along the reverse chain, one can recover the full UDM reverse-chain: 
\begin{equation}
    \label{eq:udm-joint}
    \txts \pdata{0:n}{}{\bx_{\tk{0:n}}}
    =
    \pdata{\tk{n}}{}{\bx_{\tk{n}}}
    \prod_{i = 1}^{n}
    \pdata{\sk{i}\tbar \tk{i}}{\bx_{\tk{i}}}{\bx_{\sk{i}}}
    \eqsp,
\end{equation}
We describe the resampling procedure in \Cref{alg:clean-conditioned-remasking}. It is a discrete-time augmented state Markov chain with one step transition 
\[ 
\txts \barpdata{s\tbar t}{\bx_t, \noise_t}{\bx_s, \noise_s} = \sum_{\bx_0} \pnoise{s\tbar 0}{\bx_0, \bx_s}{\noise_s} \fw{s\tbar 0, t}{\bx_0, \bx_t}{\bx_s}  \, \pdata{0\tbar t}{\bx_t, \noise_t}{\bx_0} \eqsp.
\]  
Equivalently, given $(\bx_t, \noise_t)$, the sampler first draws $X_0$ from the absorbing-state denoiser, then draws $X_s$ from the UDM bridge, and finally refreshes the absorbing sequence by drawing $U_s$ conditionally on $X_0, X_s$. Thus, compared with the fixed-$U$ sampler described above, the only additional operation is the final resampling step. The resampling distribution is defined as 
\begin{equation}
    \label{main-eq:noise_cond_x0_xs}
    \pnoise{s\tbar 0}{\bx_0, \bx_s}{\noise_s}
    \eqdef
    \frac{\pnoise{}{}{\noise_s}\fw{s\tbar 0}{\bx_0,\noise_s}{\bx_s}}
    {\fw{s\tbar 0}{\bx_0}{\bx_s}}
    =
    \prod_{\ell=1}^{\len}
    \begin{cases}
        \categorical(\noise^\ell _s; \bx^\ell _s) & \bx^\ell_s \neq \bx^\ell_0, \\
        \categorical\!\left(\noise^\ell_s;\frac{\alpha_s \one + (1 - \alpha_s) \bx^\ell_0)}{1+(\vocab-1)\alpha_s}\right), & \bx^\ell_s = \bx^\ell_0,
    \end{cases}
\end{equation}
The first case says that, if a coordinate changed between $\bx^\ell _0$ and $\bx^\ell _s$, then the absorbing value, or noise, must be the observed value $\bx^\ell _s$. The second case accounts for the ambiguity when the coordinate did not change. This sampler preserves the UDM joint law \eqref{eq:udm-joint} marginally as stated in the next Proposition. 
\begin{mdframed}[style=propFrame]
    \begin{proposition}
        \label{prop:remasked-audm-udm-joint}
       The trajectory generated in \Cref{alg:clean-conditioned-remasking} satisfies $\law(X_{\tk{0}},\ldots,X_{\tk{n}}) = \pdata{0:n}{}{}$, where $\pdata{0:n}{}{}$ is the uniform UDM reverse-chain law defined in \eqref{eq:udm-joint}.
    \end{proposition}
\end{mdframed}
The proof is given in \Cref{sec:proof-remasking}. This result essentially shows that uniform diffusion can be decomposed into simple sampling operations: first sampling from an absorbing-state denoiser, then the uniform bridge, and then resampling the absorbing-state according to \eqref{main-eq:noise_cond_x0_xs}. In practice, we take as approximate model one that uses the factorized transitions of the joint law on the augmented states generated by \Cref{alg:clean-conditioned-remasking}, which is equivalent to replacing the exact $\pdata{0\tbar t}{\bx_t,\noise_t}{}$ by the approximation introduced in the previous section. For the initialization we set $X_\tk{n} = U_\tk{n}$. 

\subsection{Masked Uniform Diffusion}
\label{sec:mudm}
The absorbing-state lifting in the previous sections gives a noise-conditioned analogue of masked diffusion. We now give a complementary lifting in which the connection to masked diffusion is exact: conditioned on latent transition times, the UDM denoising problem is precisely a MDM denoising problem. Following \cite{chen2024fast}, we introduce the transition times $\bm{\tau}\!\in\!\rset^\len _{\geq 0}$ with c.d.f $\smash{\mathbb{P}(\bm{\tau}^\ell<t)=1 - \alpha_t}$. Conditionally on $(X_0,\btau)$, we consider the process $(X_t)_{t\in(0,1]}$ evolving in reverse time with $\smash{X_1 \sim \mathrm{Uniform}([\intset{1}{\vocab}])^{\otimes \len}}$ and with transitions for $0 < s < t < 1$ given by 
\[
    \fw{s\tbar 0,t}{\bx_0,\bx_t,\bm{\tau}}{\bx_s}
    =
    \prod_{\ell=1}^{\len}
    \begin{cases}
        \categorical(\bx_s^\ell;\bx_0^\ell)
        & \text{if $s<\bm{\tau}^\ell\le t$,}\\
        \categorical(\bx_s^\ell;\bx_t^\ell)
        & \text{otherwise,}
    \end{cases}
    \eqsp .
\]
Essentially $X^\ell _s$ equals $X^\ell _0$ exactly when the jump occurs between $s$ and $t$; otherwise no jump occurs on \((s,t]\), so \(X_s^\ell=X_t^\ell\). Furthermore, the conditional distribution of $X_t$ is 
\[ 
    \txts \fw{t\tbar 0}{\bx_0, \btau}{\bx_t} = \prod_{\ell = 1}^\len \categorical \left(
        \bx_t^\ell;
        \indic_{\bm{\tau}^\ell>t}\bx_0^\ell
        +
        \indic_{\bm{\tau}^\ell\le t} \one / \vocab 
    \right) \eqsp,
\]
and thus, marginalizing over $\btau$ recovers the UDM forward transition $\fw{t\tbar 0}{\bx_0}{\bx_t}$. 

Define $\tilde\bx_t(\bm{\tau})^\ell \eqdef \bx_t^\ell$ when $\bm{\tau}^\ell>t$ and $\tilde\bx_t(\bm{\tau})^\ell \eqdef \mask$ when $\bm{\tau}^\ell\le t$. We show in \Cref{app:mudm} that the transitions of the reverse Markov chain conditioned only on $\btau$ admits the following one step transition 
\[
    \pdata{s\tbar t}{\bx_t,\bm{\tau}}{\bx_s}
    =
    \sum_{\bx_0}
    \fw{s\tbar 0,t}{\bx_0,\bx_t,\bm{\tau}}{\bx_s}\,
    \pdata{0\tbar t}{\tilde\bx_t(\bm{\tau})}{\bx_0}[\mathrm{mask}]
    \eqsp,
\]
where $\pdata{0\tbar t}{}{}[\mathrm{mask}]$ is the denoiser of a MDM. This shows that we can recycle a MDM denoiser to define a discrete diffusion model with UDM marginals. 
\paragraph{Recovering the UDM joint distribution.} We can go one step further and recover the full UDM joint law \eqref{eq:udm-joint}. We apply the same logic as in \Cref{sec:reaudm} and \Cref{alg:clean-conditioned-remasking} and resample the transition times after each reverse step according to $\smash{j_{s\tbar 0}(\btau_s|\bx_0,\bx_s) \eqdef \prod_{\ell = 1}^\len j^\ell _{s\tbar 0}(\btau^\ell _s | \bx^\ell _0, \bx^\ell _s)}$ with 
\[
    j_{s\tbar 0}(\btau^\ell_s|\bx^\ell_0,\bx^\ell_s)
    =
    \begin{cases}
        \frac{-\alpha'_{\btau^\ell_s}}{1-\alpha_s}\indic_{\btau^\ell_s\le s} 
        & \text{if $\bx^\ell_s \neq \bx^\ell_0$,}\\
        \frac{-\alpha'_{\btau^\ell_s}}{1+(\vocab-1)\alpha_s}
        \big(
            \indic_{\btau^\ell_s\le s}
            +
            \vocab\indic_{\btau^\ell_s>s}
        \big)
        & \text{if $\bx^\ell_s = \bx^\ell_0$.}
    \end{cases}
 \]
The corresponding lifted reverse transition is then
\[
    \barpdata{s\tbar t}{\bx_t,\bm{\tau}_t}{\bx_s,\bm{\tau}_s}
    \eqdef
    \sum_{\bx_0}
    j_{s\tbar 0}(\bm{\tau}_s|\bx_0,\bx_s)\,
    \fw{s\tbar 0,t}{\bx_0,\bx_t}{\bx_s}\,
    \pdata{0\tbar t}{\tilde\bx_t(\bm{\tau}_t)}{\bx_0}[\mathrm{mask}]
    \eqsp .
\]
This has a simple remasking interpretation. If $\bx^\ell_s \neq \bx^\ell_0$, then the transition must already have occurred before time $s$, so  $\btau^\ell_s \le s$. If $\bx^\ell_s = \bx^\ell_0$, there are two possible explanations: either the position has not transitioned yet, so $\btau^\ell_s > s$, or it has already transitioned before time $s$ and the uniform resampling happened to return $\bx^\ell_0$. The law $j_{s\tbar 0}$ reweights these two cases appropriately. 

The resulting algorithm analogous to \Cref{alg:clean-conditioned-remasking} is provided in \Cref{alg:mudm-flag-resampling}. We then have the following result:
\begin{mdframed}[style=propFrame]
\begin{proposition}
    \label{prop:mudm-joint}
       The trajectory generated in \Cref{alg:mudm-flag-resampling} satisfies $\law(X_{\tk{0}},\ldots,X_{\tk{n}}) = \pdata{0:n}{}{}$, where $\pdata{0:n}{}{}$ is the uniform UDM reverse-chain law defined in \eqref{eq:udm-joint}.
\end{proposition}
\end{mdframed}
This formalizes the idea that UDM is naturally associated with a remasking mechanism absent from standard MDM, and clarifies the relation between these two constructions. \Cref{app:mudm} gives the full derivation and algorithm.

%% file: body/related_works.tex
\section{Related Works}
\paragraph{Leave-one-out denoiser.} \emph{Categorical ratio matching} \cite{sun2023scorebased} trains discrete diffusion models by learning the concrete score $\score{t}{}{}$; by \eqref{eq:score-ratio-loo}, learning the conditional distribution in \eqref{eq:cond-distr-loo}, which is also leave-one-out, is enough to recover the score and has motivated the use of the Hollow Transformer architecture. As recalled in \Cref{sec:categorical_ratio_matching}, using the Hollow Transformer is essential for their loss function to recover the correct conditional distribution; otherwise, the model can simply learn the identity function. Here we show instead that standard UDMs trained with the bridge parameterization and ELBO already target the same leave-one-out object, and we propose a simple cross-entropy objective whose \emph{unique} minimizer is the leave-one-out conditional distribution, without imposing architectural constraints.

\paragraph{Predictor-corrector samplers.} Predictor-corrector samplers alternate a reverse transition step with a corrector kernel that preserves the marginal $\pdata{t}{}{}$. For discrete diffusion, \cite{campbell2022continuous} propose an uninformed birth-death corrector whose rates combine the forward and reverse generators, and ReMDM \cite{wang2026remaskingdiscretediffusionmodels} introduces remasking in masked diffusion, which can be read as predictor steps interleaved with a single corrector step. Closer to our setting, \cite{zhao2025informed} train an auxiliary leave-one-out model to build an informed Gibbs corrector. For UDMs, our key observation is that, leveraging the leave-one-out representation, this correction does not require a separately trained model. Finally, in the recent work \cite{deschenaux2026diffusiondualitychapterii} the authors propose the $\Psi$-samplers framework, which recovers specific instances of known predictor-corrector samplers. Interestingly, when specialized to UDMs, their sampler assumes that the trained network has learned a LOO denoiser and thus does not work when the network has learned a denoiser instead. We provide in \Cref{sec:loo-denoiser-confusion} the general form that applies in both cases. 

\paragraph{UDMs with augmented states.} Several recent works also augment the state space of UDMs, but for different purposes. \cite{zheng2024reparameterizeddiscretediffusionmodel} introduces routing variables that decide, in a confidence-dependent way, which positions should be treated as already clean during sampling. \cite{liu2024think} conditions on a binary corruption mask and learns a separate planner that chooses which positions to denoise as well as a mask-conditioned denoiser, leading to a planner-denoiser factorization. \cite{chen2024fast} instead uses latent transition times, which reveals an exact connection between UDMs and masked denoising, although the implemented sampler ultimately plugs in an ordinary denoiser that is not conditioned on these times. Our absorbing-state construction is different in spirit: the auxiliary variable is part of an exact lifted model, and the resampling step is chosen so that marginalizing the augmentation recovers the UDM reverse-chain law rather than only providing a decoding heuristic. We give the detailed comparison in \Cref{sec:audm-related-works}.

%% file: body/experiments.tex
\section{Experiments}
\label{sec:experiments}
\begin{figure}[t]
    \centering
    \begin{subfigure}[t]{0.48\textwidth}
        \centering
        \includegraphics[width=\linewidth]{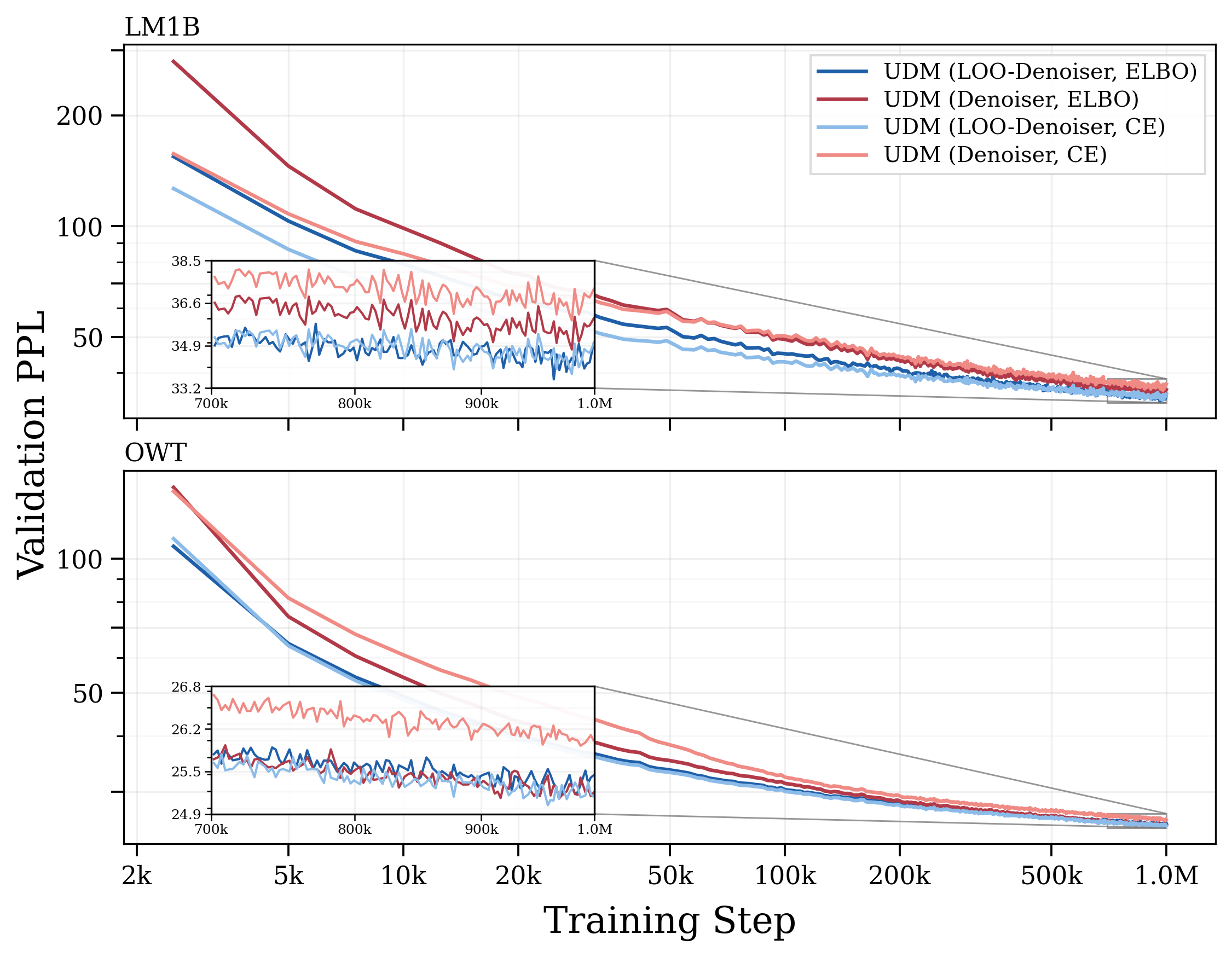}
        \label{fig:training_loss_loo_denoiser}
    \end{subfigure}
    \hfill
    \begin{subfigure}[t]{0.48\textwidth}
        \centering
        \includegraphics[width=\linewidth]{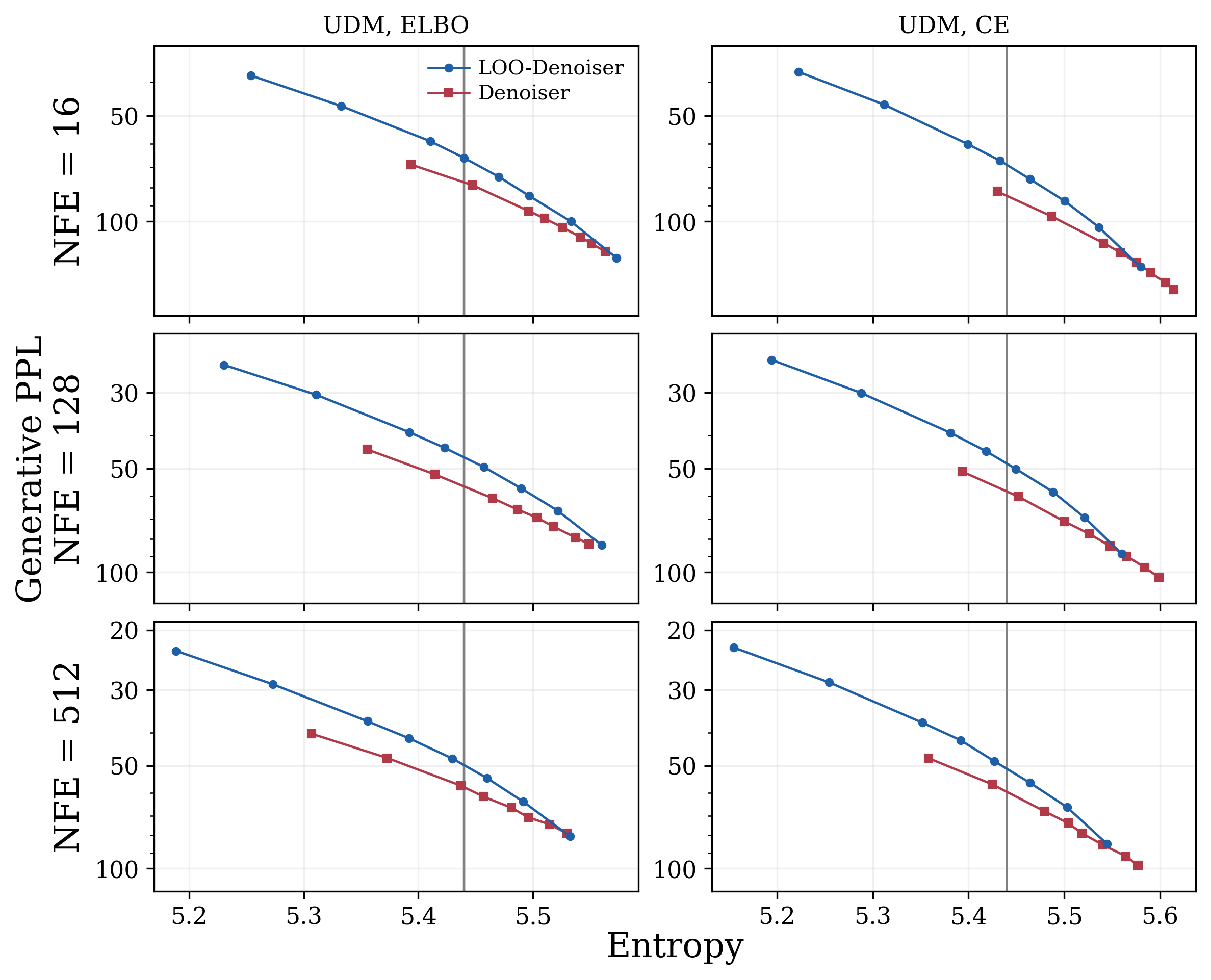}
        \label{fig:generative_ppl_loo_denoiser}
    \end{subfigure}
    \captionsetup{font=small}
    \vspace{-.5cm}
    \caption{Comparison between the denoiser and leave-one-out parameterizations. \textbf{Left:} validation perplexity during training. \textbf{Right:} top-$p$ sampling Gen-PPL frontier, obtained by sweeping ($p \in [0.8, 1.0]$). Top$-p$ is applied to the LOO denoiser and denoiser respectively. The curves are shown only up to entropy 5.6. The vertical line corresponds to the dataset entropy}
    \label{fig:parameterization-comparison}
\end{figure}
This section tests the practical implications of our analysis for UDM parameterization, sampling, and reverse process design. We compare the LOO denoiser and plain denoiser as primitive learned objects, evaluate the predictor-corrector sampler of \Cref{sec:loo}, and compare AUDM and ReAUDM from \Cref{sec:audm} against UDMs and MDMs. We do not benchmark against wrapper methods or more elaborate UDM variants, which act at another level of the modeling pipeline; our conclusions are complementary to such methods and can guide the parameterization on which they build. We evaluate primarily on language modeling and use Sudoku as a smaller structured discrete check. All the experimental details are provided in \Cref{app:exp-details}.

\textbf{Language modeling.\,}
We experiment on One Billion Word (LM1B \cite{chelba2014billionword}) and OpenWebText (OWT \cite{Gokaslan2019OpenWeb}), following \cite{lou2024discrete,sahoo2024simple,sahoo2025diffusion}. LM1B uses the $\mathrm{bert\text{-}base\text{-}uncased}$ tokenizer with $\len = 128$, while OWT uses the $\mathrm{gpt2}$ tokenizer with $\len = 1024$. All models use the same DiT-based architecture as \cite{lou2024discrete,sahoo2025diffusion} and are trained from scratch for 1M steps with batch size 512. To assess sample quality, we report GPT-2 Large generative perplexity (Gen-PPL) with the OWT-trained model. Similarly to \cite{liu2024think,pynadath2026generativefrontiersevaluationmatters} we rely on the \emph{generative frontier} where we plot the Gen-PPL as a function of 1-gram entropy. We vary the entropy using either temperature or nucleus sampling. This evaluation is substantially more informative than reporting the performance at a single entropy, as emphasized by \cite{pynadath2026generativefrontiersevaluationmatters} since different entropies may reverse the Gen-PPL ordering of two models.

\textbf{Sudoku.\,}
We also run a $9 \times 9$ Sudoku completion experiment to test the same modeling choices on a different structured discrete task. We use a similar setup to \citep{ye2025autoregressiondiscretediffusioncomplex}. The model is a smaller DiT with about 6M parameters, trained on a 50k-puzzle Sudoku split. We sweep the learning rate independently for every method family and report, for each method, the learning rate that maximizes the average validation solve rate across seeds and NFEs. The full setup is summarized in \Cref{tab:sudoku-exp}.

\subsection{Leave-one-out parameterization}
\newcommand{\bestcat}[1]{\cellcolor{orange!25}#1}
\begin{table}[!t]
\centering
\captionsetup{font=small}
\caption{Perplexity ($\downarrow$). OWT is the validation distribution; other columns report zero-shot transfer. Bold marks the overall best value, and orange the best within each method family.}
\label{tab:ppl-eval-multidataset_ppl}
\makebox[\linewidth][c]{%
\resizebox{0.90\linewidth}{!}{%
\begin{tabular}{lllr|rrrrrr}
\toprule
Method & Loss & Param. & OWT & AG News & Lambada & LM1B & PTB & PubMed & WikiText \\
\midrule
\rowcolor{blue!10}
\multicolumn{10}{l}{Absorbing Uniform Diffusion} \\
& ELBO & Denoiser & \underline{23.32} & \textbf{61.37} & \textbf{45.87} & \textbf{72.02} & \textbf{72.21} & \textbf{43.01} & \textbf{33.22} \\
\addlinespace
\rowcolor{blue!10}
\multicolumn{10}{l}{Uniform Diffusion} \\
& Cross-Entropy & Denoiser & 25.45 & 74.03 & 49.04 & 83.43 & 84.13 & 44.66 & 39.17 \\
& Cross-Entropy & LOO-Denoiser & \bestcat{25.18} & \bestcat{72.66} & 46.91 & 81.54 & 87.79 & \bestcat{43.72} & \bestcat{36.67} \\
& ELBO & Denoiser & 25.20 & 73.33 & 46.80 & \bestcat{80.33} & 84.94 & 45.75 & 37.44 \\
& ELBO & LOO-Denoiser & 25.30 & 73.82 & \bestcat{\underline{46.22}} & 82.03 & \bestcat{83.24} & 44.52 & 37.23 \\
\addlinespace
\rowcolor{blue!10}
\multicolumn{10}{l}{Max Coupling Uniform Diffusion} \\
& ELBO & Denoiser & 32.96 & 112.64 & 64.71 & 117.21 & 110.78 & 65.82 & 51.40 \\
& ELBO & LOO-Denoiser & \bestcat{32.21} & \bestcat{109.97} & \bestcat{62.44} & \bestcat{112.70} & \bestcat{106.53} & \bestcat{61.20} & \bestcat{49.41} \\
\midrule
\rowcolor{blue!10}
\multicolumn{10}{l}{Masked Diffusion} \\
& ELBO & Denoiser & \textbf{22.90} & \underline{64.57} & 48.79 & \underline{73.06} & \underline{74.99} & \underline{43.13} & \underline{35.16} \\
\bottomrule
\end{tabular}
}}
\end{table}
\paragraph{LOO denoiser vs denoiser.} We first isolate the effect of the parameterization by training UDMs with either the denoiser or the LOO denoiser parameterization under matched architectures. We consider both cross-entropy and continuous-time ELBO training for both parameterizations, using the standard UDM ELBO of \cite{schiff2024simple} and the Maximal Coupling\footnote{This model is related to the DDIM sampler for UDMs derived in \cite[Appendix A]{song2021ddim} and is used to confirm that our findings carry over to different processes. This process is discussed extensively in \Cref{sec:max-coupling}} ELBO derived in \Cref{prop:max-coupling-ctmc-loss}. The configuration ELBO with the denoiser parameterization is the continuous-time analogue of D3PM Uniform trained without the auxiliary cross-entropy term in the
hybrid objective \cite[Equation 5]{austin2021structured}. Although \cite{schiff2024simple} compare against D3PM Uniform, their comparison also changes the training objective: the D3PM baseline is trained with a discrete-time loss, while the plug-in UDM model uses a continuous-time loss, which already improves perplexity; see \cite[Figure 2]{schiff2024simple}. We therefore train both the denoiser and LOO parameterizations with the same continuous-time loss, isolating the effect of the parameterization itself.

The validation curves on LM1B and OWT, on the left plot of \Cref{fig:parameterization-comparison} show that the LOO parameterization improves optimization, with better perplexity throughout the training. The zero-shot results in \Cref{tab:ppl-eval-multidataset_ppl} show that this trend transfers beyond OWT where in the three settings the LOO improves over its denoiser counterpart on almost all the datasets considered.
We then evaluate generation with top-$p$ frontiers, sweeping $p \in [0.8,1.0]$ and plotting Gen-PPL against 1-gram entropy. The right panel of \Cref{fig:parameterization-comparison} shows that the LOO parameterization improves the frontier in the relevant low-entropy regime.
\paragraph{Cross-entropy training.} \Cref{fig:parameterization-comparison} shows that cross-entropy training with the denoiser is the weakest combination, despite its simplicity. Using the LOO parameterization largely removes this drawback: cross-entropy becomes
competitive with ELBO training while remaining simpler to implement. Our practical takeaway is therefore to use the LOO denoiser as the default UDM prediction target, especially with cross-entropy training.
\paragraph{Applying top-$p$ to the LOO denoiser.} In the comparison in \Cref{fig:parameterization-comparison}, top-$p$ is applied to the learned primitive itself: to the LOO denoiser before using \ref{eq:param_plug_in} and to the denoiser before using \ref{eq:param_marginalization}. A useful consequence of our conversion formulas is that one can also start from a denoiser-trained model, convert the denoiser to the LOO representation, apply top-$p$ there, and then sample with the bridge plug-in parameterization. This gives a non-negligible \textbf{improvement of the frontier at no additional cost} when using top-$p$ sampling and reduces the gap with the LOO parameterization, as shown in \Cref{fig:nucleus-application-frontier}.
\begin{figure}[ht!]
    \centering
    \includegraphics[width=0.95\textwidth]{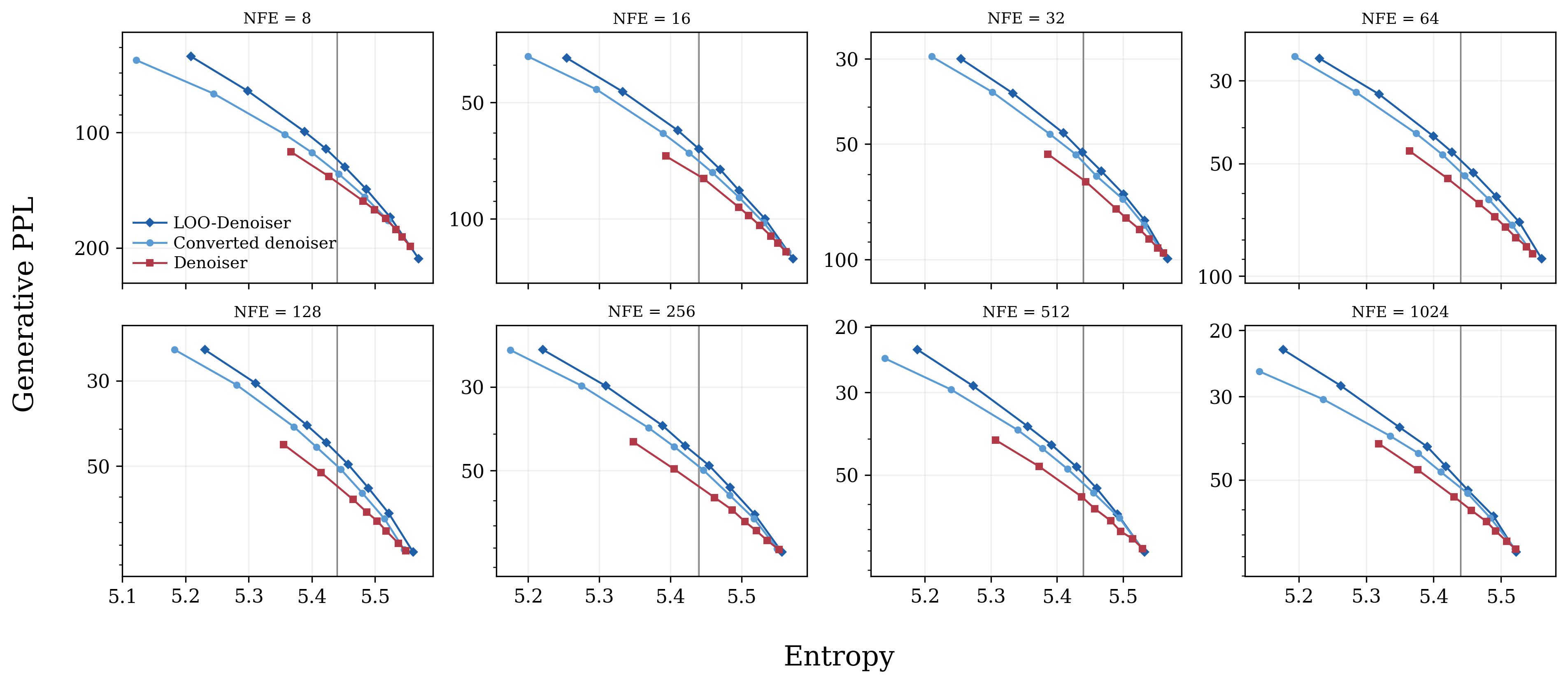}
    \caption{Top-$p$ sampling Gen-PPL frontier, obtained by sweeping ($p \in [0.8, 1.0]$). Top-$p$ sampling is applied to the denoiser, denoiser converted into LOO denoiser and LOO denoiser.} 
    \label{fig:nucleus-application-frontier}
\end{figure}
\paragraph{Predictor-corrector}
We next evaluate the predictor-corrector sampler (\Cref{alg:pc}) described at the end of \Cref{sec:loo}. As detailed in \Cref{app:predictor-corrector}, the LOO parameterization gives access to a corrector step without training an auxiliary model, in contrast to \cite{zhao2025informed}. We use the confidence-based resampling rule of \cite{zhao2025informed}, with $M \in \intset{1}{5}$ corrector steps per predictor step and $k \in \intset{1}{15}$ parallel token updates per corrector step. We show the results for a trained LOO denoiser on the left panel of \Cref{fig:predictor-corrector} and for a trained denoiser converted into a LOO denoiser on the right panel. The resulting predictor-corrector frontiers Pareto-dominate the ancestral temperature frontiers in both cases. In addition to being substantially better, a key advantage is that this correction is cheap: in our implementation, the predictor-corrector sampler has the same runtime as the regular ancestral sampler. For completeness, we also report the comparison with top-$p$ sampling in \Cref{fig:corrector-vs-nucleus}. Thus, the LOO parameterization is useful not only as a training target, but also as a training-free mechanism for improving inference.
\begin{figure}[t]
    \centering
    \begin{subfigure}[t]{0.49\textwidth}
        \centering
        \includegraphics[width=\linewidth]{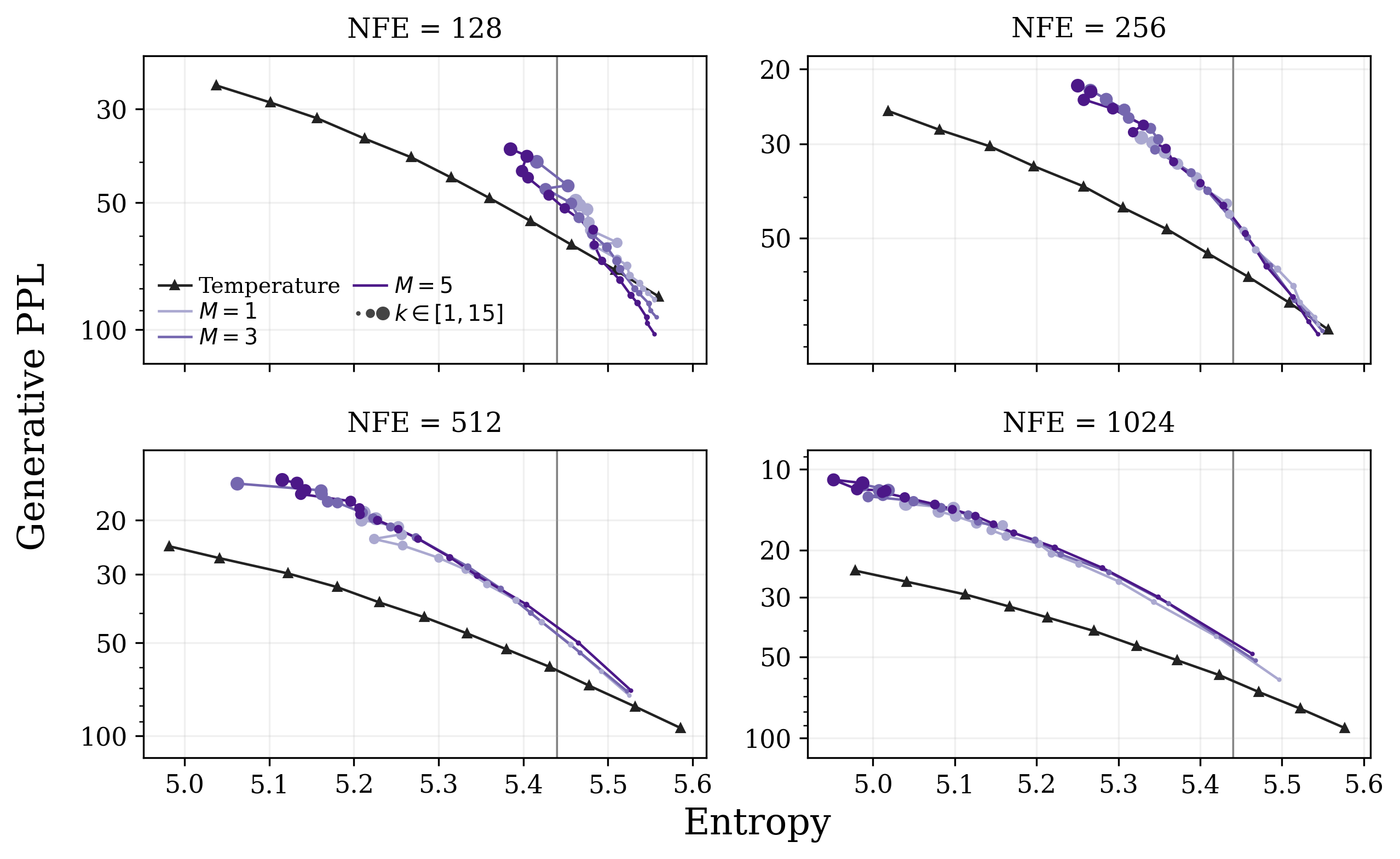}
    \end{subfigure}
    \hfill
    \begin{subfigure}[t]{0.49\textwidth}
        \centering
        \includegraphics[width=\linewidth]{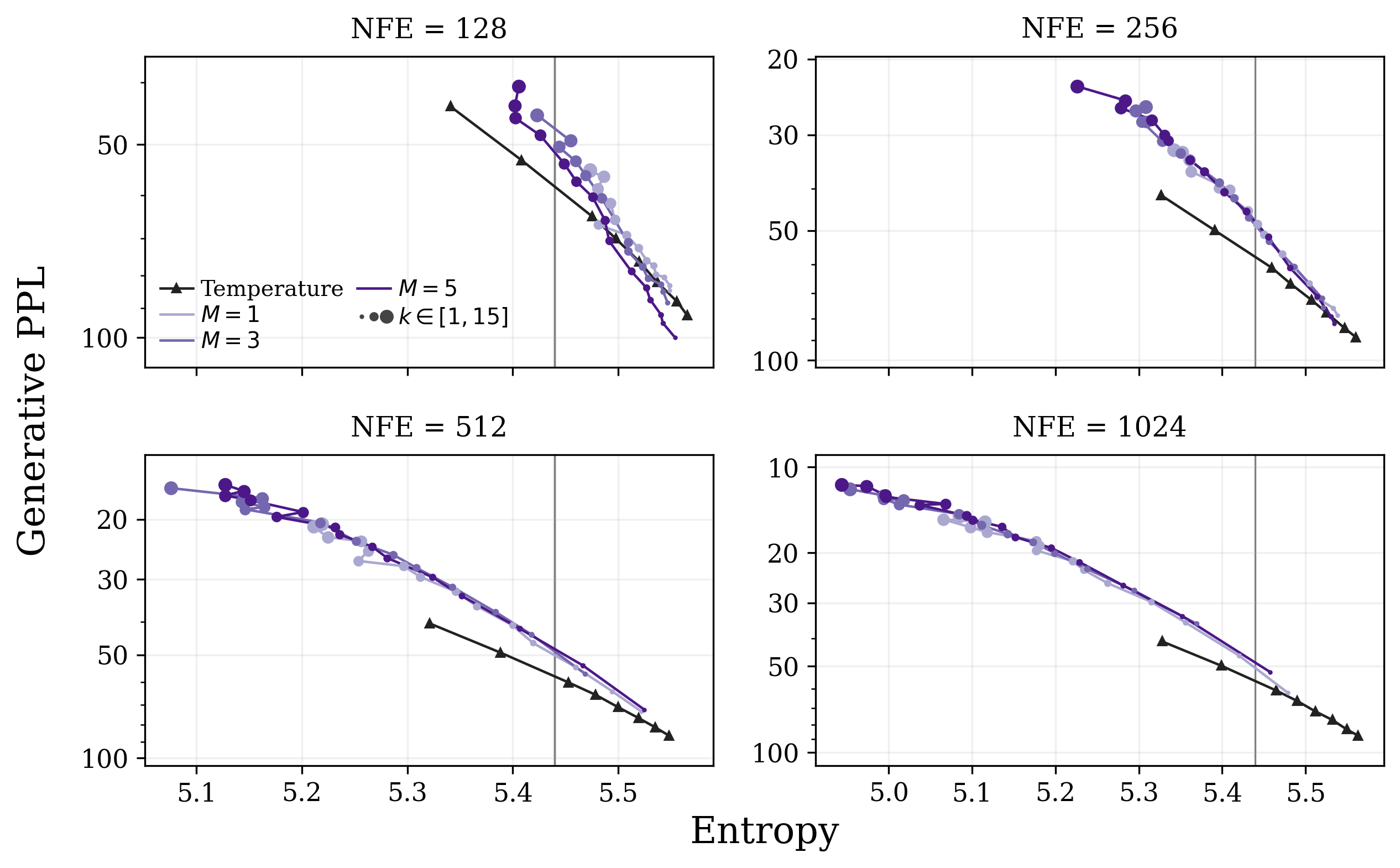}
    \end{subfigure}
    \captionsetup{font=small}
    \caption{Gen-PPL frontiers. Temperature sampling Gen-PPL frontier, obtained by sweeping the temperature in $[0.8, 1.10]$. \textbf{Left:} predictor-corrector sampling using a trained LOO denoiser against ancestral sampling using the same LOO denoiser. \textbf{Right:} predictor-corrector sampling using a trained plain denoiser converted at inference time to a LOO against ancestral sampling using the same plain denoiser. The vertical line corresponds to the dataset entropy}
    \label{fig:sampling-frontiers}
    \label{fig:predictor-corrector}
    \label{fig:loo-converted-denoiser-frontier}
\end{figure}

\subsection{Absorbing-state uniform diffusion}
  We finally evaluate the absorbing-state uniform diffusion model of \Cref{sec:audm}. AUDM uses the same 170M-parameter DiT backbone as the UDM and MDM baselines, and adds a 7M-parameter module to process the auxiliary noise variable $U$, corresponding to a 4.1\% parameter overhead. AUDM gives the best likelihood results among the uniform-diffusion variants in \Cref{tab:ppl-eval-multidataset_ppl}. It improves over UDM on the OWT validation distribution and
  on every zero-shot dataset. It also outperforms MDM on all zero-shot datasets, although MDM remains slightly better on OWT validation perplexity. At sampling time, \Cref{fig:sudoku-audm-frontiers} shows that AUDM remains competitive along the generative frontier, and is often slightly better than MDM, uniformly across NFEs. Thus, the absorbing-state construction preserves the marginals of UDMs while improving its likelihood.

The Sudoku results in \Cref{fig:sudoku-accuracy-vs-nfe} give a smaller-scale check of the same conclusions outside language modeling. The best solve rates are obtained by UDMs trained with the leave-one-out target, with cross-entropy and ELBO both reaching nearly perfect accuracy. UDM also outperforms MDM on this task. We see that AUDM improves over MDM even though they share structural properties, and interestingly, the ReAUDM construction improves over AUDM even though they use the same checkpoint and the sampling scheme does not add any additional cost or learned component.

\begin{figure}[t]
    \centering
    \begin{subfigure}[t]{0.48\textwidth}
        \centering
        \includegraphics[width=\linewidth]{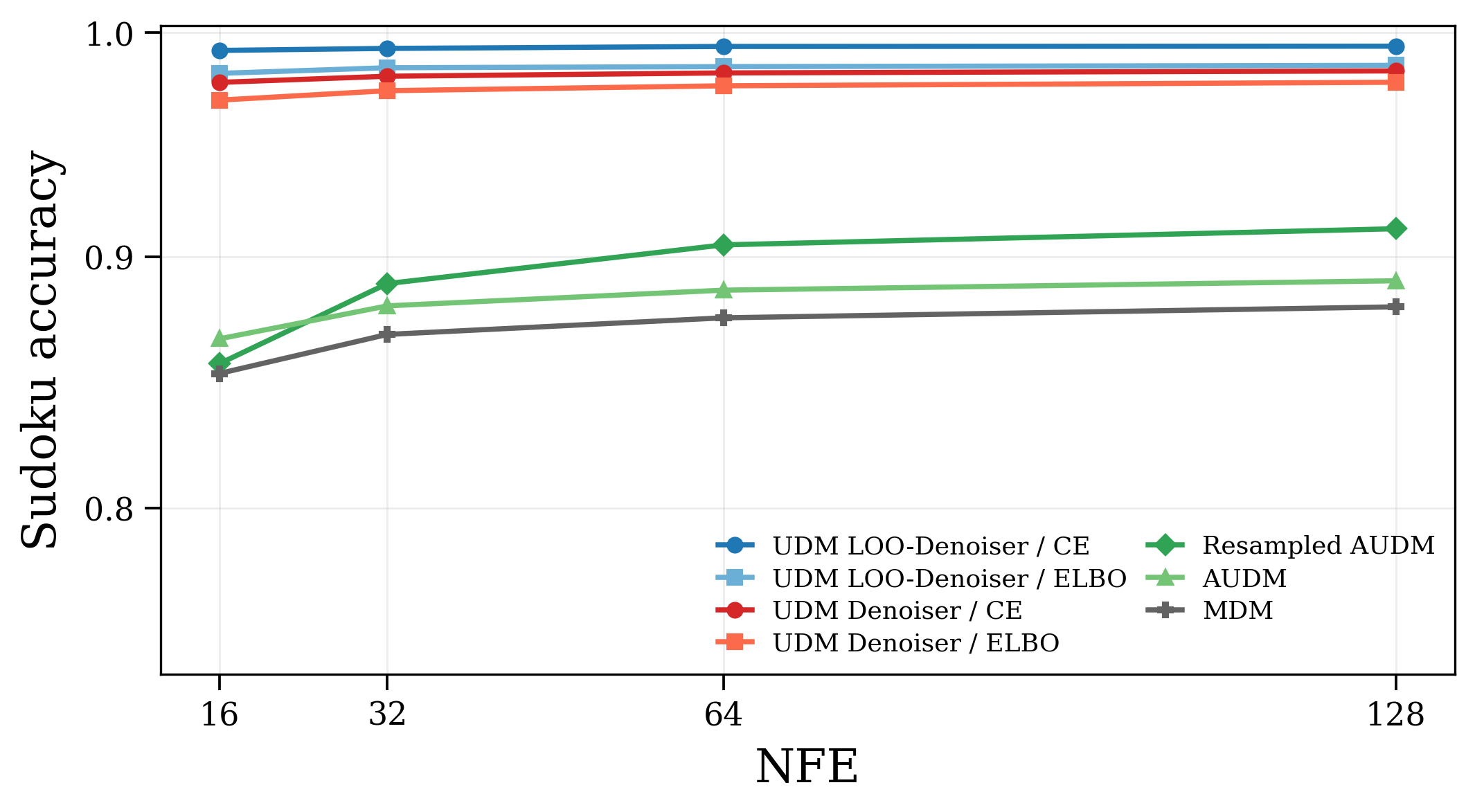}
    \end{subfigure}
    \hfill
    \begin{subfigure}[t]{0.48\textwidth}
        \centering
        \includegraphics[width=\linewidth]{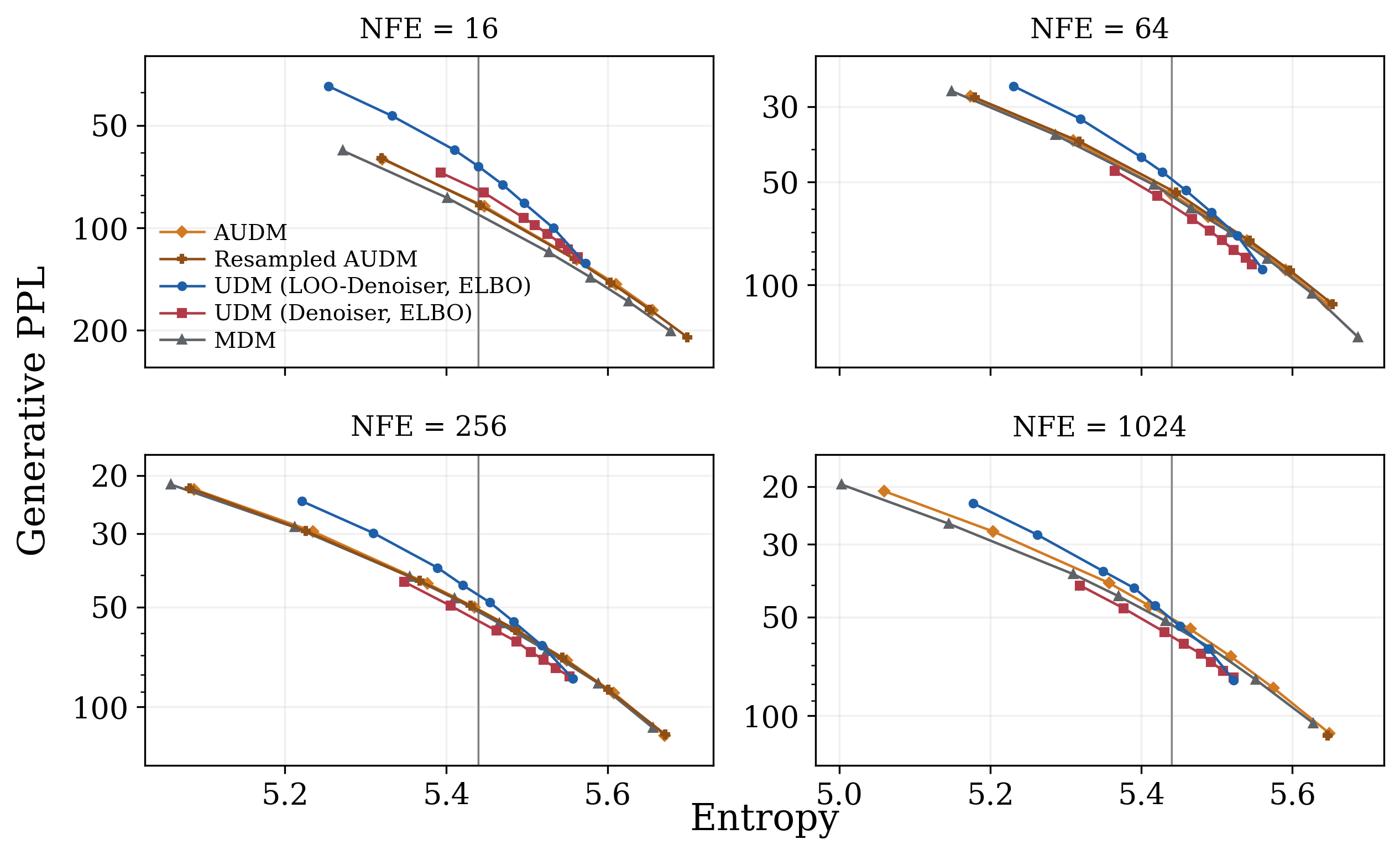}
    \end{subfigure}
    \captionsetup{font=small}
    \caption{\textbf{Left:} Sudoku solve rate as a function of NFE for the best learning rate of each method. \textbf{Right:} AUDM, resampled AUDM, UDM, and MDM Gen-PPL frontier. The vertical line corresponds to the dataset entropy.}
    \label{fig:sudoku-audm-frontiers}
    \label{fig:sudoku-accuracy-vs-nfe}
    \label{fig:genppl_vs_entropy}
\end{figure}

%% file: body/conclusion.tex
\section{Conclusion}

We revisited uniform discrete diffusion from two complementary perspectives: the parameterization of the reverse kernel and the design of the forward process. On the parameterization side, we showed that the bridge plug-in objective of UDMs is optimized not by the standard denoising posterior but by a leave-one-out posterior. We derived exact conversions between the denoiser, the leave-one-out predictor, and the score, which separate the choice of training target from the choice of sampling parameterization. In practice, this viewpoint leads to a simple cross-entropy objective for the leave-one-out predictor that improves over the standard cross-entropy parameterization, a way to apply top-$p$ or temperature in the most effective representation, and a predictor-corrector sampler that can be obtained from a trained model without any auxiliary network.

On the process-design side, we introduced an absorbing-state reformulation of uniform diffusion, together with a resampled version that recovers the full UDM reverse chain. By conditioning on the transition time, we showed that it is even possible to re-use the masked diffusion denoiser while keeping UDM joint law. These constructions show that masked-diffusion-like denoisers, carry-over structure, and remasking can be transferred to a model that remains faithful to the uniform-diffusion viewpoint. Empirically, leave-one-out parameterizations consistently improve optimization and the generative frontier of UDMs, while AUDM and ReAUDM match or surpass masked diffusion on several frontier evaluations while improving likelihood within the uniform-diffusion family. Altogether, our results indicate that the practical gap between masked and uniform diffusion is driven less by the corruption marginals themselves than by the learned representation and the inference scheme built on top of them.

\textbf{Limitations and future work.\,}\
Our evidence for the advantage of the leave-one-out target remains mainly empirical, and a theoretical explanation for its improved generative frontier remains open. In addition, although AUDM and ReAUDM clarify how masked-style structure can be transferred to uniform diffusion and substantially improve likelihood, the best generative frontiers are still obtained by leave-one-out UDMs. Improving these absorbing-state models through better parameterizations or sampling schemes is therefore an important direction for future work. Finally, when training the AUDM denoiser, we condition on one-hot absorbing tokens. A natural extension would be to replace this with a simplex-valued conditioning vector, possibly with mass on the mask, so as to interpolate within a single network between AUDM-type posteriors, the standard UDM denoiser obtained at $\mathbf{1}/K$, and more general hybrid corruption mechanisms.

\paragraph{Acknowledgments. } The authors thank Jawad Chemaou for carefully proofreading the paper and providing helpful comments.

%% file: appendix/loo_denoiser.tex
\section{Leave-one-out (LOO) denoiser}
\label{appendix:loo}
\subsection{Proof of \Cref{prop:loo}}
\label{sec:proof-crefprop:loo}
The proof of \Cref{prop:loo} relies on the \emph{plug-in} of the prediction of the neural network into the bridge. This assumes a certain extension of the bridge, initially defined for the one-hots to the simplex. To this end, we consider the following assumption: 
\begin{hypA}
\label{ass:simplex-bayes-bridge}
\emph{Bayes-compatible simplex extension.} 
For every $x_s,x_t \in \msv$, and
$\nu \in \Simplex{\vocab}$ such that $\fw{t\tbar 0}{\nu}{x_t}[\ell]>0$, the bridge extension used in \ref{eq:param_plug_in}
satisfies
\[
\fw{s\tbar 0,t}{\nu,x_t}{x_s}[\ell]
=
\frac{
\fw{t\tbar s}{x_s}{x_t}[\ell]\,
\fw{s\tbar 0}{\nu}{x_s}[\ell]
}{
\fw{t\tbar 0}{\nu}{x_t}[\ell]
}
\eqsp.
\]

\end{hypA}

\newcounter{savedpropositionloo}
\setcounter{savedpropositionloo}{\value{proposition}}
\begingroup
\setcounterref{proposition}{prop:loo}
\addtocounter{proposition}{-1}
\renewcommand{\theHproposition}{appendix.loo.\arabic{proposition}}

\begin{mdframed}[style=propFrame]
\begin{proposition}[Generalized form]
For the \ref{eq:param_plug_in} parameterization, under
\assp{\ref{ass:simplex-bayes-bridge}}, a minimizer of \eqref{eq:kl-sum} is
obtained by setting
$\denoiser{t}{}{\bx_t}[\param_\star] = \loodenoiser{t}{}{\bx_t}$ on the support of
$\pdata{t}{}{}$. In particular, for such $\bx_t$,
\[
\txts \pdata{s\tbar t}{\bx_t}{}[\ell]
=
\fw{s\tbar 0,t}{\loodenoiser{t}{}{\bx_t}^\ell,\bx_t^\ell}{}[\ell]
\eqsp.
\]
Furthermore, this minimizer is unique in the case of UDM.
\end{proposition}
\end{mdframed}
\endgroup

\setcounter{proposition}{\value{savedpropositionloo}}
\begin{proof}
Fix $0<s<t$, $\, \ell \in \intset{1}{\len}$ and $\bx_t \in \msx$ such that
$\pdata{t}{}{\bx_t}>0$. We write
$\nu_\star^\ell \eqdef \loodenoiser{t}{}{\bx_t}^\ell$. We first prove that
\begin{equation}
\label{eq:loo-proof-target}
\pdata{s\tbar t}{\bx_t}{}[\ell]
=
\fw{s\tbar 0,t}{\nu_\star^\ell,\bx_t^\ell}{}[\ell]
\eqsp.
\end{equation}
Since \eqref{eq:kl-sum} is a sum of token-wise KL divergences, this identity implies
that $\denoiser{t}{}{\bx_t}[\param_\star] = \loodenoiser{t}{}{\bx_t}$ is one minimizer. For every $\bx_s^\ell \in \msv$, the definition of the marginal reverse transition and the factorization of the bridge across the dimensions yields
\begin{align*}
\pdata{s\tbar t}{\bx_t}{\bx_s^\ell}[\ell]
&=
\txts \sum_{\bx_s^{-\ell}} \pdata{s\tbar t}{\bx_t}{\bx_s} \\
&=
\txts \sum_{\bx_s^{-\ell},\bx_0}
\fw{s\tbar 0,t}{\bx_0,\bx_t}{\bx_s}\,
\pdata{0\tbar t}{\bx_t}{\bx_0} \\
&=
\txts \sum_{\bx_0^\ell}
\fw{s\tbar 0,t}{\bx_0^\ell,\bx_t^\ell}{\bx_s^\ell}[\ell]\,
\pdata{0\tbar t}{\bx_t}{\bx_0^\ell}[\ell].
\end{align*}
Moreover, for every $(\bx_0^\ell,\bx_t^\ell)$, the product identity
\[
\fw{t\tbar 0}{\bx_0^\ell}{\bx_t^\ell}[\ell] \fw{s\tbar 0,t}{\bx_0^\ell,\bx_t^\ell}{\bx_s^\ell}[\ell]
=
\fw{t\tbar s}{\bx_s^\ell}{\bx_t^\ell}[\ell]\,
\fw{s\tbar 0}{\bx_0^\ell}{\bx_s^\ell}[\ell]
\eqsp.
\]
holds. This is the Bayes' formula if $\fw{t\tbar 0}{\bx_0^\ell}{\bx_t^\ell}[\ell]>0$. Otherwise the left hand side vanishes by the Chapman--Kolmogorov identity and we still get equality. We also write the posterior marginal at position
$\ell$ as
\[
\pdata{0\tbar t}{\bx_t}{\bx_0^\ell}[\ell]
=
\frac{
\pdata{0\tbar t}{\bx_t^{-\ell}}{\bx_0^\ell}[\loo,\ell]\,
\fw{t\tbar 0}{\bx_0^\ell}{\bx_t^\ell}[\ell]
}{
\sum_{\tilde{\bx}_0^\ell \in \msv}
\pdata{0\tbar t}{\bx_t^{-\ell}}{\tilde{\bx}_0^\ell}[\loo,\ell]\,
\fw{t\tbar 0}{\tilde{\bx}_0^\ell}{\bx_t^\ell}[\ell]
}
\eqsp.
\]
Since $x_0 \mapsto \fw{t\tbar 0}{x_0}{\bx^\ell _t}[\ell]$  is linear in its first argument, the denominator is exactly
$\fw{t\tbar 0}{\nu_\star^\ell}{\bx_t^\ell}[\ell]$. It is strictly positive since it is
the conditional probability of $\bx_t^\ell$ given $\bx_t^{-\ell}$ under $\pdata{t}{}{}$ by \Cref{prop:gibbs-conditional},
and the fixed state $\bx_t$ is in the support of $\pdata{t}{}{}$. Substituting the two
previous identities yields
\[
\pdata{s\tbar t}{\bx_t}{\bx_s^\ell}[\ell]
=
\frac{
\fw{t\tbar s}{\bx_s^\ell}{\bx_t^\ell}[\ell]
\sum_{\bx_0^\ell \in \msv}
\fw{s\tbar 0}{\bx_0^\ell}{\bx_s^\ell}[\ell]\,
\pdata{0\tbar t}{\bx_t^{-\ell}}{\bx_0^\ell}[\loo,\ell]
}{
\fw{t\tbar 0}{\nu_\star^\ell}{\bx_t^\ell}[\ell]
}
\eqsp.
\]
Using again the linearity of the forward transition,
\[
\sum_{\bx_0^\ell \in \msv}
\fw{s\tbar 0}{\bx_0^\ell}{\bx_s^\ell}[\ell]\,
\pdata{0\tbar t}{\bx_t^{-\ell}}{\bx_0^\ell}[\loo,\ell]
=
\fw{s\tbar 0}{\nu_\star^\ell}{\bx_s^\ell}[\ell]
\]
and therefore
\[
\pdata{s\tbar t}{\bx_t}{\bx_s^\ell}[\ell]
=
\frac{
\fw{t\tbar s}{\bx_s^\ell}{\bx_t^\ell}[\ell]\,
\fw{s\tbar 0}{\nu_\star^\ell}{\bx_s^\ell}[\ell]
}{
\fw{t\tbar 0}{\nu_\star^\ell}{\bx_t^\ell}[\ell]
}
\eqsp.
\]

By \assp{\ref{ass:simplex-bayes-bridge}}, applied with $\nu=\nu_\star^\ell$ and
$x_t=\bx_t^\ell$, we have
\[
\pdata{s\tbar t}{\bx_t}{\bx_s^\ell}[\ell]
=
\frac{
\fw{t\tbar s}{\bx_s^\ell}{\bx_t^\ell}[\ell]\,
\fw{s\tbar 0}{\nu_\star^\ell}{\bx_s^\ell}[\ell]
}{
\fw{t\tbar 0}{\nu_\star^\ell}{\bx_t^\ell}[\ell]
}
=
\fw{s\tbar 0,t}{\nu_\star^\ell,\bx_t^\ell}{\bx_s^\ell}[\ell],
\]
which is exactly \eqref{eq:loo-proof-target}. 

We now prove uniqueness in the UDM case. Fix $\ell$ and $\bx_t$. The corresponding
token-wise term in \eqref{eq:kl-sum} is minimized if and only if the model marginal
transition coincides with $\pdata{s\tbar t}{\bx_t}{}[\ell]$. Under the plug-in
parameterization, this marginal is of the form
$\fw{s\tbar 0,t}{\nu,\bx_t^\ell}{}[\ell]$ with
$\nu = \denoiser{t}{}{\bx_t}[\param]^\ell \in \Simplex{K}$. It is therefore enough
to prove that, for UDM, the map
\[
\nu \longmapsto \fw{s\tbar 0,t}{\nu,\bx_t^\ell}{}[\ell]
\]
is injective.

Write $\bx_t^\ell = \onehot{k}$. By \eqref{eq:usdm-bridge}, for every
$\nu \in \Simplex{K}$,
\[
\fw{s\tbar 0,t}{\nu,\onehot{k}}{}[\ell]
=
\categorical\!\left(
\cdot;
\frac{
\vocab \alpha_t \nu_k \onehot{k}
+ (\alpha_{t\tbar s}-\alpha_t)\onehot{k}
+ (\alpha_s-\alpha_t)\nu
+ D_{s,t}\one/\vocab
}{
\vocab\alpha_t \nu_k + 1-\alpha_t
}
\right).
\]
Hence, for every $j \neq k$,
\begin{equation}
\label{eq:bridge-injective}
\dotp{\onehot{j}}{\fw{s\tbar 0,t}{\nu,\onehot{k}}{}[\ell]}
=
\frac{
(\alpha_s-\alpha_t)\nu_j + D_{s,t}/\vocab
}{
\vocab\alpha_t \nu_k + 1-\alpha_t
}
\eqsp.
\end{equation}
Summing \eqref{eq:bridge-injective} over $j \neq k$ gives
\[
1-\dotp{\onehot{k}}{\fw{s\tbar 0,t}{\nu,\onehot{k}}{}[\ell]}
=
\frac{
(\alpha_s-\alpha_t)(1-\nu_k) + (\vocab-1)D_{s,t}/\vocab
}{
\vocab\alpha_t \nu_k + 1-\alpha_t
}
\eqsp.
\]
Since $0<s<t$, we have
$\alpha_s-\alpha_t > 0$ and $D_{s,t} > 0$.
The numerator is therefore strictly decreasing in $\nu_k$, whereas the denominator is
strictly increasing in $\nu_k$. The ratio is thus strictly decreasing and determines
$\nu_k$ uniquely from $\fw{s\tbar 0,t}{\nu,\onehot{k}}{}[\ell]$. Once $\nu_k$ is known,
\eqref{eq:bridge-injective} yields, for every $j \neq k$,
\[
\nu_j
=
\frac{
(\vocab\alpha_t\nu_k + 1-\alpha_t)
\dotp{\onehot{j}}{\fw{s\tbar 0,t}{\nu,\onehot{k}}{}[\ell]}
- D_{s,t}/\vocab
}{
\alpha_s-\alpha_t
}
\eqsp.
\]
Hence every coordinate of $\nu$ is uniquely determined by
$\fw{s\tbar 0,t}{\nu,\onehot{k}}{}[\ell]$, so the map is injective. Since
$\nu_\star^\ell = \loodenoiser{t}{}{\bx_t}^\ell$ attains the minimum by the first part
of the proof, it is the unique minimizer in the UDM case.
\end{proof}

Assumption \assp{\ref{ass:simplex-bayes-bridge}} is a condition on the simplex extension used by the
plug-in parameterization. When $\fw{t\tbar 0}{x_0}{x_t}[\ell]>0$, the bridge is defined via the Bayes' formula as
\[
\fw{s\tbar 0,t}{x_0,x_t}{x_s}[\ell]
=
\frac{
\fw{t\tbar s}{x_s}{x_t}[\ell]\,
\fw{s\tbar 0}{x_0}{x_s}[\ell]
}{
\fw{t\tbar 0}{x_0}{x_t}[\ell]
}
\eqsp.
\]
When $\fw{t\tbar 0}{x_0}{x_t}[\ell]=0$, the value of the one-hot bridge must be completed
by convention. This first completion does not determine the second extension to a
simplex-valued first argument, and \assp{\ref{ass:simplex-bayes-bridge}} requires this second
extension to keep the Bayes-ratio form on the simplex.

For the MDM bridge used in \eqref{eq:mdm-bridge}, this compatibility holds on the simplex, where $\dotp{\mask}{\nu}=0$. If $\bx_t^\ell=\onehot{k} \neq \mask$
and $\dotp{\onehot{k}}{\nu}>0$, then
\[
\frac{
\fw{t\tbar s}{\bx_s^\ell}{\onehot{k}}[\ell]\,
\fw{s\tbar 0}{\nu}{\bx_s^\ell}[\ell]
}{
\fw{t\tbar 0}{\nu}{\onehot{k}}[\ell]
}
=
\frac{
\alpha_{t\tbar s}\alpha_s\dotp{\onehot{k}}{\nu}\,
\indic\{\bx_s^\ell=\onehot{k}\}
}{
\alpha_t\dotp{\onehot{k}}{\nu}
}
=
\indic\{\bx_s^\ell=\onehot{k}\},
\]
so the Bayes-ratio expression gives $\categorical(\cdot;\bx_t^\ell)$, which is the first
branch of \eqref{eq:mdm-bridge}. If $\bx_t^\ell=\mask$, then
\[
\frac{
\fw{t\tbar s}{\bx_s^\ell}{\mask}[\ell]\,
\fw{s\tbar 0}{\nu}{\bx_s^\ell}[\ell]
}{
\fw{t\tbar 0}{\nu}{\mask}[\ell]
}
=
\frac{\alpha_s-\alpha_t}{1-\alpha_t}
\dotp{\bx_s^\ell}{\nu}\indic\{\bx_s^\ell\neq\mask\}
+ \frac{1-\alpha_s}{1-\alpha_t}\indic\{\bx_s^\ell=\mask\},
\]
or, equivalently,
\[
\categorical\left(
\cdot;
\frac{\alpha_s-\alpha_t}{1-\alpha_t}\nu
+ \frac{1-\alpha_s}{1-\alpha_t}\mask
\right),
\]
which is the second branch of \eqref{eq:mdm-bridge} with $x_0$ replaced by $\nu$. Therefore \assp{\ref{ass:simplex-bayes-bridge}} is verified and \Cref{prop:loo} holds.

However, the assumption is not automatic for an arbitrary simplex extension. For example, suppose
that, on unsupported visible pairs, one completes the bridge by the rule
$\categorical(\cdot;\frac{\alpha_s-\alpha_t}{1-\alpha_t} x_0+ \frac{1-\alpha_s}{1-\alpha_t} x_t)$ and then extends this rule to $\nu$ as
$\categorical(\cdot;\frac{\alpha_s-\alpha_t}{1-\alpha_t} \nu+\frac{1-\alpha_s}{1-\alpha_t} x_t)$, without the MDM case split. If
$x_t=\onehot{k}$, $\dotp{\onehot{k}}{\nu}>0$, and $j\neq k$, then the
Bayes-ratio extension satisfies
\[
\frac{
\fw{t\tbar s}{\onehot{j}}{\onehot{k}}[\ell]\,
\fw{s\tbar 0}{\nu}{\onehot{j}}[\ell]
}{
\fw{t\tbar 0}{\nu}{\onehot{k}}[\ell]
}
=0,
\]
whereas the naive simplex extension gives
\[
\dotp{\onehot{j}}{\frac{\alpha_s-\alpha_t}{1-\alpha_t}\nu+\frac{1-\alpha_s}{1-\alpha_t}\onehot{k}}
=
\frac{\alpha_s-\alpha_t}{1-\alpha_t}\dotp{\onehot{j}}{\nu}.
\]
Thus \assp{\ref{ass:simplex-bayes-bridge}} fails in general. If the leave-one-out posterior
has positive mass on such a token $\onehot{j}$, plugging it into this naive extension
does not reproduce the exact marginal reverse transition, and the conclusion of
\Cref{prop:loo} can be false. Therefore, depending on the extension that one uses for the bridge, \Cref{prop:loo} may or may not hold. In the paper however, both the bridge considered satisfy \assp{\ref{ass:simplex-bayes-bridge}}. And the assumption stated in the main paper: $\fw{t \tbar 0}{x_0}{x_t}[\ell] > 0$ for all $x_0, x_t \in \msv$, is a sufficient condition for the canonical bridge extension to verify \assp{\ref{ass:simplex-bayes-bridge}}.

Even when \assp{\ref{ass:simplex-bayes-bridge}} holds for the bridge in \eqref{eq:mdm-bridge},
uniqueness of the minimizer is not automatic because the map
$\nu \mapsto \fw{s\tbar 0,t}{\nu,\bx_t^\ell}{}[\ell]$ may not be injective. For instance, the map is not injective for MDMs. Indeed, if
$\bx_t^\ell \neq \mask$, then by \eqref{eq:mdm-bridge},
$\fw{s\tbar 0,t}{\nu,\bx_t^\ell}{}[\ell] = \categorical(\cdot;\bx_t^\ell)$ for every
$\nu$ in the simplex. Therefore the loss does not identify a unique value of
$\nu$ on visible positions. In particular, the denoiser is also a minimizer.

\subsection{Conversion Formula}
\label{app:loo-to-denoiser}
In the next Proposition we provide the conversion formulas for arbitrary $\bpi$. Note that we are only interested in recovering $\pdata{0\tbar t}{\bx_t}{\bx_0}[\ell]$ and $\pdata{0\tbar t}{\bx_t^{- \ell}}{\bx_0^\ell}[\loo, \ell]$ for  $x_t$ satisfying $\pdata{t}{}{\bx_t}>0$ and $\pdata{t}{}{\bx_t^{-\ell}}>0$. Outside this support, the conditional densities may be defined arbitrarily.
\begin{mdframed}[style=propFrame]
\begin{proposition}
  \label{prop:loo_to_denoiser}
  It holds for any $\bx_t$ such that $\pdata{t}{}{\bx_t}>0$,
  \begin{equation}
    \label{eq:loo_to_denoiser}
    \pdata{0\tbar t}{\bx_t}{\bx_0^\ell}[\ell] =\pdata{t}{}{\bx_t^{-\ell}}
\fw{t\tbar 0}{\bx_0^\ell}{\bx_t^\ell}[\ell]\,
\pdata{0\tbar t}{\bx_t^{- \ell}}{\bx_0^\ell}[\loo, \ell]/  \pdata{t}{}{\bx_t} \eqsp.
  \end{equation}
Conversely, suppose that $\fw{t\tbar 0}{\bx_0^\ell}{\bx_t^\ell}[\ell] > 0$ for any $\bx_0$ and $\bx_t$, it holds for any $\bx_t^{-\ell}$, $\pdata{t}{}{\bx_t^{-\ell}}>0$,
  \begin{equation}
    \pdata{0\tbar t}{\bx_t^{- \ell}}{\bx_0^\ell}[\loo, \ell] = \frac{\pdata{t}{}{\bx_t}\pdata{0\tbar t}{\bx_t}{\bx_0^\ell}[\ell] }{\pdata{t}{}{\bx_t^{-\ell}}
\fw{t\tbar 0}{\bx_0^\ell}{\bx_t^\ell}[\ell]} \eqsp.
  \end{equation}
\end{proposition}
\end{mdframed}

\begin{proof}

By definition, for every $\ell \in \intset{1}{L}$,
every $\bx_t \in \msx$, and every $\bx_0^\ell \in \msv$, we have
\[
\pdata{0\tbar t}{\bx_t}{\bx_0^\ell}[\ell]
=
\frac{
\sum_{\bx_0^{-\ell}} \pdata{0}{}{\bx_0}
\fw{t\tbar 0}{\bx_0^\ell}{\bx_t^\ell}[\ell]
\prod_{j \neq \ell} \fw{t\tbar 0}{\bx_0^j}{\bx_t^j}[j]
}{
\pdata{t}{}{\bx_t}
}.
\]
Since $\fw{t\tbar 0}{\bx_0^\ell}{\bx_t^\ell}[\ell]$ does not depend on $\bx_0^{-\ell}$,
this can be rewritten as
\[
\pdata{0\tbar t}{\bx_t}{\bx_0^\ell}[\ell]
=
\frac{
\pdata{t}{}{\bx_t^{-\ell}}
\fw{t\tbar 0}{\bx_0^\ell}{\bx_t^\ell}[\ell]
\pdata{0\tbar t}{\bx_t^{-\ell}}{\bx_0^\ell}[\loo,\ell]
}{
\pdata{t}{}{\bx_t}
}
\propto
\fw{t\tbar 0}{\bx_0^\ell}{\bx_t^\ell}[\ell]\,
\pdata{0\tbar t}{\bx_t^{-\ell}}{\bx_0^\ell}[\loo,\ell].
\]
This is the conversion identity stated in \Cref{eq:udm-loo-to-denoiser}. In other
words, the denoiser is obtained from the leave-one-out posterior by reinserting the
local likelihood term
$\fw{t\tbar 0}{\bx_0^\ell}{\bx_t^\ell}[\ell]$ associated with the observation at
position $\ell$.

Conversely, assume that for the considered $(t,\bx_t,\ell)$, we have
$\fw{t\tbar 0}{\bx_0^\ell}{\bx_t^\ell}[\ell] > 0$ for every $\bx_0^\ell \in \msv$.
Dividing the previous identity by $\fw{t\tbar 0}{\bx_0^\ell}{\bx_t^\ell}[\ell]$ and
normalizing over $\bx_0^\ell$ yields
\begin{equation}
\label{eq:loo-from-denoiser-general}
\pdata{0\tbar t}{\bx_t^{-\ell}}{\bx_0^\ell}[\loo,\ell]
=
\frac{
\pdata{0\tbar t}{\bx_t}{\bx_0^\ell}[\ell] \big/ \fw{t\tbar 0}{\bx_0^\ell}{\bx_t^\ell}[\ell]
}{
\sum_{\tilde{\bx}_0^\ell \in \msv}
\pdata{0\tbar t}{\bx_t}{\tilde{\bx}_0^\ell}[\ell] \big/ \fw{t\tbar 0}{\tilde{\bx}_0^\ell}{\bx_t^\ell}[\ell]
}
\eqsp.
\end{equation}

\end{proof}

Therefore, the conversion from the denoiser to the leave-one-out posterior is
available exactly when the forward has full support. In UDM this condition is
satisfied for every $t>0$, so the two representations can always be converted into
one another. In MDM, by contrast, the condition fails for unmasked positions. If
$\bx_t^\ell = \mask$, the likelihood is constant in $\bx_0^\ell$, so the two
representations coincide. If $\bx_t^\ell \neq \mask$, the denoiser collapses to the
Dirac mass at $\bx_t^\ell$, and the leave-one-out posterior cannot be recovered
from it.

For UDM, these same relations can be rewritten in matrix form. Since
\[
\fw{t\tbar 0}{\bx_0^\ell}{\bx_t^\ell}[\ell]
=
\alpha_t \dotp{\bx_t^\ell}{\bx_0^\ell} + \frac{1-\alpha_t}{\vocab},
\]
the first identity above yields
\begin{equation}
\label{eq:loo-to-denoiser-uniform}
\pdata{0\tbar t}{\bx_t}{}[\ell]
= \categorical\left(
\frac{
(1-\alpha_t)\loodenoiser{t}{}{\bx_t}^\ell
+
\vocab \alpha_t \dotp{\bx_t^\ell}{\loodenoiser{t}{}{\bx_t}^\ell}\bx_t^\ell
}{
1-\alpha_t+\vocab \alpha_t \dotp{\bx_t^\ell}{\loodenoiser{t}{}{\bx_t}^\ell}
}\right)
\eqsp.
\end{equation}
Conversely, the inverse relation can be written as
\begin{equation}
  \label{eq:denoiser-to-loo}
\loodenoiser{t}{}{\bx_t}^\ell
=
\frac{
\left(1+(\vocab-1)\alpha_t\right)\matrixdenoiser{t}{}{\bx_t}^\ell
-
\vocab \alpha_t \dotp{\bx_t^\ell}{\matrixdenoiser{t}{}{\bx_t}^\ell}\bx_t^\ell
}{
1+(\vocab-1)\alpha_t-\vocab \alpha_t \dotp{\bx_t^\ell}{\matrixdenoiser{t}{}{\bx_t}^\ell}
}
\eqsp.
\end{equation}
Where $\matrixdenoiser{t}{}{\bx_t}^\ell = \pdata{0\tbar t}{\bx_t}{}[\ell]$.

For MDM, the conversion is explicit only on the support of the forward process. If
$\bx_t^\ell = \mask$, then $\fw{t\tbar 0}{\bx_0^\ell}{\mask}[\ell] = 1-\alpha_t$ for
every $\bx_0^\ell \in \msv$, so
\[
\pdata{0\tbar t}{\bx_t}{}[\ell]
=
\categorical(\loodenoiser{t}{}{\bx_t}^\ell)
\eqsp.
\]
If instead $\bx_t^\ell \neq \mask$, then $\fw{t\tbar 0}{\bx_0^\ell}{\bx_t^\ell}[\ell] = \alpha_t \indic\{\bx_0^\ell=\bx_t^\ell\}$, hence
\[
\pdata{0\tbar t}{\bx_t}{}[\ell]
=
\categorical(\bx_t^\ell)
\eqsp.
\]
Therefore, on unmasked positions, the denoiser no longer contains enough information
to reconstruct the leave-one-out posterior, so the inverse formula is not available.

\begin{remark}
    When these conversion formulas hold in both directions, as they do for UDM, the uniqueness of the denoiser as a minimizer of the loss is equivalent to the uniqueness of the leave-one-out denoiser as a minimizer of the loss.
\end{remark}

\subsection{Score and leave-one-out denoiser}
\label{apdx-sec:score-loo}
Recall that the score matrix $\score{t}{}{\bx_t} \in \rset^{\len \times \vocab}$ is defined by
\[
\dotp{\by^\ell}{\score{t}{}{\bx_t}^\ell}
=
\pdata{t}{}{\by} \big/ \pdata{t}{}{\bx_t},
\]
for every $\by \in \msx$ such that $\by^{-\ell} = \bx_t^{-\ell}$. This quantity appears naturally in the CTMC time-reversal; see \Cref{sec:ctmc}. By \cite[Proposition 1]{campbell2022continuous}, it is related to the denoiser through
\begin{equation}
\label{eq:score-identity}
 \dotp{\by^\ell}{\score{t}{}{\bx_t}^\ell}
=
\sum_{\bx_0^\ell}
\frac{\fw{t\tbar 0}{\bx_0^\ell}{\by^\ell}[\ell]}{\fw{t\tbar 0}{\bx_0^\ell}{\bx_t^\ell}[\ell]}
\pdata{0\tbar t}{\bx_t}{\bx_0^\ell}[\ell]
\eqsp.
\end{equation}
This already shows that the score can be parameterized from the denoiser. The same quantity can also be expressed directly in terms of the leave-one-out denoiser:
\begin{equation}
\label{eq:loo-to-score-general}
\dotp{\by^\ell}{\score{t}{}{\bx_t}^\ell}
=
\frac{
\fw{t\tbar 0}{\loodenoiser{t}{}{\bx_t}^\ell}{\by^\ell}[\ell]
}{
\fw{t\tbar 0}{\loodenoiser{t}{}{\bx_t}^\ell}{\bx_t^\ell}[\ell]
}\eqsp,
\end{equation}
for every $\by \in \msx$ such that $\by^{-\ell} = \bx_t^{-\ell}$.

\begin{proof}
By definition of the score, for every $\by \in \msx$ such that $\by^{-\ell} = \bx_t^{-\ell}$,
\[
\dotp{\by^\ell}{\score{t}{}{\bx_t}^\ell}
=
\frac{\pdata{t}{}{\by}}{\pdata{t}{}{\bx_t}}.
\]
Expanding the numerator using the factorization of the forward process gives
\[
\pdata{t}{}{\by}
=
\sum_{\bx_0} \pdata{0}{}{\bx_0} \fw{t\tbar 0}{\bx_0}{\by}
=
\sum_{\bx_0} \pdata{0}{}{\bx_0}
\fw{t\tbar 0}{\bx_0^\ell}{\by^\ell}[\ell]
\prod_{j \neq \ell} \fw{t\tbar 0}{\bx_0^j}{\bx_t^j}[j].
\]
Summing first over $\bx_0^{-\ell}$, we obtain
\[
\pdata{t}{}{\by}
=
\sum_{\bx_0^\ell \in \msv}
\fw{t\tbar 0}{\bx_0^\ell}{\by^\ell}[\ell]
\sum_{\bx_0^{-\ell}} \pdata{0}{}{\bx_0}
\prod_{j \neq \ell} \fw{t\tbar 0}{\bx_0^j}{\bx_t^j}[j].
\]
Using the definition of the leave-one-out posterior, the inner sum can be written as
\[
\sum_{\bx_0^{-\ell}} \pdata{0}{}{\bx_0}
\prod_{j \neq \ell} \fw{t\tbar 0}{\bx_0^j}{\bx_t^j}[j]
=
\pdata{t}{}{\bx_t^{-\ell}}
\pdata{0\tbar t}{\bx_t^{-\ell}}{\bx_0^\ell}[\loo,\ell].
\]
Therefore,
\[
\pdata{t}{}{\by}
=
\pdata{t}{}{\bx_t^{-\ell}}
\sum_{\bx_0^\ell \in \msv}
\fw{t\tbar 0}{\bx_0^\ell}{\by^\ell}[\ell]
\pdata{0\tbar t}{\bx_t^{-\ell}}{\bx_0^\ell}[\loo,\ell].
\]
Since the map $\nu \mapsto \fw{t\tbar 0}{\nu}{\by^\ell}[\ell]$ is affine in its first argument, this becomes
\[
\pdata{t}{}{\by}
=
\pdata{t}{}{\bx_t^{-\ell}}
\fw{t\tbar 0}{\loodenoiser{t}{}{\bx_t}^\ell}{\by^\ell}[\ell].
\]
Taking $\by = \bx_t$ yields in the same way
\[
\pdata{t}{}{\bx_t}
=
\pdata{t}{}{\bx_t^{-\ell}}
\fw{t\tbar 0}{\loodenoiser{t}{}{\bx_t}^\ell}{\bx_t^\ell}[\ell].
\]
Dividing the two equations gives
\[
\dotp{\by^\ell}{\score{t}{}{\bx_t}^\ell}
=
\frac{
\fw{t\tbar 0}{\loodenoiser{t}{}{\bx_t}^\ell}{\by^\ell}[\ell]
}{
\fw{t\tbar 0}{\loodenoiser{t}{}{\bx_t}^\ell}{\bx_t^\ell}[\ell]
},
\]
which is exactly \eqref{eq:loo-to-score-general}.
\end{proof}

\subsection{Categorical ratio matching}
\label{sec:categorical_ratio_matching}

In this section we review \emph{categorical ratio matching} as introduced in \cite{sun2023scorebased}. For $\bx_t,\by$ that differ only at position $\ell$, the concrete score used in the main text satisfies
\begin{equation}
\label{eq:crm-score-conditional}
\langle \by^\ell, \score{t}{}{\bx_t}^\ell \rangle
=
\frac{\pdata{t}{}{\by}}{\pdata{t}{}{\bx_t}}
=
\frac{\pdata{t}{\bx_t^{-\ell}}{\by^\ell}[\ell]}
{\pdata{t}{\bx_t^{-\ell}}{\bx_t^\ell}[\ell]}
\eqsp.
\end{equation}
Thus it is enough to learn the one-position conditional marginal matrix $\bm{p}^{*} _t$ which is such that $\bm{p}^{*} _t(\bx_t) \in \rset^{\len \times \vocab}$ and $\langle x, \bm{p}^{*} _t(\bx_t)^\ell \rangle = \pdata{t}{\bx^{-\ell} _t}{x}[\ell]$ since the score is recovered by probing this matrix at $\by^\ell$ and $\bx_t^\ell$ and taking the ratio. In its population form, \emph{categorical ratio matching} (CRM) trains a network $\bm{p}_\param$ by minimizing
\begin{align}
\label{eq:crm-population-loss}
\mathcal{L}^{\mathrm{CRM}}_t(\param)
& =
\sum_{\ell = 1}^\len
\pE_{\pdata{t}{}{}}
\left[
\kldivergence{
\pdata{t}{X_t^{-\ell}}{}[\ell]
}{
\categorical(\bm{p}_\param(X_t, t)^\ell)
}
\right] \\ 
& = - \sum_{\ell = 1}^\len \sum_{\bx_t, \tilde\bx^\ell _t} \log\, \langle \tilde\bx^\ell _t, \bm{p}_\param (\bx_t, t)^\ell \rangle \, \pdata{t}{\bx^{-\ell} _t}{\tilde\bx^\ell _t}[\ell] \pdata{t}{}{\bx_t} + C \eqsp.
\end{align}
The minimizer of \eqref{eq:crm-population-loss} is therefore $\bm{p}_\param= \bm{p}^* _t$. This loss is however intractable as we cannot sample from $\pdata{t}{\bx^{-\ell} _t}{}[\ell]$ for any $\bx^{-\ell} _t$. By assuming that the network $\bm{p}_\param$ is leave-one-out; \emph{i.e.} the output $\bm{p}_\param(\bx_t, t)^\ell$ is independent of the input $\bx^\ell _t$, the previous loss then becomes equal, up to a constant, to 
\begin{equation}
\label{eq:crm-sampled-loss}
\widetilde{\mathcal{L}}^{\mathrm{CRM}}_t(\param)
=
-
\sum_{\ell = 1}^\len
\pE_{\pdata{t}{}{}}
\left[
\log \, \dotp{X_t^\ell}{\bm{p}_\param(X_t, t)^\ell}
\right]
\eqsp.
\end{equation}
Without the leave-one-out assumption, this second loss is minimized when $\bm{p}_\param(\bx_t, t)^\ell=\bx_t^\ell$. 

\cite{sun2023scorebased} also considers a LOO denoiser parameterization which follows from \Cref{prop:gibbs-conditional}. 
 Therefore, if $\smash{\denoiser{t}{}{\bx_t}[\param]}$ is a model for the leave-one-out denoiser, one may use the parameterization
\begin{equation}
\label{eq:crm-loo-parameterization}
\bm{p}_\param(\bx_t, t) = \alpha_t \denoiser{t}{}{\bx_t}[\param] + (1-\alpha_t)\bpi^\ell
\eqsp.
\end{equation}
The same restriction is still needed after this reparameterization. If $\denoiser{t}{}{\bx_t}[\param]^\ell$ may depend on $\bx_t^\ell$, the minimizer of \eqref{eq:crm-sampled-loss} maximizes the probability of $\bx_t$ under $\categorical(\bp_\param(\bx_t, t))$, hence at optimum $\denoiser{t}{}{\bx_t}[\param]^\ell=\bx_t^\ell$ and $\bm{p}_\param(\bx_t,t)^\ell=\alpha_t\bx_t^\ell+(1-\alpha_t)\bpi^\ell$. 
If instead the admissible functions satisfy $\denoiser{t}{}{\bx_t}[\param]^\ell=\denoiser{t}{}{\tilde{\bx}_t}[\param]^\ell$ whenever $\bx_t^{-\ell}=\tilde{\bx}_t^{-\ell}$, then the unique minimizer is $\denoiser{t}{}{\bx_t}[\param]=\loodenoiser{t}{}{\bx_t}$. 

\subsection{Conditional marginal in terms of the LOO denoiser}
\label{sec:proof_gibbs}

\begin{mdframed}[style=propFrame]
\begin{proposition}
\label{prop:gibbs-conditional}
  For every $\ell \in \intset{1}{\len}$, it holds
  \begin{equation}
    \label{eq:gibbs-conditional}
    \pdata{t}{\bx_t^{-\ell}}{}[\ell] = \categorical\big(\alpha_t \loodenoiser{t}{}{\bx_t}^\ell + (1-\alpha_t)\bpi^\ell\big) \eqsp.
  \end{equation}
\end{proposition}
\end{mdframed}

\begin{proof}

Fix $\ell \in \intset{1}{\len}$ and let
$\by \in \msx$ satisfy $\by^{-\ell}=\bx_t^{-\ell}$. Expanding the marginal of $\by$
and summing first over $\bx_0^{-\ell}$ gives
\[
\pdata{t}{}{\by}
=
\sum_{\bx_0^\ell \in \msv}
\fw{t\tbar 0}{\bx_0^\ell}{\by^\ell}[\ell]
\sum_{\bx_0^{-\ell}}
\pdata{0}{}{\bx_0}
\prod_{j \neq \ell} \fw{t\tbar 0}{\bx_0^j}{\bx_t^j}[j].
\]
By definition of the leave-one-out posterior,
\[
\sum_{\bx_0^{-\ell}}
\pdata{0}{}{\bx_0}
\prod_{j \neq \ell} \fw{t\tbar 0}{\bx_0^j}{\bx_t^j}[j]
=
\pdata{t}{}{\bx_t^{-\ell}}
\pdata{0\tbar t}{\bx_t^{-\ell}}{\bx_0^\ell}[\loo,\ell].
\]
Therefore
\[
\pdata{t}{}{\by}
=
\pdata{t}{}{\bx_t^{-\ell}}
\sum_{\bx_0^\ell \in \msv}
\fw{t\tbar 0}{\bx_0^\ell}{\by^\ell}[\ell]
\pdata{0\tbar t}{\bx_t^{-\ell}}{\bx_0^\ell}[\loo,\ell].
\]
Using the affine extension of the forward kernel in its first argument, this becomes
\[
\pdata{t}{}{\by}
=
\pdata{t}{}{\bx_t^{-\ell}}
\fw{t\tbar 0}{\loodenoiser{t}{}{\bx_t}^\ell}{\by^\ell}[\ell].
\]
Dividing by $\pdata{t}{}{\bx_t^{-\ell}}$ yields
\[
\pdata{t}{\bx_t^{-\ell}}{\by^\ell}[\ell]
=
\fw{t\tbar 0}{\loodenoiser{t}{}{\bx_t}^\ell}{\by^\ell}[\ell].
\]
Taking $\by=\bx_t$ gives \Cref{prop:gibbs-conditional}.
\end{proof}

\subsection{An example of leave-one-out and denoiser differences}
\label{app:loo-vs-denoiser}

In order to confirm some intuition on the difference between the leave-one-out and the denoiser, we train a uniform diffusion model with the \ref{eq:param_plug_in} parameterization that reaches a leave-one-out denoiser at optimality and a model with the \ref{eq:param_marginalization} parameterization that recovers a denoiser. We then plot a heatmap corresponding to the probabilities output by the leave-one-out and the denoiser before passing them into their respective updates during sampling. We show the results in \Cref{fig:mnist_loo_vs_denoiser}.

In this case, the difference is mainly on the edges of the digits, where the denoiser is more confident than the leave-one-out posterior. This is because, on the edges of the digits, the local observation $\bx_t^\ell$ is comparatively more informative about $\bx_0^\ell$ than on the inner pixels where the value of $\bx_0^\ell$ is already determined by its neighbors.

However, this difference varies with time, for large $t$, so at the start of sampling, \eqref{eq:loo-to-denoiser-uniform} shows that the denoiser smoothes the prediction of the loo by convolving them with the marginal of the $\ell$-th token. This blurring effect is observed in practice earlier in the sampling.

\begin{figure}[t]
    \centering
    \includegraphics[width=\linewidth]{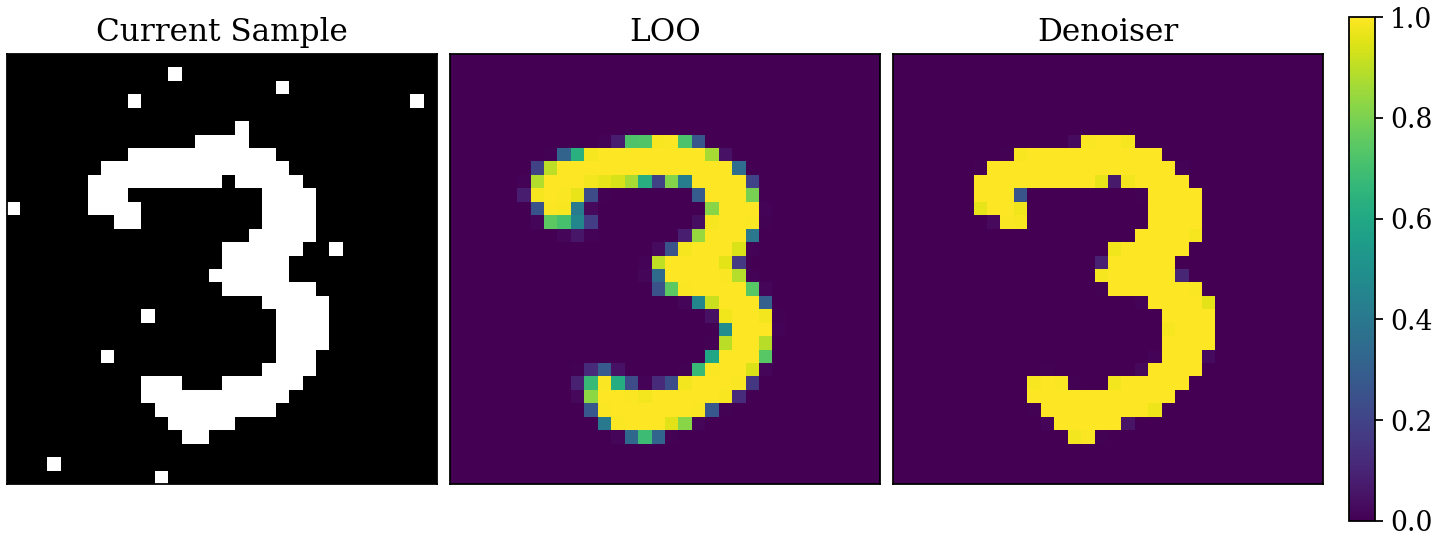}
    \caption{Comparison between the leave-one-out posterior $\loodenoiser{t}{}{\bx_t}$ and the denoising posterior $\matrixdenoiser{t}{}{\bx_t}$.}
    \label{fig:mnist_loo_vs_denoiser}
\end{figure}


%% file: appendix/audm.tex
\section{Absorbing State Uniform Diffusion}
\label{apdx:audm}
\subsection{AUDM NELBO}
\begin{algorithm}[t]
\caption{AUDM sampler}
\label{alg:audm-sampler}
\begin{algorithmic}
\Require grid $0=\tk{0}<\cdots<\tk{n}=1$
\State Draw $U \sim \pnoise{}{}{}$
\State Draw $X_{\tk{n}} \sim \pdata{\tk{n}}{U}{\,\cdot}$
\For{$i=n{-}1,\dots,0$}
    \State $\hat{\bx}_{0,i} \gets \hat{\bx}^\param_0(X_{\tk{i+1}},\tk{i+1};U)$
    \State \emph{In parallel over $\ell$:}
    \State \quad if $X^\ell_{\tk{i+1}} \neq U^\ell$:\quad $X^\ell_{\tk{i}} \gets X^\ell_{\tk{i+1}}$
    \State \quad else:\quad $X^\ell_{\tk{i}} \sim \categorical\!\left(\frac{\alpha_{\tk{i}}-\alpha_{\tk{i+1}}}{1-\alpha_{\tk{i+1}}}\hat{\bx}_{0,i}^\ell+\frac{1-\alpha_{\tk{i}}}{1-\alpha_{\tk{i+1}}}U^\ell\right)$
\EndFor
\State \Return $X_{\tk{0}}$
\end{algorithmic}
\end{algorithm}
The following proposition gives the NELBO form for $-\log \pdata{0}{}{}[\param]$ defined in \Cref{subsec:audm}. See \Cref{sec:ctmc} for background on CTMCs.
We write $\jpdata{0}{\noise}{\cdot}[\param]$ for the time-zero marginal of the CTMC reverse model conditioned on $U=\noise$, so that $\pdata{0}{}{\bx_0}[\param]=\sum_{\noise}\pnoise{}{}{\noise}\jpdata{0}{\noise}{\bx_0}[\param]$.
\newcounter{savedpropositionaudm}
\setcounter{savedpropositionaudm}{\value{proposition}}
\begingroup
\setcounterref{proposition}{prop:absorbing-uniform-cont-loss}
\addtocounter{proposition}{-1}
\renewcommand{\theHproposition}{appendix.audm.\arabic{proposition}}
\begin{mdframed}[style=propFrame]
\begin{proposition}
    \label{prop:absorbing-uniform-cont-loss-app}
    A continuous-time NELBO for $-\log \pdata{0}{}{\bx_0}[\param]$ is
\begin{multline}
    \label{eq:absorbing-ctmc-loss-app}
        \txts \mathrm{L}^{\audm} _{\infty}(\bx_0; \param)
        \eqdef
        - \int_0^1
         \frac{\alpha^\prime _t}{1-\alpha_t}\,
        \pE\big[
            \sum_{\ell=1}^{\len}
            \indic_{X^\ell _t = {\color{crimson}U^\ell}}
            \big\{
                1-\dotp{{\color{crimson}U^\ell}}{\hat{\bx}^\param _0(X_t,t; {\color{crimson}U})^\ell}
                \\
                -
                \indic_{\bx^\ell _0 \neq {\color{crimson}U^\ell}} \cdot
                \big(
                    1+ \log \, \dotp{\bx^\ell _0}{ \hat{\bx}^\param _0(X_t,t; {\color{crimson}U})^\ell}
                \big)
            \big\}
        \big]\rmd t,
\end{multline}
where the expectation is \wrt\ $U \sim \text{Uniform}(\intset{1}{\vocab})^{\otimes \len}$, and $X_t \sim \fw{t\tbar 0}{\bx_0,U}{}$.
\end{proposition}
\end{mdframed}
\endgroup
\setcounter{proposition}{\value{savedpropositionaudm}}
\begin{proof}
For every $\bx_0$, Jensen's inequality gives
\[
    \txts
    -\log \pdata{0}{}{\bx_0}[\param]
    =
    -\log\sum_{\noise}\pnoise{}{}{\noise}\pdata{0}{\noise}{\bx_0}[\param]
    \leq
    \sum_{\noise}\pnoise{}{}{\noise}
    \bigl[-\log \pdata{0}{\noise}{\bx_0}[\param]\bigr].
\]
It is therefore enough to upper bound each conditional negative log-likelihood by the CTMC ELBO \eqref{eq:ctmc-specific-elbo}, with $\pi^\ell=U^\ell$, and then average over $U$.

Since the prefactor is $\dotp{X^\ell _}{U^\ell}$, only the position $\ell$ such that $X^\ell _t = U^\ell$ contribute to the loss. Let $\ell$ denote such a position and let $\tilde x \neq U^\ell$. We first compute the two arguments of $\Phi$ in \eqref{eq:ctmc-specific-elbo}. By definition of the conditional score \eqref{eq:conditional-score},
\[
    \dotp{\tilde x}{\score{t}{X_0^\ell}{U^\ell}}
    =
    \frac{
        \dotp{\tilde x}{\alpha_t X_0^\ell + (1-\alpha_t)U^\ell}
    }{
        \fw{t\tbar 0}{X_0^\ell,U^\ell}{U^\ell}
    } = \frac{\alpha_t}{1-\alpha_t}
    \indic_{\tilde x = X_0^\ell}
    \indic_{X_0^\ell \neq U^\ell} \eqsp.
\]
Next, by the definition \eqref{eq:score-definition} and since $X^\ell _t = U^\ell$, we have for every $\tilde x \neq U^\ell$,
\begin{equation*}
    \dotp{\tilde x}{\score{t}{}{X_t}[\param]^\ell}
    =
    \frac{\alpha_t}{1-\alpha_t}
    \dotp{\tilde x}{\hat{\bx}^\param _0(X_t,t; U)^\ell} \eqsp.
\end{equation*}
Substituting both expressions into \eqref{eq:ctmc-specific-elbo}, the contribution of position $\ell$ at time $t$ becomes
\[
    \beta_t \indic\{X_t^\ell = U^\ell\}
    \sum_{\tilde x \neq U^\ell}
    \Phi\!\left(
        \frac{\alpha_t}{1-\alpha_t}
       \indic_{\tilde x = X_0^\ell}
    \indic_{X_0^\ell \neq U^\ell},
        \frac{\alpha_t}{1-\alpha_t}
        \dotp{\tilde x}{\hat{\bx}^\param _0(X_t,t; U)^\ell}
    \right).
\]
If $X_0^\ell = U^\ell$, then the first argument of $\Phi$ is zero for every $\tilde x \neq U^\ell$, so the sum reduces to
\[
    \beta_t \frac{\alpha_t}{1-\alpha_t}
    \indic_{X_t^\ell = U^\ell}
    \sum_{\tilde x \neq U^\ell}
    \dotp{\tilde x}{\hat{\bx}^\param _0(X_t,t; U)^\ell}
    =
    \frac{\alpha^\prime_t}{1-\alpha_t}
    \indic_{X_t^\ell = U^\ell}
    \bigl(1-\dotp{U^\ell}{\hat{\bx}^\param _0(X_t,t; U)^\ell}\bigr).
\]
If $X_0^\ell \neq U^\ell$, only the term $\tilde x = X_0^\ell$ has a nonzero first argument, and therefore
\begin{align*}
    &\beta_t \indic\{X_t^\ell = U^\ell\}
    \sum_{\tilde x \neq U^\ell}
    \Phi\!\left(
        \frac{\alpha_t}{1-\alpha_t}\indic_{\tilde x = X_0^\ell},
        \frac{\alpha_t}{1-\alpha_t}
        \dotp{\tilde x}{\hat{\bx}^\param _0(X_t,t; U)^\ell}
    \right) \\
    &\qquad=
     \frac{- \alpha^\prime_t}{1-\alpha_t}
    \indic_{X_t^\ell = U^\ell}
    \left(
        \sum_{\tilde x \neq U^\ell}
        \dotp{\tilde x}{\hat{\bx}^\param _0(X_t,t; U)^\ell}
        -
        1
        -
        \log \dotp{X_0^\ell}{\hat{\bx}^\param _0(X_t,t; U)^\ell}
    \right) \\
    &\qquad=
    \frac{- \alpha^\prime _t}{1-\alpha_t}
    \indic_{X_t^\ell = U^\ell}
    \left(
        1-\dotp{U^\ell}{\hat{\bx}^\param _0(X_t,t; U)^\ell}
        -
        1
        -
        \log \dotp{X_0^\ell}{\hat{\bx}^\param _0(X_t,t; U)^\ell}
    \right).
\end{align*}
Combining the two cases gives exactly the integrand in \eqref{eq:absorbing-ctmc-loss-app}.
\end{proof} 
\subsection{Resampled AUDM}
\label{sec:proof-remasking}
In the rest of this section we write $\bz=(\bx,\noise)$. The UDM denoiser admits the mixture decomposition
\begin{equation}
    \label{eq:audm-denoiser-mixture-app}
    \txts
    \pdata{0\tbar t}{\bx_t}{\bx_0}
    =
    \sum_{\noise_t}
    \pdata{0\tbar t}{\bx_t,\noise_t}{\bx_0}\,
    \pnoise{t}{\bx_t}{\noise_t}
    \eqsp,
    \qquad
    \pnoise{t}{\bx_t}{\noise_t}
    \propto
    \pnoise{}{}{\noise_t}\pdata{t}{\noise_t}{\bx_t}
    \eqsp.
\end{equation}
A first exact lifted sampler is obtained by augmenting the state as $\tilde Z_t=(X_t,\tilde U_t)$ with conditional law $\tilde U_t \sim \pnoise{t}{X_t}{}$. The conditional law of $\tilde U_t$ belongs to the current lifted marginal, so the backward transition can be written as
\[
    \txts
    \tilde p_{s\tbar t}(\bz_s|\bz_t)
    \eqdef
    \pnoise{s}{\bx_s}{\noise_s}
    \sum_{\bx_0}
    \fw{s\tbar 0,t}{\bx_0,\bx_t}{\bx_s}
    \pdata{0\tbar t}{\bx_t,\noise_t}{\bx_0}
    \eqsp.
\]
Using \eqref{eq:audm-denoiser-mixture-app}, this transition satisfies the one-step marginalization identity
\[
    \txts
    \sum_{\noise_t}
    \tilde p_{s\tbar t}(\bz_s|\bz_t)
    \pnoise{t}{\bx_t}{\noise_t}
    =
    \pdata{s\tbar t}{\bx_t}{\bx_s}
    \pnoise{s}{\bx_s}{\noise_s}
    \eqsp.
\]
This is the same propagation identity as in \Cref{lem:lifted-process-marginalization} that we show next: after additionally summing over $\noise_s$, the remaining transition is the UDM reverse transition. Iterating this identity over the time grid shows that the ideal lifted chain preserves the UDM marginal path law. It is not directly useful, however, because drawing $\noise_s$ from $\pnoise{s}{\bx_s}{}$ is intractable. 

The resampled AUDM sampler in \Cref{alg:clean-conditioned-remasking} avoids this step by resampling the absorbing state after drawing $(X_0,X_s)$. It is the Markov chain on $Z_t=(X_t,U_t)$ whose one-step transition is
\[ 
\barpdata{s\tbar t}{\bz_t}{\bz_s} = \sum_{\bx_0} \pnoise{s\tbar 0}{\bx_0, \bx_s}{\noise_s} \fw{s\tbar 0, t}{\bx_0, \bx_t}{\bx_s}  \, \pdata{0\tbar t}{\bx_t,\noise_t}{\bx_0} \eqsp.
\]  
This equation is exactly the inline algorithm in the main text: given $(\bx_t,\noise_t)$, first draw $\bx_0$ from the absorbing-state denoiser, then draw $\bx_s$ from the UDM bridge, and finally resample the absorbing state conditionally on $(\bx_0,\bx_s)$. The last step is tractable because
\begin{equation}
    \label{eq:noise_cond_x0_xs}
\pnoise{s\tbar 0}{\bx_0, \bx_s}{\noise_s} = \frac{\pnoise{}{}{\noise_s} \fw{s\tbar 0}{\bx_0,\noise_s}{\bx_s}}{\fw{s\tbar 0}{\bx_0}{\bx_s}} \eqsp. 
\end{equation}
The exact lifted generative model is
\begin{equation}
    \label{eq:lifted-joint}
    \barpdata{0:n}{}{\bz_{\tk{0}:\tk{n}}} = \barpdata{\tk{n}}{}{\bz_\tk{n}} \prod_{i = 1}^\last \barpdata{\sk{i} \tbar \tk{i}}{\bz_{\tk{i}}}{\bz_\sk{i}}  \eqsp, \quad \barpdata{\tk{n}}{}{\bz_\tk{n}} \eqdef \pnoise{\tk{n}}{\bx_\tk{n}}{\noise_\tk{n}} \pdata{\tk{n}}{}{\bx_\tk{n}} 
\end{equation}
with $\barpdata{0\tbar \tk{1}}{\bz_\tk{1}}{} = \delta_{\noise_\tk{1}}\!\otimes \pdata{0\tbar \tk{1}}{\bx_\tk{1},\noise_\tk{1}}{}$. The next lemma states the one-step marginalization identity. Iterating it over the grid gives \Cref{prop:lifted-marginal}, hence the trajectory produced by the sampler has the same marginal law on $(X_{\tk{0}},\ldots,X_{\tk{n}})$ as UDM.
\begin{lemma} 
    \label{lem:lifted-process-marginalization}
    For any $0 < s < t \leq 1$, 
    \[ 
        \sum_{\noise_t} \barpdata{s\tbar t}{\bz_t}{\bz_s} \, \pnoise{t}{\bx_t}{\noise_t} = \pdata{s\tbar t}{\bx_t}{\bx_s} \pnoise{s}{\bx_s}{\noise_s} \eqsp.
    \]
\end{lemma} 
\begin{proof} 
We have that 
\begin{align*} 
    \sum_{\noise_t} \barpdata{s\tbar t}{\bz_t}{\bz_s} \, \pnoise{t}{\bx_t}{\noise_t} & = \sum_{\bx_0} \pnoise{s\tbar 0}{\bx_0, \bx_s}{\noise_s} \fw{s\tbar 0, t}{\bx_0, \bx_t}{\bx_s}  \,\sum_{\noise_t} \pdata{0\tbar t}{\bx_t,\noise_t}{\bx_0} \pnoise{t}{\bx_t}{\noise_t} \\ 
    & = \sum_{\bx_0} \pnoise{s\tbar 0}{\bx_0, \bx_s}{\noise_s} \fw{s\tbar 0, t}{\bx_0, \bx_t}{\bx_s} \, \pdata{0\tbar t}{\bx_t}{\bx_0} \eqsp. 
\end{align*}
where we have used the identity \eqref{eq:audm-denoiser-mixture-app} in the second line. Next, using the definition $\pdata{0\tbar t}{\bx_t}{\bx_0} = \pdata{0}{}{\bx_0} \fw{t\tbar 0}{\bx_0}{\bx_t} \big/ \pdata{t}{}{\bx_t}$ and \eqref{eq:noise_cond_x0_xs}, we get 
\begin{align*} 
    \sum_{\noise_t} \barpdata{s\tbar t}{\bz_t}{\bz_s} \, \pnoise{t}{\bx_t}{\noise_t} & \!=\! 
    \frac{\pnoise{}{}{\noise_s}}{\pdata{t}{}{\bx_t}} \sum_{\bx_0} \frac{\fw{s\tbar 0}{\bx_0,\noise_s}{\bx_s}}{\fw{s\tbar 0}{\bx_0}{\bx_s}}  \fw{s\tbar 0, t}{\bx_0, \bx_t}{\bx_s} \fw{t\tbar 0}{\bx_0}{\bx_t}\, \pdata{0}{}{\bx_0} \\
    & = \frac{\pnoise{}{}{\noise_s}}{\pdata{t}{}{\bx_t}}  \fw{t\tbar s}{\bx_s}{\bx_t} \sum_{\bx_0} \fw{s\tbar 0}{\bx_0,\noise_s}{\bx_s}\, \pdata{0}{}{\bx_0} \\
    & =  \frac{\pnoise{}{}{\noise_s}}{\pdata{t}{}{\bx_t}}  \fw{t\tbar s}{\bx_s}{\bx_t} \pdata{s}{\noise_s}{\bx_s} \\ 
    & =   \frac{\pnoise{}{}{\noise_s} \pdata{s}{\noise_s}{\bx_s}}{\pdata{s}{}{\bx_s}} \frac{\pdata{s}{}{\bx_s}  \fw{t\tbar s}{\bx_s}{\bx_t}}{\pdata{t}{}{\bx_t}} 
\end{align*}
where in the second line we have used the bridge identity $\fw{s\tbar 0, t}{\bx_0, \bx_t}{\bx_s} \fw{t\tbar 0}{\bx_0}{\bx_t} = \fw{s\tbar 0}{\bx_0}{\bx_s} \fw{t\tbar s}{\bx_s}{\bx_t}$. 
\end{proof}

\begin{proposition}
\label{prop:lifted-marginal}
It holds that 
\begin{equation}
\label{eq:backward-markovized-law}
\sum_{\noise_{\tk{0}:\tk{n}}} \barpdata{0:n}{}{\bz_{\tk{0}:\tk{n}}} \eqdef
\pdata{\tk{n}}{}{\bx_\tk{n}} \prod_{i = 1}^{\last} \pdata{\sk{i}\tbar \tk{i}}{\bx_\tk{i}}{\bx_\sk{i}} \eqsp. 
\end{equation}
\end{proposition}
\begin{proof}
    We define $\barpdata{0:k\tbar k+1}{\bz_\tk{k+1}}{\bz_{\tk{0}:\tk{k}}} \eqdef \prod_{i = 1} ^{k+1} \barpdata{\sk{i} \tbar \tk{i}}{\bz_\tk{i}}{\bz_\sk{i}}$ and proceed by induction. 

    First, by the definition \eqref{eq:lifted-joint} and \Cref{lem:lifted-process-marginalization} we have that  
    \begin{align*} 
        \sum_{\noise_{\tk{0}: \tk{n}}} \barpdata{0:n}{}{\bz_{\tk{0}:\tk{n}}}& = \sum_{\noise_{\tk{0}: \tk{n}}} \pdata{\tk{n}}{}{\bx_\tk{n}} \pnoise{\tk{n}}{\bx_\tk{n}}{\noise_\tk{n}} \barpdata{\sk{\last}\tbar \tk{\last}}{\bz_\tk{\last}}{\bz_\sk{\last}} \barpdata{0:\last-2\tbar \last-1}{\bz_\sk{\last}}{\bz_{\tk{0}:\tk{\last-1}}} \\
        & = \sum_{\noise_{\tk{0}:\tk{n-1}}} \pdata{\tk{n}}{}{\bx_\tk{n}} \pdata{\sk{\last}\tbar \tk{\last}}{\bx_\tk{\last}}{\bx_\sk{\last}} \\
        & \hspace{2cm} \times \pnoise{\sk{n}}{\bx_\sk{n}}{\noise_\sk{n}} \barpdata{0:\last-2\tbar \last-1}{\bz_\sk{\last}}{\bz_{\tk{0}:\tk{\last-1}}} \eqsp.
    \end{align*}
    Next, assuming that for $k \in \intset{2}{\last-1}$, 
    \begin{equation*} 
        \sum_{\noise_{\tk{0}: \tk{n}}} \barpdata{0:n}{}{\bz_{\tk{0}:\tk{n}}}  = \sum_{\noise_{\tk{0}: \tk{k}}} \pdata{\tk{\last}}{}{\bx_\tk{\last}} \prod_{i = k+1}^\last \pdata{\sk{i}\tbar \tk{i}}{\bx_\tk{i}}{\bx_\sk{i}} \pnoise{\tk{k}}{\bx_\tk{k}}{\noise_\tk{k}} \barpdata{0:k-1\tbar k}{\bz_\tk{k}}{\bz_{\tk{0}:\tk{k-1}}} 
    \end{equation*}
    we get by applying \Cref{lem:lifted-process-marginalization} again,  
        \begin{align*} 
        \sum_{\noise_{\tk{0}: \tk{n}}} \barpdata{0:n}{}{\bz_{\tk{0}:\tk{n}}} & = \sum_{\noise_{\tk{0}: \tk{k-1}}} \pdata{\tk{\last}}{}{\bx_\tk{\last}} \prod_{i = k+1}^\last \pdata{\sk{i}\tbar \tk{i}}{\bx_\tk{i}}{\bx_\sk{i}} \\
        & \hspace{1cm} \times \sum_{\noise_\tk{k}} \pnoise{\tk{k}}{\bx_\tk{k}}{\noise_\tk{k}} \barpdata{\sk{k}\tbar\tk{k}}{\bz_\tk{k}}{\bz_\sk{k}}  \barpdata{0:k-2\tbar k-1}{\bz_\sk{k}}{\bz_{\tk{0}:\tk{k-2}}}  \\
        & = \sum_{\noise_{\tk{0}: \tk{k-1}}} \pdata{\tk{\last}}{}{\bx_\tk{\last}} \prod_{i = k+1}^\last \pdata{\sk{i}\tbar \tk{i}}{\bx_\tk{i}}{\bx_\sk{i}} \\
        & \hspace{1cm} \times \pdata{\tk{k-1}\tbar \tk{k}}{\bx_\tk{k}}{\bx_\tk{k-1}} \pnoise{\tk{k-1}}{\bx_\tk{k-1}}{\noise_\tk{k-1}} \barpdata{0:k-2\tbar k-1}{\bz_\tk{k-1}}{\bz_{\tk{0}:\tk{k-2}}}  \\
        & = \sum_{\noise_{\tk{0}: \tk{k-1}}} \pdata{\tk{\last}}{}{\bx_\tk{\last}} \prod_{i = k}^\last \pdata{\sk{i}\tbar \tk{i}}{\bx_\tk{i}}{\bx_\sk{i}} \\
        & \hspace{1cm} \times \pnoise{\tk{k-1}}{\bx_\tk{k-1}}{\noise_\tk{k-1}} \barpdata{0:k-2\tbar k-1}{\bz_\tk{k-1}}{\bz_{\tk{0}:\tk{k-2}}} 
    \end{align*}
	    which finishes the induction. We hence get for $k = 2$, upon using the identity \eqref{eq:audm-denoiser-mixture-app}, 
    \begin{align*} 
        \sum_{\noise_{\tk{0}: \tk{n}}} \barpdata{0:n}{}{\bz_{\tk{0}:\tk{n}}} & = \sum_{\noise_\tk{0}, \noise_\tk{1}} \pdata{\tk{\last}}{}{\bx_\tk{\last}} \prod_{i = 2}^\last \pdata{\sk{i}\tbar \tk{i}}{\bx_\tk{i}}{\bx_\sk{i}} \pnoise{\tk{1}}{\bx_\tk{1}}{\noise_\tk{1}} \barpdata{0\tbar \tk{1}}{\bz_\tk{1}}{\bz_\tk{0}} \\
        & = \pdata{\tk{\last}}{}{\bx_\tk{\last}} \prod_{i = 2}^\last \pdata{\sk{i}\tbar \tk{i}}{\bx_\tk{i}}{\bx_\sk{i}} \sum_{\noise_\tk{1}} \pnoise{\tk{1}}{\bx_\tk{1}}{\noise_\tk{1}} \pdata{0\tbar \tk{1}}{\bx_\tk{1},\noise_\tk{1}}{\bx_0} \\
        & = \pdata{\tk{\last}}{}{\bx_\tk{\last}} \prod_{i = 1}^\last \pdata{\sk{i}\tbar \tk{i}}{\bx_\tk{i}}{\bx_\sk{i}} \eqsp,
	    \end{align*}
	    which gives the result. 
	\end{proof}

The exact lifted model above is useful conceptually, but it still contains the intractable denoiser $\pdata{0\tbar t}{\bx_t,\noise_t}{}$. Its true one-position reverse transition is obtained by marginalizing the exact lifted kernel over all other positions. Since the bridge and the resampling law are token-wise conditionally on $\bx_0$, this gives
\[
    \txts
    \barpdata{s\tbar t}{\bz_t}{\bx_s^\ell,\noise_s^\ell}[\ell]
    \eqdef
    \sum_{\bz_s^{-\ell}}
    \barpdata{s\tbar t}{\bz_t}{\bz_s}
    =
    \sum_{\bx^\ell_0}
    \pnoise{s\tbar 0}{\bx^\ell_0,\bx^\ell_s}{\noise^\ell_s}
    \fw{s\tbar 0,t}{\bx^\ell_0,\bx^\ell_t}{\bx^\ell_s}[\ell]\,
    \pdata{0\tbar t}{\bx_t,\noise_t}{\bx^\ell_0}[\ell],
\]
where $\pdata{0\tbar t}{\bx_t,\noise_t}{\bx^\ell_0}[\ell] \eqdef \sum_{\bx_0^{-\ell}}\pdata{0\tbar t}{\bx_t,\noise_t}{\bx_0}$. ReAUDM replaces this exact one-position posterior marginal with the factorized approximation induced by \eqref{eq:carry-over-uniform},
\[
    \txts
    \pdata{0\tbar t}{\bx_t,\noise_t}{\bx_0}[\param]
    \eqdef
    \prod_{\ell=1}^{\len}
    \categorical\!\left(\bx_0^\ell;\hat{\bx}^\param_0(\bx_t,t;\noise_t)^\ell\right),
\]
while keeping the UDM bridge and the resampling law \eqref{eq:noise_cond_x0_xs} exact. This gives the parameterized token-wise reverse transition
\[
    \txts
    \barpdata{s\tbar t}{\bz_t}{\bx_s^\ell,\noise_s^\ell}[\param,\ell]
    \eqdef
    \sum_{\bx^\ell_0}
    \pnoise{s\tbar 0}{\bx^\ell_0,\bx^\ell_s}{\noise^\ell_s}
    \fw{s\tbar 0,t}{\bx^\ell_0,\bx^\ell_t}{\bx^\ell_s}[\ell]\,
    \categorical\!\left(\bx^\ell_0;\hat{\bx}^\param_0(\bx_t,t;\noise_t)^\ell\right),
\]
and therefore the full transition factorizes as
\[
    \txts
    \barpdata{s\tbar t}{\bz_t}{\bz_s}[\param]
    \eqdef
    \prod_{\ell=1}^{\len}
    \barpdata{s\tbar t}{\bz_t}{\bx_s^\ell,\noise_s^\ell}[\param,\ell].
\]
Following the same argument as in the main paper, this factorized transition is optimal when the factorized denoiser matches the one-position marginals of the exact noise-conditioned posterior. Indeed, the associated joint model is
\[
    \txts
    \barpdata{0:n}{}{\bz_{\tk{0}:\tk{n}}}[\param]
    =
    \barpdata{\tk{n}}{}{\bz_{\tk{n}}}
    \prod_{i=1}^{\last}
    \barpdata{\sk{i}\tbar \tk{i}}{\bz_{\tk{i}}}{\bz_{\sk{i}}}[\param].
\]
As in the main paper, its discrete time NELBO decomposes into KL divergences between backward transitions:
\begin{multline}
    \txts
    \mathcal L^{\mathrm{ReAUDM}}_n(\param;\bx_0)
    \\=
    \pE\!\left[
        -\log \pdata{0\tbar \tk{1}}{X_{\tk{1}},U_{\tk{1}}}{\bx_0}[\param]
        +
        \sum_{i=1}^{\last}
        \kldivergence{\barpdata{\sk{i}\tbar \tk{i}}{Z_{\tk{i}}}{}}{\barpdata{\sk{i}\tbar \tk{i}}{Z_{\tk{i}}}{}[\param]}
        \,\middle|\, X_0=\bx_0
    \right]
    \eqsp.
\end{multline}
Since $\barpdata{s\tbar t}{}{}$ and $\barpdata{s\tbar t}{}{}[\param]$ have the same backward decomposition except for the absorbing-state denoiser, these KL terms are minimized by matching $\pdata{0\tbar t}{\bx_t,\noise_t}{}$ with its factorized approximation. Equivalently, the optimal factorized approximation is obtained by matching the noise-conditioned posterior marginals at every position, namely
$
\hat{\bx}^\param_0(\bx_t,t;\noise_t)^\ell = \pE[X_0^\ell \mid X_t=\bx_t,U_t=\noise_t]
$,
with the carry-over constraint in \eqref{eq:carry-over-uniform}. In particular, positions satisfying $\bx_t^\ell \neq \noise_t^\ell$ are already optimal by setting $\hat{\bx}^\param_0(\bx_t,t;\noise_t)^\ell=\bx_t^\ell$.
The corresponding sampler is:
\begin{algorithm}[H]
\caption{ReAUDM sampler}
\label{alg:reaudm-sampler}
\begin{algorithmic}
\Require grid $0=\tk{0}<\cdots<\tk{n}=1$
\State Draw $U_{\tk{n}} \sim \pnoise{\tk{n}}{}{}$ and $X_{\tk{n}} = U_\tk{n}$
\For{$i=n{-}1,\dots,0$}
\State Draw $X_0 \sim \pdata{0\tbar \tk{i+1}}{X_{\tk{i+1}},U_{\tk{i+1}}}{}[\param]$
\State Draw $X_{\tk{i}} \sim \fw{\tk{i}\tbar 0,\tk{i+1}}{X_0,X_{\tk{i+1}}}{}$
    \vspace{.1cm}
    \State \emph{In parallel over $\ell$:}
    \State \quad if $X^\ell_{\tk{i}} \neq X^\ell_0$:\quad $U^\ell_{\tk{i}} \gets X^\ell_{\tk{i}}$
    \State \quad else:\quad $U^\ell_{\tk{i}} \sim \categorical\!\left(\frac{X^\ell_0+\alpha_{\tk{i}}(\one-X^\ell_0)}{1+(\vocab-1)\alpha_{\tk{i}}}\right)$
\EndFor
\State \Return $X_{\tk{0}}$
\end{algorithmic}
\end{algorithm}

\subsection{Related works}
    \label{sec:audm-related-works}
\paragraph{Planned denoising \cite{liu2024think}.} 
We now relate AUDM to the planned-denoising parameterization of \cite{liu2024think}. The connection is that both constructions start from the same UDM denoiser, but they expose different latent variables behind it: AUDM keeps the full absorbing state $U$, while planned denoising keeps only the binary mask indicating whether a position has been corrupted. This mask decomposition leads naturally to a planner-denoiser factorization as we now see.  

We may also write an analogous identity to \eqref{eq:audm-denoiser-mixture-app} using a binary noise mask rather than with the absorbing state $U$. Let $N_t\in\{0,1\}^{\len}$ with p.m.f. 
\[
    \nu_t(\bn_t)
    \eqdef
    \prod_{\ell=1}^{\len}
    \mathrm{Bern}(\bn_t^\ell; 1-\alpha_t)
    \eqsp .
\]
Then the UDM forward transition admits the augmented representation
\begin{equation}
    \label{eq:planned-mask-forward}
    \fw{t\tbar 0}{\bx_0}{\bx_t}
    \eqdef
    \sum_{\bn_t}
    \underbrace{\prod_{\ell=1}^{\len}
    \categorical\!\left(
        \bx_t^\ell;
        (1-\bn_t^\ell)\bx_0^\ell+\bn_t^\ell \one / \vocab
    \right)\, \nu_t(\bn_t)}_{\eqdef \fw{t\tbar 0}{\bx_0}{\bx_t, \bn_t}}
    \eqsp .
\end{equation}
We can thus get an equivalent of the identity \eqref{eq:audm-denoiser-mixture-app}: 
\[ 
    \pdata{0\tbar t}{\bx_t}{\bx_0} \eqdef \sum_{\bn_t} \pdata{0\tbar t}{\bx_t, \bn_t}{\bx_0} \, \nu_t(\bn_t | \bx_t) \eqsp, 
\]
where  
\[
    \nu_t(\bn_t | \bx_t)
    \eqdef
    \frac{
        \sum_{\tilde\bx_0}
        \pdata{0}{}{\tilde\bx_0}
        \fw{t\tbar 0}{\tilde\bx_0}{\bx_t,\bn_t}
    }{
        \pdata{t}{}{\bx_t}
    }
    \eqsp
\]
and
\begin{equation}
    \label{eq:planned-mask-posterior}
    \pdata{0\tbar t}{\bx_t,\bn_t}{\bx_0}
    \eqdef
    \frac{\pdata{0}{}{\bx_0}\fw{t\tbar 0}{\bx_0}{\bx_t,\bn_t}}
    {\sum_{\tilde\bx_0}\pdata{0}{}{\tilde\bx_0}\fw{t\tbar 0}{\tilde\bx_0}{\bx_t,\bn_t}}
    =
    \frac{
        \pdata{0}{}{\bx_0}
        \prod_{\ell:\bn_t^\ell=0}\indic_{\bx_0^\ell=\bx_t^\ell}
    }{
        \sum_{\tilde\bx_0}
        \pdata{0}{}{\tilde\bx_0}
        \prod_{\ell:\bn_t^\ell=0}\indic_{\tilde\bx_0^\ell=\bx_t^\ell}
    }
    \eqsp .
\end{equation}
The mask-conditioned posterior has thus the carry-over unmasking property and, in contrast with \eqref{eq:audm-posterior}, it is time-independent. Conditioning on $\bn_t$ instead removes the ambiguity that remains when one only conditions on $U$ and observes $\bx_t^\ell=U^\ell$.

We now focus on the factorized denoiser $\prod_{\ell = 1}^\len \pdata{0\tbar t}{\bx_t}{}[\ell]$. The $\ell$-th posterior marginal $\pdata{0\tbar t}{\bx_t}{}[\ell]$ is  obtained by marginalizing only the $\ell$-th mask variable:
\begin{equation}
    \label{eq:planned-token-denoiser-decomp}
    \txts \pdata{0\tbar t}{\bx_t}{\bx^\ell _0}[\ell]
    =
    \sum_{\bn_t^\ell}
    \pdata{0\tbar t}{\bx_t,\bn_t^\ell}{\bx^\ell _0}[\ell]\,
    \nu_t^\ell(\bn_t^\ell|\bx_t)
\end{equation}
where 
\[
  \txts   \nu_t^\ell(\bn^\ell _t|\bx_t)
    \eqdef
    \sum_{\bn_t^{-\ell}}
    \nu_t(\bn _t|\bx_t)
    \eqsp,
    \quad \bn^\ell _t \in \{ 0, 1 \} 
\]
and, whenever $\nu_t^\ell(\bn^\ell_t|\bx_t)>0$,
\[
    \nu_t^{-\ell}(\bn_t^{-\ell}|\bx_t,\bn^\ell _t)
    \eqdef \nu_t(\bn_t | \bx_t)
     \big / \nu_t^\ell(\bn^\ell _t|\bx_t)
    \eqsp .
\]
and finally the token-wise mask conditional posterior in \eqref{eq:planned-token-denoiser-decomp} is
\[
    \txts
    \pdata{0\tbar t}{\bx_t,\bn^\ell _t}{\bx^\ell _0}[\ell]
    \eqdef
    \sum_{\bn_t^{-\ell}}
    \sum_{\bx^{-\ell} _0}
    \pdata{0\tbar t}{\bx_t, \bn_t}{\bx_0} \, \nu_t^{-\ell}(\bn_t^{-\ell}|\bx_t, \bn^\ell _t)
    \eqsp .
\]
It follows from the carry-over property above that
\[
    \txts
    \pdata{0\tbar t}{\bx_t,\bn^\ell _t = 0}{\bx_0}[\ell]
    =
    \categorical(\bx_0; \bx_t^\ell)
    \eqsp .
\]
Thus the factorized denoiser has the two-component form
\begin{equation}
    \label{eq:planned-factorized-denoiser}
    \txts
    \pdata{0\tbar t}{\bx_t}{\bx_0}[\ell]
    =
    \categorical(\bx^\ell _0; (1-\bm{a}(\bx_t,t)^\ell)\bx_t^\ell + 
    \bm{a}(\bx_t,t)^\ell \bm{d}(\bx_t,t)^\ell)
    \eqsp,
\end{equation}
where the exact \emph{planner} is
\[
    \txts
    \bm{a}(\bx_t,t)^\ell
    \eqdef
    \nu_t^\ell(\bn^\ell _t = 1|\bx_t)
    \eqsp,
\]
and the corresponding one-position denoiser is the probability vector such that 
\[
    \txts
 \categorical(\bx^\ell _0; \bm{d}(\bx_t,t)^\ell) = \pdata{0\tbar t}{\bx_t,\bn^\ell _t = 1}{\bx^\ell _0}[\ell] 
    \eqsp .
\]
The posterior conditioned on the full mask is time-independent, but the practical factorized denoiser is not: both the planner and the one-position denoiser are functions of $\bx_t$ and $t$, because they marginalize over the posterior law of the unobserved mask variables. Thus planned denoising parameterizes
\begin{equation}
    \label{eq:planner-denoiser}
    \denoiser{t}{}{\bx_t}[\param]^\ell
    \eqdef
    \big(1-\bm{a}_\param(\bx_t,t)^\ell\big)\bx_t^\ell
    +
    \bm{a}_\param(\bx_t,t)^\ell \bm{d}_\param(\bx_t,t)^\ell
    \eqsp,
\end{equation}
where $\bm{a}_\param(\bx_t,t)^\ell$ approximates $\bm{a}(\bx_t,t)^\ell$ and $\bm{d}_\param(\bx_t,t)^\ell$ approximates $\bm{d}(\bx_t,t)^\ell$. The continuous-time sampler associated with this parameterization is the maximal-coupling CTMC of \Cref{sec:max-coupling}, not the standard UDM CTMC of \Cref{sec:ctmc}. By \Cref{prop:uniform-max-coupling-rate}, the exact reverse generator is
\[
    \txts
    \rratemat{t}{\bx}{\tilde\bx}[\mathrm{mc}]
    =
    \begin{cases}
        \lambda_t\,
        \pdata{0\tbar t}{\bx}{\tilde\bx^\ell}[\ell]
        & \text{if $\bx^{-\ell}=\tilde\bx^{-\ell}$ and $\bx^\ell\neq \tilde\bx^\ell$,}\\
        0 & \text{otherwise,}
    \end{cases}
    \quad
    \lambda_t\eqdef-\alpha_t^\prime/(1-\alpha_t)
    \eqsp .
\]
Replacing the exact denoiser by $\denoiser{t}{}{\bx}[\param]$ in \eqref{eq:max-coupling-reverse-generator} gives the planned-denoising rate
\begin{equation}
    \label{eq:planned-denoising-rate}
    \txts
    \rratemat{t}{\bx}{\tilde\bx}[\param]
    =
    \begin{cases}
        \lambda_t\,
        \bm{a}_\param(\bx,t)^\ell
        \dotp{\tilde\bx^\ell}{\bm{d}_\param(\bx,t)^\ell}
        & \text{if $\bx^{-\ell}=\tilde\bx^{-\ell}$ and $\bx^\ell\neq \tilde\bx^\ell$,}\\
        0 & \text{otherwise.}
    \end{cases}
    \eqsp ,
\end{equation}
where the carry-over term vanishes because $\tilde\bx^\ell\neq\bx^\ell$. The planner and denoiser are then learned with the cross-entropy losses of Theorem~4.1 in \cite{liu2024think}, rather than with the continuous-time loss \eqref{eq:max-coupling-ctmc-loss}.

There are several differences between this construction and AUDM. First, AUDM conditions on the full noise vector $U$ and does not rely on the mask perspective. As a result, AUDM uses a single noise-conditioned denoiser network, whereas planned denoising splits the marginal denoiser into the planner $\bm{a}_\param$ and denoiser $\bm{d}_\param$. Third, our framework allows us to recover the UDM process using a simple resampling operation and does not require training a planner. 

\paragraph{Reparameterized discrete diffusion models \cite{zheng2024reparameterizeddiscretediffusionmodel}.}
We describe the construction in \cite{zheng2024reparameterizeddiscretediffusionmodel} using our own notation. Their starting point is the exact UDM bridge $\fw{s\tbar 0,t}{x_0,x_t}{}[\ell]$, split according to whether $x_t$ is equal to $x_0$. This is the same split used in \Cref{app:bridge-extension}; with the constants $\gamma_0,\gamma_2$ of \eqref{eq:linear-bridge-extension-udm}, we have 
\begin{equation}    
    \label{eq:decomposed-udm-bridge}
    \fw{s\tbar 0, t}{x_0, x_t}{x_s}[\ell]
    \eqdef \begin{cases}
        (1 - \gamma_0) x_t + \gamma_0 \one / \vocab  & \quad \text{if $x_t = x_0$}\\
        \gamma_2x_0+(1 - \gamma_2) \big(\alpha_{t\tbar s} x_t+ (1 - \alpha_{t\tbar s}) \one/\vocab \big)  & \quad \text{if $x_t \neq x_0$} \eqsp, 
    \end{cases}
\end{equation}
Their Eq. 4 augments this bridge with Bernoulli latent random variables. In the notation above, this can be written as a joint p.m.f. over $x_s$ and $\br_s \in \{0,1\}^{\len \times 2}$:
\[
    \tilde q^\ell_{s\tbar 0,t}(x_s,\br^\ell _s|x_0,x_t)
    \eqdef
     \tilde q^\ell_{s\tbar 0,t}(x_s|x_0,x_t, \br^\ell _s) \, \mathrm{Bern}(\br^{\ell1} _s; 1 - \gamma_0) \mathrm{Bern}(\br^{\ell 2} _s; \gamma_2)
    \]
where 
\[
\tilde{q}^\ell_{s\tbar 0,t}(x_s|x_0,x_t, \br^\ell _s) \eqdef 
    \begin{cases}
        \categorical(x_s;\br^{\ell 1} _s x_t+(1-\br^{\ell 1} _s)\one/\vocab)
        & \quad \text{if $x_t=x_0$}\\
        \categorical(x_s;\br^{\ell 2} _s x_0+(1-\br^{\ell 2} _s)(\alpha_{t\tbar s}x_t+(1-\alpha_{t\tbar s})\one/\vocab))
        & \quad \text{if $x_t\neq x_0$} \eqsp.
    \end{cases}
\]
We emphasize here that marginalizing $\br_s$ recovers the original bridge above and that using \eqref{eq:param_marginalization} with \eqref{eq:decomposed-udm-bridge} yields the same generative model as in \Cref{app:bridge-extension}. Indeed, define the transition: 
\[
    \jpdata{s\tbar t}{\bx_{t}}{\bx_s, \br_{s}}[\param]
    \eqdef
    \prod_{\ell=1}^{\len}
    \sum_{\bx_0^\ell}
    \tilde q^\ell_{s \tbar 0, t}(\bx^\ell _s,\br_s ^\ell|\bx^\ell _0,\bx_{t}^\ell)
    \, \categorical(\bx_0^\ell;\denoiser{t}{}{\bx_{t}}[\param]^\ell)
\]
and the corresponding joint law
\[
    \jpdata{0:n}{}{\bx_{\tk{0:n}},\br_{\tk{0}:\tk{n-1}}}[\param]
    \eqdef
    \pdata{\tk{n}}{}{\bx_{\tk{n}}}
    \prod_{i=1}^{n}
    \jpdata{\sk{i}\tbar \tk{i}}{\bx_{\tk{i}}}{\bx_\sk{i},\br_\sk{i}}[\param]
    \eqsp.
\]
Summing over $\br_{\tk{0}:\tk{n-1}}$ gives exactly the reverse-chain law obtained from \eqref{eq:param_marginalization} and \eqref{eq:decomposed-udm-bridge}, hence the same generative model as in \Cref{app:bridge-extension}. Moreover, under this joint law of $(X_\tk{0:n}, R_{\tk{0:n-1}})$, the random variables $R_\sk{i}$ and $X_{\tk{i}}$ are independent. 

The RDM model in \cite{zheng2024reparameterizeddiscretediffusionmodel} changes this lifted model by adding a dependence of $R_\sk{i}$ on $X_\tk{i}$ and a second latent variable $\mathbf{b}_t$ whose purpose is to substitute the indicators $\langle \bx^\ell _t, \bx^\ell _0 \rangle$ in \eqref{eq:decomposed-udm-bridge}. The considered transitions are instead 
\[
    \begin{aligned}
    & \jpdata{s\tbar t}{\bx_t,\mathbf{b}_t}{x_s^\ell,\br_s,\mathbf{b}_s^\ell}[\param,\ell] \\
    & \hspace{1cm} \eqdef
    \mathrm{Bern}(\br^{\ell1}_s;\rho_{s,t,\param}(\bx_t,\mathbf{b}_t)^{\ell 1})
    \mathrm{Bern}(\br^{\ell2} _s;\rho _{s,t,\param}(\bx_t,\mathbf{b}_t)^{\ell 2})
    \indic\{\mathbf{b}_s^\ell=(\mathbf{b}_t^\ell\wedge \br^{\ell 1} _s)\vee \br^{\ell 2}_s\}
    \\
    &\hspace{1cm} \quad \times
    \begin{cases}
        \categorical(x_s^\ell;\br^{\ell 1} _s \cdot\bx_t^\ell+(1-\br^{\ell 1}_s)\one/\vocab)
        & \quad \text{if $\mathbf{b}_t^\ell=1$}\\
        \categorical(x_s^\ell;\br^{\ell 2}_s \cdot \denoiser{t}{}{\bx_t}[\param]^\ell +(1-\br^{\ell 2}_s)(\alpha_{t\tbar s}\bx_t^\ell+(1-\alpha_{t\tbar s})\one/\vocab))
        & \quad \text{if $\mathbf{b}_t^\ell=0$} \eqsp,
    \end{cases}
    \end{aligned}
\]
and $\rho_{s,t,\param}(\bx_t,\mathbf{b}_t) \in [0, 1]^{\len \times 2}$ are probability vectors computed with a given rule provided in \cite[Section 4.3]{zheng2024reparameterizeddiscretediffusionmodel}. 

These transitions are best viewed as a type of confidence-sampling for UDMs. Despite the latent variables, AUDM and ReAUDM are substantially different, since our data augmentation is part of an exact lifted model: $U$ is sampled from the uniform corruption law, the reverse chain is defined conditionally on this variable, and marginalizing $U$ recovers the UDM reverse-chain law by \Cref{prop:remasked-audm-udm-joint}. By contrast, once the route probabilities are made confidence-dependent, the variables $(\br_s,\mathbf{b}_s)$ in RDM no longer describe a marginally exact augmentation of the UDM reverse process. They are auxiliary decoding variables used to decide which positions should be treated as already denoised, and thus serve a completely different purpose than our data augmentation. 

\paragraph{Discrete non-Markov diffusion models \cite{chen2024fast}. }
We now describe the process used in \cite{chen2024fast} through the bridge it induces. In the UDM case, their construction first samples transition times $\btau\in\rset^\len_{\geq 0}$ with independent coordinates satisfying $\pP(\btau^\ell<t)=1-\alpha_t$. Conditionally on $(X_0,\btau)$, the process evolves in reverse time with $X_1\sim\mathrm{Uniform}(\intset{1}{\vocab})^{\otimes\len}$, and its bridge for $0<s<t<1$ is deterministic at every position:
\[
    \txts
    \fw{s\tbar 0,t}{\bx_0,\bx_t,\btau}{\bx_s}
    =
    \prod_{\ell=1}^{\len}
    \begin{cases}
        \categorical(\bx_s^\ell;\bx_0^\ell)
        & \text{if $s<\btau^\ell\leq t$,}\\
        \categorical(\bx_s^\ell;\bx_t^\ell)
        & \text{otherwise.}
    \end{cases}
    \eqsp .
\]
The corresponding conditional law at time $t$ is
\[
    \txts
    \fw{t\tbar 0}{\bx_0,\btau}{\bx_t}
    =
    \prod_{\ell=1}^{\len}
    \categorical\!\left(
        \bx_t^\ell;
        \indic_{\btau^\ell>t}\bx_0^\ell
        +
        \indic_{\btau^\ell\le t}\one/\vocab
    \right)
    \eqsp,
\]
and marginalizing over $\btau$ recovers the UDM forward transition. Equivalently, with $j(\btau)\eqdef\prod_{\ell=1}^{\len}(-\alpha'_{\btau^\ell})$, these bridge transitions satisfy
\[ 
    \txts
    \int \sum_{\bx_t}
    \fw{s\tbar 0, t}{\bx_0, \bx_t,\btau}{\bx_s}
    \fw{t\tbar 0}{\bx_0,\btau}{\bx_t}
    j(\btau)\,\rmd\btau
    =
    \fw{s\tbar 0}{\bx_0}{\bx_s}
    \eqsp.
\]
Thus this process is marginal preserving.

The transition of the Markovian projection relevant for reverse sampling conditioned on $\btau$ is given by
\[
    \txts
    \pdata{s\tbar t}{\bx_t,\btau}{\bx_s}
    =
    \sum_{\bx_0}
    \fw{s\tbar 0,t}{\bx_0,\bx_t,\btau}{\bx_s}\,
    \pdata{0\tbar t}{\bx_t,\btau}{\bx_0}
    \eqsp .
\]
The posterior appearing in this reverse transition is
\[
    \pdata{0\tbar t}{\bx_t,\btau}{\bx_0}
    =
    \frac{
        \pdata{0}{}{\bx_0}
        \prod_{\ell:\btau^\ell>t}\indic_{\bx_0^\ell=\bx_t^\ell}
    }{
        \sum_{\tilde\bx_0}
        \pdata{0}{}{\tilde\bx_0}
        \prod_{\ell:\btau^\ell>t}\indic_{\tilde\bx_0^\ell=\bx_t^\ell}
    }
    \eqsp,
\]
Writing $\bn_t^\ell\eqdef\indic_{\btau^\ell\leq t}$, this posterior is exactly the full-mask posterior of \eqref{eq:planned-mask-posterior}:
\[
    \txts
    \pdata{0\tbar t}{\bx_t,\btau}{\bx_0}
    =
    \pdata{0\tbar t}{\bx_t,\bn_t}{\bx_0}
    \eqsp .
\]
This makes the link with planned denoising explicit. The transition times define a deterministic planner mask between $s$ and $t$: positions satisfying $s<\btau^\ell\leq t$ are resampled from the denoising posterior through the bridge, while all other positions are carried over from $\bx_t$. Thus the exact denoising object suggested by the Markovian projection is the $\btau$-conditioned posterior $\pdata{0\tbar t}{\bx_t,\btau}{}$. In \cite{chen2024fast}, however, the implemented transition plugs in a standard denoiser $\pdata{0\tbar t}{\bx_t}{}$ which is not conditioned on $\btau$.

%% file: appendix/mudm.tex
\section{Masked Uniform Diffusion}
\label{app:mudm}
This appendix complements \Cref{sec:mudm}. Unless explicitly stated otherwise, all forward and reverse quantities without an additional argument are the ones associated with UDM. We write $j(\btau) \eqdef \prod_{\ell=1}^{\len}(-\alpha'_{\btau^\ell})$ for the density of the transition times.

\begin{lemma}
    \label{lem:mudm-forward-tau}
    For every $t \in [0,1]$,
    \begin{equation}
        \label{eq:mudm-forward-tau}
        \txts
        \fw{t\tbar 0}{\bx_0,\btau}{\bx_t}
        =
        \prod_{\ell = 1}^\len
        \categorical\!\left(
            \bx_t^\ell;
            \indic_{\btau^\ell>t}\bx_0^\ell
            +
            \indic_{\btau^\ell\le t}\one/\vocab
        \right)
        \eqsp .
    \end{equation}
\end{lemma}
\begin{proof}
Fix a grid $0=\tk{0}<\cdots<\tk{n}=1$. Starting from $\fw{\tk{0}\tbar 0}{\bx_0,\btau}{\bx_{\tk{0}}} = \categorical(\bx_{\tk{0}};\bx_0)$, define recursively
\[
    \fw{\tk{i}\tbar 0}{\bx_0,\btau}{\bx_{\tk{i}}}
    \eqdef
    \sum_{\bx_{\tk{i+1}}}
    \fw{\tk{i}\tbar 0,\tk{i+1}}{\bx_0,\bx_{\tk{i+1}},\btau}{\bx_{\tk{i}}}\,
    \fw{\tk{i+1}\tbar 0}{\bx_0,\btau}{\bx_{\tk{i+1}}}
    \eqsp .
\]
We prove by backward induction on the grid that \eqref{eq:mudm-forward-tau} holds at every time $\tk{i}$. Assume it holds at time $\tk{i+1}$. Since both the bridge and the induction hypothesis factorize over positions, it is enough to compute one position $\ell$:
\[
    \begin{aligned}
        &\sum_{\bx_{\tk{i+1}}^\ell}
        \fw{\tk{i}\tbar 0,\tk{i+1}}{\bx_0^\ell,\bx_{\tk{i+1}}^\ell,\btau^\ell}{\bx_{\tk{i}}^\ell}[\ell]\,
        \categorical\!\left(
            \bx_{\tk{i+1}}^\ell;
            \indic_{\btau^\ell>\tk{i+1}}\bx_0^\ell
            +
            \indic_{\btau^\ell\le \tk{i+1}}\one/\vocab
        \right) \\
        &\qquad =
        \indic_{\btau^\ell\le \tk{i}}
        \sum_{\bx_{\tk{i+1}}^\ell}
        \categorical(\bx_{\tk{i}}^\ell;\bx_{\tk{i+1}}^\ell)\,
        \categorical(\bx_{\tk{i+1}}^\ell;\one/\vocab) \\
        &\qquad\quad +
        \indic_{\tk{i}<\btau^\ell\le \tk{i+1}}
        \sum_{\bx_{\tk{i+1}}^\ell}
        \categorical(\bx_{\tk{i}}^\ell;\bx_0^\ell)\,
        \categorical(\bx_{\tk{i+1}}^\ell;\one/\vocab) \\
        &\qquad\quad +
        \indic_{\btau^\ell>\tk{i+1}}
        \sum_{\bx_{\tk{i+1}}^\ell}
        \categorical(\bx_{\tk{i}}^\ell;\bx_{\tk{i+1}}^\ell)\,
        \categorical(\bx_{\tk{i+1}}^\ell;\bx_0^\ell) \\
        &\qquad =
        \indic_{\btau^\ell\le \tk{i}}\categorical(\bx_{\tk{i}}^\ell;\one/\vocab)
        +
        \indic_{\tk{i}<\btau^\ell\le \tk{i+1}}\categorical(\bx_{\tk{i}}^\ell;\bx_0^\ell)
        +
        \indic_{\btau^\ell>\tk{i+1}}\categorical(\bx_{\tk{i}}^\ell;\bx_0^\ell) \\
        &\qquad =
        \categorical\!\left(
            \bx_{\tk{i}}^\ell;
            \indic_{\btau^\ell>\tk{i}}\bx_0^\ell
            +
            \indic_{\btau^\ell\le \tk{i}}\one/\vocab
        \right)
        \eqsp .
    \end{aligned}
\]
This proves the induction step. Hence \eqref{eq:mudm-forward-tau} holds for every time of the grid, and therefore for every $t$.
\end{proof}

\begin{lemma}
    \label{lem:mudm-denoiser-equivalence}
    If
    \[
        \tilde{\bx}_t(\btau)^\ell
        \eqdef
        \begin{cases}
            \bx_t^\ell & \text{if $\btau^\ell>t$,}\\
            \mask & \text{if $\btau^\ell\le t$,}
        \end{cases}
    \]
    then
    \[
        \pdata{0\tbar t}{\bx_t,\btau}{\bx_0}
        =
        \pdata{0\tbar t}{\tilde{\bx}_t(\btau)}{\bx_0}[\mathrm{mask}]
        \eqsp .
    \]
\end{lemma}
\begin{proof}
By Bayes' rule and \eqref{eq:mudm-forward-tau},
\[
    \pdata{0\tbar t}{\bx_t,\btau}{\bx_0}
    =
    \frac{
        \pdata{0}{}{\bx_0}\fw{t\tbar 0}{\bx_0,\btau}{\bx_t}
    }{
        \sum_{\tilde{\bx}_0}
        \pdata{0}{}{\tilde{\bx}_0}\fw{t\tbar 0}{\tilde{\bx}_0,\btau}{\bx_t}
    }
    =
    \frac{
        \pdata{0}{}{\bx_0}
        \prod_{\ell:\btau^\ell>t}\indic_{\bx_0^\ell=\bx_t^\ell}
    }{
        \sum_{\tilde{\bx}_0}
        \pdata{0}{}{\tilde{\bx}_0}
        \prod_{\ell:\btau^\ell>t}\indic_{\tilde{\bx}_0^\ell=\bx_t^\ell}
    }
    \eqsp .
\]
Since $\tilde{\bx}_t(\btau)^\ell \neq \mask$ if and only if $\btau^\ell>t$, and then necessarily $\tilde{\bx}_t(\btau)^\ell=\bx_t^\ell$, we have
\[
    \prod_{\ell:\btau^\ell>t}\indic_{\bx_0^\ell=\bx_t^\ell}
    =
    \prod_{\ell:\tilde{\bx}_t(\btau)^\ell\neq\mask}
    \indic_{\bx_0^\ell=\tilde{\bx}_t(\btau)^\ell}
    \eqsp,
\]
and similarly in the denominator. Therefore
\[
    \pdata{0\tbar t}{\bx_t,\btau}{\bx_0}
    =
    \frac{
        \pdata{0}{}{\bx_0}
        \prod_{\ell:\tilde{\bx}_t(\btau)^\ell\neq\mask}
        \indic_{\bx_0^\ell=\tilde{\bx}_t(\btau)^\ell}
    }{
        \sum_{\tilde{\bx}_0}
        \pdata{0}{}{\tilde{\bx}_0}
        \prod_{\ell:\tilde{\bx}_t(\btau)^\ell\neq\mask}
        \indic_{\tilde{\bx}_0^\ell=\tilde{\bx}_t(\btau)^\ell}
    }
    =
    \pdata{0\tbar t}{\tilde{\bx}_t(\btau)}{\bx_0}[\mathrm{mask}]
    \eqsp .
\]
\end{proof}

Using the $\btau$-conditioned bridge and reverse transition introduced in \Cref{sec:mudm}, we denote by $j_t(\btau|\bx_t)$ the posterior density of $\btau$ given $X_t=\bx_t$, namely
\[
    j_t(\btau|\bx_t)
    \eqdef
    \sum_{\bx_0}
    \frac{
        \pdata{0}{}{\bx_0}\fw{t\tbar 0}{\bx_0,\btau}{\bx_t}
    }{
        \pdata{t}{}{\bx_t}
    }
    j(\btau)
    \eqsp ,
\]
then the exact UDM denoising posterior decomposes as
\begin{equation}
    \label{eq:mudm-udm-denoising-posterior-mixture}
    \txts
    \pdata{0\tbar t}{\bx_t}{\bx_0}
    =
    \int
    \pdata{0\tbar t}{\bx_t,\btau}{\bx_0}\,
    j_t(\btau|\bx_t)
    \eqsp .
\end{equation}
This is the exact $\btau$-mixture identity behind the masked-diffusion view of UDM. The difficulty is that $j_t(\cdot|\bx_t)$ is not tractable, so it cannot be used directly in a sampler.

To obtain an exact lifted chain with a tractable resampling step, we instead condition on $(X_0,X_s)$ and define
\begin{equation}
    \label{eq:mudm-tau-cond-x0-xs}
    j_{s\tbar 0}(\btau_s|\bx_0,\bx_s)
    \eqdef
    \frac{
        j(\btau_s)\fw{s\tbar 0}{\bx_0,\btau_s}{\bx_s}
    }{
        \fw{s\tbar 0}{\bx_0}{\bx_s}
    }
    =
    \prod_{\ell=1}^{\len}
    \begin{cases}
        \frac{-\alpha'_{\btau^\ell_s}}{1-\alpha_s}\indic_{\btau^\ell_s\le s}
        & \bx^\ell_s \neq \bx^\ell_0, \\
        \frac{-\alpha'_{\btau^\ell_s}}{1+(\vocab-1)\alpha_s}
        \big(
            \indic_{\btau^\ell_s\le s}
            +
            \vocab\indic_{\btau^\ell_s>s}
        \big)
        & \bx^\ell_s = \bx^\ell_0.
    \end{cases}
\end{equation}
If $\bx^\ell_s \neq \bx^\ell_0$, then the transition must already have occurred before time $s$, so $\btau^\ell_s$ is sampled from the prior restricted to $[0,s]$. If $\bx^\ell_s = \bx^\ell_0$, then either the position has not transitioned yet, or it has already transitioned and the uniform resampling happened to return $\bx^\ell_0$. The law $j_{s\tbar 0}$ exactly reweights these two possibilities.

The lifted reverse transition is then
\begin{equation}
    \label{eq:mudm-lifted-reverse-tau}
    \barpdata{s\tbar t}{\bx_t,\btau_t}{\bx_s,\btau_s}
    \eqdef
    \sum_{\bx_0}
    j_{s\tbar 0}(\btau_s|\bx_0,\bx_s)\,
    \fw{s\tbar 0,t}{\bx_0,\bx_t}{\bx_s}\,
    \pdata{0\tbar t}{\bx_t,\btau_t}{\bx_0}\, 
    \eqsp .
\end{equation}
For a grid $0=\tk{0}<\cdots<\tk{n}=1$, let $\barpdata{0:n}{}{}$ denote the corresponding path law, with initialization $\barpdata{\tk{n}}{}{\bx_{\tk{n}},\btau_{\tk{n}}} \eqdef \pdata{\tk{n}}{}{\bx_{\tk{n}}} j_{\tk{n}}(\btau_{\tk{n}}|\bx_{\tk{n}})$. If $\alpha_{\tk{n}}=0$, this reduces to $X_{\tk{n}} \sim \bpiu^{\otimes \len}$ and $\btau_{\tk{n}} \sim j$.

\begin{algorithm}[t]
\caption{Resampled MUDM}
\label{alg:mudm-flag-resampling}
\begin{algorithmic}
\Require grid $0=\tk{0}<\cdots<\tk{n}=1$
\State Draw $X_{\tk{n}} \sim \pdata{\tk{n}}{}{\cdot}$ and $\btau_{\tk{n}} \sim j_{\tk{n}}(\cdot|X_{\tk{n}})$
\For{$i=n{-}1,\dots,0$}
    \State Draw $X_0 \sim \pdata{0\tbar \tk{i+1}}{\tilde X_{\tk{i+1}}(\btau_{\tk{i+1}})}{}[\param, \mathrm{mask}]$
    \State Draw $X_{\tk{i}} \sim \fw{\tk{i}\tbar 0,\tk{i+1}}{X_0,X_{\tk{i+1}}}{}$
    \For{$\ell=1,\dots,\len$}
        \If{$X^\ell_{\tk{i}} \neq X^\ell_0$}
            \State Draw $\tau^\ell_{\tk{i}}$ with density $\frac{-\alpha'_{\tau}}{1-\alpha_{\tk{i}}}\indic_{\tau\le \tk{i}}$
        \Else
            \State Draw $\tau^\ell_{\tk{i}}$ with density $\frac{-\alpha'_{\tau}}{1+(\vocab-1)\alpha_{\tk{i}}}\big(\indic_{\tau\le \tk{i}}+\vocab\indic_{\tau>\tk{i}}\big)$
        \EndIf
    \EndFor
\EndFor
\State \Return $X_{\tk{0}}$
\end{algorithmic}
\end{algorithm}

\begin{mdframed}[style=propFrame]
    \begin{proposition}
        \label{prop:mudm-udm-joint}
        The trajectory generated by \Cref{alg:mudm-flag-resampling} satisfies
        \[
            \law(X_{\tk{0}},\ldots,X_{\tk{n}})
            =
            \pdata{0:n}{}{}
            \eqsp ,
        \]
        where $\pdata{0:n}{}{}$ is the UDM reverse-chain law.
    \end{proposition}
\end{mdframed}
\begin{proof}
Similar to the proof of Proposition \ref{prop:remasked-audm-udm-joint}, it is enough to prove the one-step identity:
\[
    \int
    \barpdata{s\tbar t}{\bx_t,\btau_t}{\bx_s,\btau_s}\,
    j_t(\btau_t|\bx_t)
    =
    \pdata{s\tbar t}{\bx_t}{\bx_s}\,
    j_s(\btau_s|\bx_s)
    \eqsp .
\]
Using \eqref{eq:mudm-lifted-reverse-tau}, we obtain
\[
    \begin{aligned}
        &\int
        \barpdata{s\tbar t}{\bx_t,\btau_t}{\bx_s,\btau_s}\,
        j_t(\btau_t|\bx_t) \\
        &\qquad =
        \sum_{\bx_0}
        j_{s\tbar 0}(\btau_s|\bx_0,\bx_s)\,
        \fw{s\tbar 0,t}{\bx_0,\bx_t}{\bx_s}
        \int
        \pdata{0\tbar t}{\bx_t,\btau_t}{\bx_0}\,
        j_t(\btau_t|\bx_t)
        \eqsp .
    \end{aligned}
\]
By \eqref{eq:mudm-udm-denoising-posterior-mixture}, the last integral is $\pdata{0\tbar t}{\bx_t}{\bx_0}$. Therefore
\[
    \int
    \barpdata{s\tbar t}{\bx_t,\btau_t}{\bx_s,\btau_s}\,
    j_t(\btau_t|\bx_t)
    =
    \sum_{\bx_0}
    j_{s\tbar 0}(\btau_s|\bx_0,\bx_s)\,
    \fw{s\tbar 0,t}{\bx_0,\bx_t}{\bx_s}
    \pdata{0\tbar t}{\bx_t}{\bx_0}
    \eqsp .
\]
Now use $\pdata{0\tbar t}{\bx_t}{\bx_0} = \pdata{0}{}{\bx_0}\fw{t\tbar 0}{\bx_0}{\bx_t}\big/\pdata{t}{}{\bx_t}$ and the UDM bridge identity
\[
    \fw{s\tbar 0,t}{\bx_0,\bx_t}{\bx_s}\,
    \fw{t\tbar 0}{\bx_0}{\bx_t}
    =
    \fw{s\tbar 0}{\bx_0}{\bx_s}\,
    \fw{t\tbar s}{\bx_s}{\bx_t}
    \eqsp .
\]
Substituting \eqref{eq:mudm-tau-cond-x0-xs} then gives
\[
    \int
    \barpdata{s\tbar t}{\bx_t,\btau_t}{\bx_s,\btau_s}\,
    j_t(\btau_t|\bx_t)
    =
    \frac{\fw{t\tbar s}{\bx_s}{\bx_t}}{\pdata{t}{}{\bx_t}}
    \sum_{\bx_0}
    \pdata{0}{}{\bx_0}\,
    j(\btau_s)\,
    \fw{s\tbar 0}{\bx_0,\btau_s}{\bx_s}
    \eqsp .
\]
The sum on the right-hand side is exactly the joint forward law of $(X_s,\btau_s)$, hence it is equal to $\pdata{s}{}{\bx_s} j_s(\btau_s|\bx_s)$. We conclude that
\[
    \begin{aligned}
        &\int
        \barpdata{s\tbar t}{\bx_t,\btau_t}{\bx_s,\btau_s}\,
        j_t(\btau_t|\bx_t) \\
        &\qquad =
        \frac{\pdata{s}{}{\bx_s}\fw{t\tbar s}{\bx_s}{\bx_t}}{\pdata{t}{}{\bx_t}}
        j_s(\btau_s|\bx_s) \\
        &\qquad =
        \pdata{s\tbar t}{\bx_t}{\bx_s}\,
        j_s(\btau_s|\bx_s)
        \eqsp .
    \end{aligned}
\]
This proves the one-step identity. Iterating along the grid yields the claimed equality of the $X$-trajectory law with the UDM reverse-chain law.
\end{proof}

%% file: appendix/bridge_linearization.tex
\section{On the role of the bridge extension in the plug-in parameterization}
\label{app:bridge-extension}

The plug-in parameterization \ref{eq:param_plug_in} is only well defined once one specifies how the token-wise bridge $\fw{s\tbar 0,t}{\cdot,\cdot}{x_s}[\ell]$, originally defined for pairs of one-hot vectors, is extended to a first argument in the simplex $\Simplex{\vocab}$. As a consequence, different extensions coinciding on one-hot inputs define the same exact bridge, while leading to different plug-in parameterizations once the network output is interpreted as a probability vector.

The extension used in the main text is the ``canonical'' one obtained by keeping the Bayes formula \eqref{eq:bridge-prob} for $\mu \in \Simplex{\vocab}$, when it is possible. In the UDM case, this yields
\[
\fw{s\tbar 0,t}{\mu,x_t}{}[\ell]
=
\categorical\!\left(
    \cdot;
    \frac{
        \vocab \alpha_t \dotp{x_t}{\mu}x_t
        +
        (\alpha_{t\tbar s}-\alpha_t)x_t
        +
        (\alpha_s-\alpha_t)\mu
        +
        D_{s,t}\one/\vocab
    }{
        \vocab \alpha_t \dotp{x_t}{\mu} + 1-\alpha_t
    }
\right),
\]
which is nonlinear in $\mu$ because of the denominator. This is precisely the extension for which \Cref{prop:loo} shows that the minimizer of the expected NELBO is the leave-one-out posterior.

Another admissible choice is the barycentric extension obtained by averaging the exact bridge over a clean token drawn from $\mu$. Since every $\mu \in \Simplex{\vocab}$ admits the decomposition $\mu = \sum_{x_0 \in \msv} \dotp{x_0}{\mu} x_0$, this extension is uniquely characterized by
\begin{equation}
\label{eq:linear-bridge-extension}
\barfw{s\tbar 0,t}{\mu,x_t}{}[\ell]
\eqdef
\sum_{x_0 \in \msv}
\dotp{x_0}{\mu}
\fw{s\tbar 0,t}{x_0,x_t}{}[\ell].
\end{equation}
By construction, \eqref{eq:linear-bridge-extension} coincides with the exact bridge on the vertices of the simplex and is affine in $\mu$. In MDM, because the bridge is already affine, this extension coincides with the canonical one. In UDM, the two differ.

For UDM, \eqref{eq:linear-bridge-extension} admits a simple closed form. Define
\[
\gamma_0 \eqdef \frac{D_{s,t}}{1+(\vocab-1)\alpha_t},
\quad
\gamma_1 \eqdef \frac{\alpha_{t\tbar s}-\alpha_t}{1-\alpha_t},
\quad
\gamma_2 \eqdef \frac{\alpha_s-\alpha_t}{1-\alpha_t},
\quad
\gamma_3 \eqdef \frac{D_{s,t}}{1-\alpha_t},
\]
and write $x_t = \onehot{k}$ and $\mu_k \eqdef \dotp{\onehot{k}}{\mu}$. Then
\begin{equation}
\label{eq:linear-bridge-extension-udm}
\begin{aligned}
\barfw{s\tbar 0,t}{\mu,\onehot{k}}{}[\ell]
&=
\categorical\!\left(
    \cdot;
    \gamma_2\mu
    +
    \left(\gamma_1 + (\gamma_3-\gamma_0)\mu_k\right)\onehot{k}
    +
    \left(\gamma_3 + (\gamma_0-\gamma_3)\mu_k\right)\frac{\one}{\vocab}
\right).
\end{aligned}
\end{equation}
Indeed, when $x_t = x_0$, the exact UDM bridge is
\[
\fw{s\tbar 0,t}{x_t,x_t}{}[\ell]
=
\categorical\!\left(
    \cdot;
    \left(
        1-\frac{D_{s,t}}{1+(\vocab-1)\alpha_t}
    \right)x_t
    +
    \frac{D_{s,t}}{1+(\vocab-1)\alpha_t}\frac{\one}{\vocab}
\right),
\]
whereas for $x_0 \neq x_t$ it is
\[
\fw{s\tbar 0,t}{x_0,x_t}{}[\ell]
=
\categorical\!\left(
    \cdot;
    \frac{\alpha_s-\alpha_t}{1-\alpha_t}x_0
    +
    \frac{\alpha_{t\tbar s}-\alpha_t}{1-\alpha_t}x_t
    +
    \frac{D_{s,t}}{1-\alpha_t}\frac{\one}{\vocab}
\right),
\]
and averaging these two cases against $\mu$ gives \eqref{eq:linear-bridge-extension-udm}.

This extension makes the denoiser reappear as the plug-in optimum. Indeed, for every $\ell \in \{1,\dots,\len\}$ and every $\bx_t \in \msx$,
\begin{align*}
\pdata{s\tbar t}{\bx_t}{}[\ell]
&=
\sum_{x_0 \in \msv}
\barfw{s\tbar 0,t}{x_0,\bx_t^\ell}{}[\ell]
\pdata{0\tbar t}{\bx_t}{x_0}[\ell] \\
&=
\barfw{s\tbar 0,t}{\matrixdenoiser{t}{}{\bx_t}^\ell,\bx_t^\ell}{}[\ell],
\end{align*}
where the second line is exactly the definition of the barycentric extension with $\mu = \matrixdenoiser{t}{}{\bx_t}^\ell$. Therefore, if one defines the plug-in reverse kernel by
\[
\hpdata{s\tbar t}{\bx_t}{\bx_s}[\param]
\eqdef
\prod_{\ell=1}^{\len}
\barfw{s\tbar 0,t}{\denoiser{t}{}{\bx_t}[\param]^\ell,\bx_t^\ell}{\bx_s^\ell}[\ell],
\]
then the argument used for the marginalization parameterization applies, and shows that the denoising posterior $\pdata{0\tbar t}{\bx_t}{}[\ell]$ is a minimizer of the expected NELBO, i.e., $\denoiser{t}{}{\bx_t}[\param_\star] = \matrixdenoiser{t}{}{\bx_t}$. 

This observation clarifies the role of \Cref{prop:loo}. The difference between the denoiser and the leave-one-out posterior is not caused by the plug-in idea itself, but by the specific extension of the bridge used to interpret a probability vector as a clean-token argument. The canonical nonlinear extension leads to the leave-one-out target, whereas the barycentric affine extension leads back to the denoiser.

A useful consequence of this observation is that it makes it simple to express the continuous-time ELBO of the marginalization parameterization directly as a plug-in of the denoiser.

We could derive this ELBO by taking $T \to \infty$ in the discrete-time ELBO, but in order to emphasize the role of the extension, we will derive it here using the CTMC point of view (see \Cref{sec:ctmc}). Again, any extension of the conditional rate matrix that coincides on one-hots gives a similar process. Here the conditional rate of the CTMC, corresponding to the linear-bridge extension, is itself linear and is given by
\[
\ratemat{t}{x, \mu}{\tilde x}[\mathrm{lin}, \ell]
\eqdef
\ratemat{t}{x}{\tilde x}[\ell]
\left(
1
-
\frac{\vocab\alpha_t}{1+(\vocab-1)\alpha_t}\dotp{x}{\mu}
+
\frac{\vocab\alpha_t}{1-\alpha_t}\dotp{\tilde x}{\mu}
\right).
\]
where the UDM forward rate is $\ratemat{t}{x}{\tilde x}[\ell]=\beta_t/\vocab$ with $\beta_t = \frac{-\alpha_t'}{\alpha_t}$. This formula can be derived by linearizing $x_0 \mapsto \frac{\fw{t \tbar 0}{x_0}{\tilde x}[\ell]}{\fw{t \tbar 0}{x_0}{x}[\ell]}$, or by taking the limit when $s \to t$ in the linear-bridge extension \eqref{eq:linear-bridge-extension-udm}.

The corresponding continuous-time marginalization NELBO is (see \Cref{sec:ctmc} for the general case)
\begin{equation}
\label{eq:marginalization-ct-elbo-linear-bridge}
\begin{aligned}
\mathcal L^{\mathrm{marg}}_\infty(\param)
&\eqdef
\pE\!\left[
    \kldivergence{\fw{1\tbar 0}{X_0}{}}{\pdata{1}{}{}}
\right]
\\
&\quad
+
\int_0^1
\frac{\beta_t}{\vocab}\,
\pE\!\bigg[
    \sum_{\ell=1}^{\len}
    \sum_{\tilde x \neq X_t^\ell}
    \Phi\!\bigg(
        1
        -
        \frac{\vocab\alpha_t}{1+(\vocab-1)\alpha_t}\dotp{X_t^\ell}{X_0^\ell}
        +
        \frac{\vocab\alpha_t}{1-\alpha_t}\dotp{\tilde x}{X_0^\ell},
        \\
        &\hspace{0.4cm}
        1
        -
        \frac{\vocab\alpha_t}{1+(\vocab-1)\alpha_t}\dotp{X_t^\ell}{\denoiser{t}{}{X_t}[\param]^\ell}
        +
        \frac{\vocab\alpha_t}{1-\alpha_t}\dotp{\tilde x}{\denoiser{t}{}{X_t}[\param]^\ell}
    \bigg)
\bigg]\rmd t .
\end{aligned}
\end{equation}
Here $\Phi(a,b)\eqdef b-a+a\log(a/b)$, and the expectation is with respect to $X_0 \sim \pdata{0}{}{}$ and $X_t \sim \fw{t\tbar 0}{X_0}{}$. Since $\ratemat{t}{x, \mu}{\tilde x}[\mathrm{lin}, \ell]$ is affine in $\mu$, the exact rate averaged over the conditional law of $X_0^\ell$ given $X_t$ is obtained by replacing $X_0^\ell$ with the denoising posterior $\pdata{0\tbar t}{X_t}{}[\ell]$, which is why \eqref{eq:marginalization-ct-elbo-linear-bridge} targets the denoiser.

In order to understand the duality between linearization of the considered bridge extension and parameterization (marginalization or plug-in), it is worth noting that the same expression can be obtained from the canonical plug-in parameterization if the network is understood to parameterize a denoiser and converted to the leave-one-out representation before it is plugged into the bridge. More precisely, let $\loodenoiser{t}{}{\bx_t}[\param]$ be obtained from $\matrixdenoiser{t}{}{\bx_t}$ by the inverse conversion \eqref{eq:loo-from-denoiser-general}. In UDM, substituting this $\loodenoiser{t}{}{\bx_t}[\param]$ into the canonical plug-in CTMC rate $\ratemat{t}{\bx_t^\ell}{\tilde x}[\ell]\fw{t\tbar 0}{\loodenoiser{t}{}{\bx_t}^\ell}{\tilde x}[\ell]/\fw{t\tbar 0}{\loodenoiser{t}{}{\bx_t}^\ell}{\bx_t^\ell}[\ell]$ gives exactly $\ratemat{t}{x, \matrixdenoiser{t}{}{\bx_t}^\ell}{\tilde x}[\mathrm{lin}, \ell]$. Thus the linear-bridge marginalization ELBO and the $T\to\infty$ limit of the converted plug-in ELBO are the same generalized-KL objective, written in denoiser coordinates.


%% file: appendix/predictor_corrector.tex
\section{Leave-one-out and predictor-corrector}
\label{app:predictor-corrector}

\paragraph{Gibbs samplers.}
Let $\pi$ be a probability distribution on a product space $\mathcal{X}_1 \times \cdots \times \mathcal{X}_{\len}$. For each coordinate $\ell \in \{1,\dots,\len\}$, denote by $\pi(x^\ell \mid x^{-\ell})$ the conditional distribution of the $\ell$-th coordinate given all the others. The Gibbs sampler associated with $\pi$ is the Markov chain obtained by repeatedly choosing a coordinate $\ell$ and replacing only $x^\ell$ by a fresh draw from $\pi(\cdot \mid x^{-\ell})$, while keeping $x^{-\ell}$ fixed. Equivalently, the one-coordinate Gibbs kernel is
\[
K_\ell(x,\mathrm{d}y)
=
\delta_{x^{-\ell}}(\mathrm{d}y^{-\ell}) \pi(\mathrm{d}y^\ell \mid x^{-\ell}).
\]
The invariance of $\pi$ under $K_\ell$ follows immediately from the definition of conditional distributions. Indeed, for any bounded measurable test function $\varphi$,
\begin{align*}
\int \pi(\mathrm{d}x) \int K_\ell(x,\mathrm{d}y)\varphi(y)
&= \int \pi(\mathrm{d}x) \int \pi(\mathrm{d}y^\ell \mid x^{-\ell}) \varphi(y^\ell,x^{-\ell}) \\
&= \int \pi(\mathrm{d}x^{-\ell}) \int \pi(\mathrm{d}x^\ell \mid x^{-\ell}) \int \pi(\mathrm{d}y^\ell \mid x^{-\ell}) \varphi(y^\ell,x^{-\ell}) \\
&= \int \pi(\mathrm{d}x^{-\ell}) \int \pi(\mathrm{d}y^\ell \mid x^{-\ell}) \varphi(y^\ell,x^{-\ell}) \\
&= \int \pi(\mathrm{d}y)\varphi(y).
\end{align*}
Hence $\pi K_\ell = \pi$. It follows that deterministic scan, which composes the kernels $K_1,\dots,K_{\len}$, and random scan, which averages them, also leave $\pi$ invariant.

At a fixed time $t$, we apply this construction with $\pi = \pdata{t}{}{}$. Identity \eqref{eq:gibbs-conditional} gives the corresponding one-coordinate conditional explicitly. Starting from a current state $\bx_t$, one chooses a coordinate $\ell$ and resamples only this coordinate according to
\[
\bx_t^\ell \sim \pdata{t}{\bx_t^{- \ell}}{}[\ell]
=
\fw{t\tbar 0}{\loodenoiser{t}{}{\bx_t}^\ell}{}[\ell],
\]
while keeping $\bx_t^{- \ell}$ unchanged. This yields a Gibbs sampler with invariant distribution $\pdata{t}{}{}$.

\paragraph{Predictor-corrector samplers.}
Predictor-corrector schemes were introduced in the continuous setting in \cite{song2021score}. They alternate a predictor step, which moves along the reverse dynamics, and a corrector step, which applies an MCMC kernel preserving the current marginal. They make use of the fact that during training, we learn the approximate score $s_\param(t, \bx_t) \approx \nabla \log \pdata{t}{}{\bx_t}$ to construct a Langevin dynamic $dX_t = s_\param(t, \bx_t)dt + \sqrt{2} dW_t$ which approximately leaves $\pdata{t}{}{}$ invariant. In the discrete setting, \cite{campbell2022continuous} proposed an analogue based on the reverse CTMC for the predictor and on a forward-backward CTMC for the corrector. The role of the corrector is to improve the sample quality at a fixed time before continuing the reverse trajectory. The identity \eqref{eq:gibbs-conditional} provides another such corrector: instead of using an uninformed transition that preserves $\pdata{t}{}{}$, one may apply Gibbs updates based on the exact one-coordinate conditionals of $\pdata{t}{}{}$.

\paragraph{Informed Gibbs corrector.}
This observation is particularly useful in UDM. Indeed, \Cref{prop:loo} shows that, under the \ref{eq:param_plug_in} parameterization, the network already learns the leave-one-out object needed to evaluate the Gibbs conditional. Therefore, unlike \cite{zhao2025informed}, no auxiliary leave-one-out model is required. Among exact Gibbs schemes, deterministic scan and random scan are the canonical choices. In practice, however, it is often more efficient to concentrate the corrector effort on the positions where the predictor is least confident. Following \cite{zhao2025informed}, we therefore use the margin score
\[
c_{\ell}^{\mathrm{margin}}(\bx_t)
\eqdef
\log \pdata{t}{\bx_t^{- \ell}}{\bx_t^{\ell}}[\ell]
- \max_{x \neq \bx_t^\ell} \log \pdata{t}{\bx_t^{- \ell}}{x}[\ell],
\]
which measures the gap between the current token and the most likely alternative. Low values of $c_{\ell}^{\mathrm{margin}}(\bx_t)$ correspond to `unlikely' positions, and the corrector updates the coordinates with the smallest margins first. This is the practical scheme used in our experiments. As in \cite{zhao2025informed}, one may also update several coordinates in parallel, in the spirit of Hogwild Gibbs sampling \cite{johnson2013hogwild}. This parallel version should be viewed as a heuristic acceleration of the exact single-coordinate Gibbs corrector.

\begin{algorithm}
\caption{Predictor-corrector sampler with margin-based Gibbs corrector}
\label{alg:pc}
\begin{algorithmic}[1]
\Require time grid $0 = \tk{0} < \tk{1} < \cdots < \tk{n} = 1$
\Require integers $(m_i)_{i=1}^n$, number of parallel updates $k$, and a model providing $\pdata{\sk{i}\tbar\tk{i}}{}{}[\param]$ and $\loodenoiser{\sk{i}}{}{}$
\Ensure a sample approximately distributed according to $\pdata{0}{}{}$
\State Draw $\bx_{\tk{n}} \sim \pdata{1}{}{}$
\For{$i=n,n-1,\dots,1$}
    \State Draw predictor state $\bx_{\sk{i}} \sim \pdata{\sk{i}\tbar\tk{i}}{\bx_{\tk{i}}}{}[\param]$
    \For{$r=1,\dots,m_i$}
        \For{$\ell=1,\dots,\len$} \Comment{in parallel over $\ell$}
            \State Define $\pdata{\sk{i}}{\bx_{\sk{i}}^{- \ell}}{\cdot}[\ell] \gets \fw{\sk{i}\tbar 0}{\loodenoiser{\sk{i}}{}{\bx_{\sk{i}}}^\ell}{}[\ell]$
            \State Compute $c_\ell \gets \log \pdata{\sk{i}}{\bx_{\sk{i}}^{- \ell}}{\bx_{\sk{i}}^\ell}[\ell] - \max_{x \neq \bx_{\sk{i}}^\ell} \log \pdata{\sk{i}}{\bx_{\sk{i}}^{- \ell}}{x}[\ell]$
        \EndFor
        \State Let $\ell_{r,1},\dots,\ell_{r,k}$ be the coordinates with smallest scores $c_\ell$
        \For{$j=1,\dots,k$}
            \State Resample $\bx_{\sk{i}}^{\ell_{r,j}} \sim \pdata{\sk{i}}{\bx_{\sk{i}}^{- \ell_{r,j}}}{\cdot}[\ell_{r,j}]$ in parallel
        \EndFor
    \EndFor
    \State Set $\bx_{\tk{i-1}} \eqdef \bx_{\sk{i}}$
\EndFor
\State \Return $\bx_0$
\end{algorithmic}
\end{algorithm}

\Cref{fig:predictor-corrector} shows that the Gibbs corrector significantly improves sample quality at a comparable computational cost. As already emphasized in \cite{campbell2022continuous}, predictor-corrector schemes are often particularly effective in the discrete setting. The reason is that a corrector step may modify the value of a token altogether, rather than merely perturbing the state locally. As a consequence, correcting a few uncertain positions can substantially alter the subsequent reverse trajectory.

This effect is typically stronger than in continuous domains, where corrector steps usually act as local refinements of an already smooth state. In discrete domains, by contrast, the uncertainty is concentrated on a finite set of alternatives, and resolving it can produce qualitatively different samples. This explains why corrector-based methods have become a central tool in discrete diffusion models; see for instance \cite{campbell2022continuous, zhao2025informed,wang2026remaskingdiscretediffusionmodels}.

\subsection{On the distinction between LOO and denoiser in \cite{deschenaux2026diffusiondualitychapterii}}
\label{sec:loo-denoiser-confusion}

In \cite{deschenaux2026diffusiondualitychapterii}, the authors introduce, for each position $\ell$, a family of superposition posteriors which, in the notation of the present paper, can be written as
\begin{equation}
\label{eq:psi-posterior}
\Psi^\ell _{s\tbar 0,t}(\cdot|\bx_0^\ell,\bx_t^\ell)
\eqdef
\kappa_t \fw{s\tbar 0,t}{\bx_0^\ell,\bx_t^\ell}{}[\ell]
+
(1-\kappa_t)\fw{s\tbar 0}{\bx_0^\ell}{}[\ell]
\eqsp.
\end{equation}
The first term is the usual reverse bridge and the second term is a forward noising step. After averaging \eqref{eq:psi-posterior} against the clean-token posterior, one obtains a reverse kernel with the same marginals as the original diffusion. In the masked case, different choices of $\kappa_t$ recover ReMDM and related predictor-corrector samplers \cite{wang2026remaskingdiscretediffusionmodels,campbell2022continuous}.
The reverse transition is given by  
\[
\txts \Psi_{s \tbar t}(\bx_s | \bx_t) = \sum_{\bx_0^\ell} \Psi^\ell _{s\tbar 0,t}(\bx_s | \bx_0^\ell,\bx_t^\ell) \, \pdata{0\tbar t}{\bx_t}{\bx_0^{\ell}}[\ell] \eqsp, 
\]
Following \Cref{prop:loo}, since DUO \cite{sahoo2025diffusion,deschenaux2026diffusiondualitychapterii} trains a LOO denoiser similarily to \cite{schiff2024simple}, the useful rewriting of the previous transition is
\begin{equation}
    \label{eq:psi-sampler-correct}
    \txts \Psi_{s \tbar t}(\bx_s | \bx_t) = \kappa_t \fw{s\tbar 0, t}{\denoiser{t}{}{\bx_t}[\mathrm{loo}], \bx_t}{\bx_s} +  (1 - \kappa_t) \fw{s\tbar 0}{\denoiser{t}{}{\bx_t}}{\bx_s} 
\end{equation}
and $\denoiser{t}{}{\bx_t}$ is the plain denoiser. On the other hand \cite[Eqn 12]{deschenaux2026diffusiondualitychapterii}, plugs inside the forward transition $\fw{s\tbar 0}{}{}{}$ the quantity $\fw{0\tbar 0, t}{\denoiser{t}{}{\bx_t}[\loo], \bx_t}{}$, \emph{i.e.} \eqref{eq:bridge-prob} with $s = 0$. It turns out that this is equivalent to the LOO denoiser to denoiser conversion formula \eqref{eq:loo-to-denoiser-uniform} and so \cite{deschenaux2026diffusiondualitychapterii} does indeed implement the transition \eqref{eq:psi-sampler-correct}. This is however valid only if we train the network to approximate a LOO denoiser. If instead we train a plain denoiser, then we need to first convert it into a LOO denoiser using \eqref{eq:denoiser-to-loo} in order to plug it in the bridge as in \eqref{eq:psi-sampler-correct}. 

%% file: appendix/ctmc.tex
\section{Continuous time perspective}
\label{sec:ctmc}
Here we consider a CTMC $(X_t)_{t\in[0, 1]}$ on $\msx$ with infinitesimal generator $\ratemat{t}{}{}$ defined, for any $\bx, \tilde\bx \in \msx$ as 
\begin{equation}
    \label{eq:factorization}
\ratemat{t}{\bx}{\tilde\bx} \eqdef \begin{cases} 
    \ratemat{t}{\bx^\ell}{\tilde\bx^\ell}[\ell] & \quad \text{if $\bx^{-\ell} \!=\! \tilde\bx^{-\ell}$ and $\bx^\ell \neq \tilde\bx^\ell$}  \\
    0  & \quad  \text{otherwise,}\\ 
\end{cases}
\end{equation}
and for all $\ell \in [1:\len]$, $\ratemat{t}{}{}[\ell]$ is the generator of the CTMC $(X^\ell _t)_{t\in[0,1]}$ on $\msv$ thus satisfying $\ratemat{t}{x}{\tilde{x}}[\ell] \geq 0$ if $x \neq \tilde{x}$ and $\sum_{\tilde{x} \in \msv} \ratemat{t}{x}{\tilde{x}}[\ell] = 0$. 
Writing $\beta_t \eqdef -\alpha_t'/\alpha_t \geq 0$, we consider the following form of the token-wise generator
\[
    \ratemat{t}{x}{\tilde{x}}[\ell]
    =
    \beta_t \dotp{\tilde{x}}{\pi}
    \quad \text{for } x \neq \tilde{x},
    \qquad
    \ratemat{t}{x}{x}[\ell]
    =
    -\beta_t(1-\dotp{x}{\pi}) \eqsp.
\]
Then, since $\exp(-\int_s^t \beta_u \rmd u) = \alpha_{t\tbar s}$ and using the Kolmogorov forward equation we find that the resulting conditional law of the sequence-level CTMC satisfies $X_t \mid X_s = \bx_s \sim \fw{t\tbar s}{\bx_s}{}$ for all $0 \le s < t \le 1$.
With $\ratemat{t}{x}{\tilde{x}}[\ell] = \beta_t/\vocab$, for all $\tilde{x}\neq x$, we recover the UDM transition. For the MDM transitions, the generator is $\ratemat{t}{x}{\onehot{\vocab}}[\ell] = \beta_t$ for $x\neq \onehot{\vocab}$ and all other off-diagonal rates vanish. 
The reverse CTMC starts from \(\pdata{1}{}{}\) at time \(1\) and evolves backward to reach $\pdata{0}{}{}$ at time $0$. Its backward transition kernels are \(\pdata{s\tbar t}{\bx_t}{\bx_s}\), and the off-diagonal elements of its generator are given by 
$
    \rratemat{t}{\bx}{\tilde{\bx}}
    \eqdef
    \ratemat{t}{\tilde{\bx}}{\bx}\, \pdata{t}{}{\tilde{\bx}} / \pdata{t}{}{\bx}
$. 
Next, the generator of the chain $(X_t)_{t\in[0, 1]}$ running backwards and conditioned on $X_0 = \bx_0$ is given by 
\begin{equation*}
\rratemat{t}{\bx_0, \bx}{\tilde{\bx}} \eqdef \begin{cases} 
    \beta_t \dotp{\bx^\ell}{\pi}\!\cdot\!\dotp{\tilde\bx^\ell}{\score{t}{\bx^\ell _0}{\bx^\ell}}& \quad \text{if $\bx^{-\ell} \!=\! \tilde\bx^{-\ell}$ and $\bx^\ell \neq \tilde\bx^\ell$}  \\
    0  & \quad  \text{otherwise,}\\ 
\end{cases}
\end{equation*}
where we define the conditional score 
\begin{equation}
    \label{eq:conditional-score}
    \score{t}{x_0}{x}
    \eqdef
    (\alpha_t x_0 + (1 - \alpha_t) \pi) / \fw{t\tbar0}{x_0}{x}
\end{equation}
if $\fw{t\tbar0}{x_0}{x} > 0$ and $0$ otherwise. Following \cite{campbell2022continuous}, we have the following identity 
\begin{equation}
    \label{eq:reverse-score}
\rratemat{t}{\bx}{\tilde{\bx}} \eqdef \begin{cases} 
    \beta_t \dotp{\bx^\ell}{\pi}\!\cdot\!\dotp{\tilde\bx^\ell}{\score{t}{}{\bx}^\ell}& \quad \text{if $\bx^{-\ell} \!=\! \tilde\bx^{-\ell}$ and $\bx^\ell \neq \tilde\bx^\ell$}  \\
    0  & \quad  \text{otherwise,}\\ 
\end{cases}
\end{equation}
where $\score{t}{}{\bx} \in \rset^{\len \times \vocab} _{\geq 0}$ and its $\ell$-th row is 
\begin{equation}
    \label{eq:score-definition}
\score{t}{}{\bx}^\ell \eqdef \sum_{\bx^\ell _0} \score{t}{\bx^\ell _0}{\bx^\ell} \pdata{0\tbar t}{\bx}{\bx^\ell _0}[\ell]. 
\end{equation}
Following SEDD, we consider next the parameterized reverse generator $\rratemat{t}{}{}[\param]$ defined similarly to \eqref{eq:reverse-score} but with the intractable score $\score{t}{}{}$ replaced with a model $\score{t}{}{}[\param]$. 
Following \cite{campbell2022continuous, lou2024discrete}, the continuous-time ELBO specialized to the generator \eqref{eq:reverse-score} is
\begin{multline}
    \label{eq:ctmc-specific-elbo}
    \mathcal L_{\infty}(\param)
     \eqdef
    \pE\!\left[
        \kldivergence{\fw{1\tbar 0}{X_0}{}}{\pdata{1}{}{}}
    \right]
    \\
    + \int_0^1
    \beta_t \,
    \pE\!\Bigg[
        \sum_{\ell=1}^{\len}
        \dotp{X_t^\ell}{\pi^\ell}
        \sum_{\tilde x \neq X_t^\ell}
        \Phi\Big(
            \dotp{\tilde\bx^\ell}{\score{t}{X_0^\ell}{X_t^\ell}},
            \dotp{\tilde\bx^\ell}{\score{t}{}{X_t}[\param]^\ell}
        \Big)
    \Bigg]\rmd t \eqsp.
\end{multline}
where
$
    \Phi(a,b)
    \eqdef
    b-a+a\log(a/b)
$
for $a,b>0$, with the usual convention $0 \cdot \log 0 = 0$, and both expectations are \wrt\ the joint law $X_0 \sim \pdata{0}{}{}$ and $X_t \sim \fw{t\tbar 0}{X_0}{}$.

\paragraph{Sampling from a CTMC} Sampling from the reverse CTMC can be done exactly with the Gillespie algorithm, but this updates only one position at a time. In practice one introduces a decreasing grid $1=t_N>\cdots>t_0=0$, writes $\Delta_i \eqdef t_i-t_{i-1}$, and approximates the reverse dynamics over each interval $[t_{i-1},t_i]$. Given the current state $\bx_i$ at time $t_i$, we freeze the model reverse rates over the interval and define
\begin{equation}
    \label{eq:reverse-ctmc-frozen-rates}
    \hat \lambda^\ell_i(\bx_i)
    \eqdef
    \sum_{\tilde\bx:\, \tilde\bx^{-\ell}=\bx_i^{-\ell},\, \tilde\bx^\ell \neq \bx_i^\ell}
    \rratemat{t_i}{\bx_i}{\tilde\bx}[\param],
    \qquad
    \hat \lambda_i(\bx_i)
    \eqdef
    \sum_{\ell=1}^{\len}\hat \lambda^\ell_i(\bx_i)
    \eqsp.
\end{equation}
The simplest discretization is the explicit Euler scheme. It keeps at most one jump over the whole interval and samples
\begin{equation}
    \label{eq:reverse-ctmc-euler-update}
    p^\mathrm{euler}_{i-1\tbar i}(\tilde\bx \mid \bx_i)
    \eqdef
    \indic_{\tilde\bx = \bx_i}\bigl(1-\Delta_i\hat \lambda_i(\bx_i)\bigr)
    +
    \Delta_i \rratemat{t_i}{\bx_i}{\tilde\bx}[\param]\indic_{\tilde\bx \neq \bx_i}
    \tag{Euler}
\end{equation}
provided $\Delta_i \hat \lambda_i(\bx_i) \leq 1$. Equivalently, one keeps $\bx_i$ with probability $1-\Delta_i\hat \lambda_i(\bx_i)$ and otherwise performs a single jump drawn from the normalized off-diagonal rates of $\rratemat{t_i}{}{}[\param]$.

Tau-leaping \cite{campbell2022continuous} keeps the same frozen rates but allows several positions to jump in parallel. Concretely, for every $\ell$ and every $\tilde x \neq \bx_i^\ell$, one draws independent Poisson variables
\[
    N_{i,\ell,\tilde x}
    \sim
    \mathrm{Poisson}\!\left(
        \Delta_i
        \rratemat{t_i}{\bx_i}{(\bx_i^{-\ell},\tilde x)}[\param]
    \right)
\]
and sets
\begin{equation}
    \label{eq:reverse-ctmc-tau-update}
    \bx_{i-1}^\ell
    \eqdef
    \begin{cases}
        \bx_i^\ell & \text{if $\sum_{\tilde x \neq \bx_i^\ell} N_{i,\ell,\tilde x} \neq 1$},\\
        \sum_{\tilde x \neq \bx_i^\ell} N_{i,\ell,\tilde x}\tilde x & \text{if $\sum_{\tilde x \neq \bx_i^\ell} N_{i,\ell,\tilde x} = 1$},
    \end{cases}
    \tag{$\tau$-leaping}
\end{equation}
independently over $\ell \in \intset{1}{\len}$. Thus, if exactly one jump occurs at a position, the token is replaced by its destination; if no jump or several jumps occur, the update is rejected at that position.

%% file: appendix/max_coupling.tex
\section{Uniform Diffusion and Maximal Coupling}
\label{app:max-coupling}

\subsection{Maximal couplings}

Let $P$ and $Q$ be two probability distributions on a finite state space $\mathcal X$. A coupling of $(P,Q)$ is a probability distribution $\gamma$ on $\mathcal X^2$ such that
\[
    \sum_{y \in \mathcal X} \gamma(x,y) = P(x),
    \qquad
    \sum_{x \in \mathcal X} \gamma(x,y) = Q(y).
\]
If $(X,Y) \sim \gamma$, a maximal coupling is, by definition, a coupling that minimizes the mismatch probability $\pP(X \neq Y)$ over all couplings of $(P,Q)$, namely $\inf_{\gamma \in \Gamma(P,Q)} \pP(X \neq Y),$ where $\Gamma(P,Q)$ denotes the set of couplings of $(P,Q)$.
It is standard that
\[
    \inf_{\gamma \in \Gamma(P,Q)} \pP(X \neq Y)
    =
    \|P-Q\|_{\mathrm{TV}}.
\]
Equivalently, since $\pP(X \neq Y)=1-\pP(X=Y)$, maximal couplings are exactly those that maximize the coincidence probability. In the discrete setting this is equivalent to
\[
    \sum_{x \in \mathcal X}\gamma(x,x)
    =
    \sum_{x \in \mathcal X}\min\{P(x),Q(x)\}
    =
    1 - \|P-Q\|_{\mathrm{TV}}.
\]

\subsection{Maximal couplings and masked diffusion} 

Masked diffusion provides a simple example of a maximal coupling at the token level. Fix $\ell \in \intset{1}{\len}$, $0 \leq s < t \leq 1$, and $x_0 \neq \mask$. The MDM forward transition is
\[
    \fw{t\tbar 0}{x_0;\mask}{x_t}[\ell]
    =
    \categorical\!\left(x;\alpha_t x_0+(1-\alpha_t)\mask\right)
    =
    \alpha_t \indic_{x=x_0}+(1-\alpha_t)\indic_{x=\mask}
    \eqsp.
\]
Consider the joint distribution with p.m.f. 
\begin{equation}
    \label{eq:maxcoupl-mdm}
    \txts
    \fw{s,t\tbar 0}{x_0;\mask}{x_s,x_t}[\ell]
    \eqdef
    \fw{s\tbar 0,t}{x_0,x_t;\mask}{x_s}[\ell]\,
    \fw{t\tbar 0}{x_0;\mask}{x_t}[\ell],
\end{equation}
which is explicitly
\[
    \txts
    \fw{s,t\tbar 0}{x_0;\mask}{x_s,x_t}[\ell]
    =
    \alpha_t \indic_{x_s=x_0}\indic_{x_t=x_0}
    +
    (\alpha_s-\alpha_t)\indic_{x_s=x_0}\indic_{x_t=\mask}
    +
    (1-\alpha_s)\indic_{x_s=\mask}\indic_{x_t=\mask}
    \eqsp.
\]
Therefore its coincidence probability is
\[
    \txts
    \sum_{x \in \msv}
    \fw{s,t\tbar 0}{x_0;\mask}{x,x}[\ell]
    =
    \alpha_t + 1-\alpha_s.
\]
Since $s<t$ implies $\alpha_s \geq \alpha_t$, the overlap of the two MDM marginals is
\[
    \txts
    \sum_{x \in \msv}
    \min\!\left\{
        \fw{s\tbar 0}{x_0;\mask}{x}[\ell],
        \fw{t\tbar 0}{x_0;\mask}{x}[\ell]
    \right\}
    =
    \alpha_t + 1-\alpha_s.
\]
Thus \eqref{eq:maxcoupl-mdm} is a maximal coupling between $\fw{s\tbar 0}{x_0;\mask}{}[\ell]$ and $\fw{t\tbar 0}{x_0;\mask}{}[\ell]$.

\subsection{Maximum coupling bridge and ELBO for uniform diffusion}
\label{sec:max-coupling}
We now introduce a token-wise bridge for uniform diffusion that is a maximal coupling between the forward marginals while preserving the same one-time marginals as the original forward process. We also derive the corresponding continuous-time reverse rate matrix and the associated CTMC ELBO. Interestingly, in their Appendix A, \cite{song2021ddim} derive a bridge that can modulate the stochasticity of the reverse process and that recovers the same maximal coupling bridge in the limit where one sets the coefficient $\sigma_t$ to $\frac{1 - \alpha_{t-1}}{1 - \alpha_{t}}$.

\begin{proposition}
    \label{prop:uniform-max-coupling-bridge}
    Fix $\ell \in \intset{1}{\len}$, $x_0 \in \msv$ and $0 \le s < t \le 1$. Consider the bridge
    \begin{equation}
        \label{eq:frozen-audm-bridge}
        \fw{s\tbar 0,t}{x_0,x_t}{x_s}[\ell]
        =
            \categorical\!\left(x_s;
                \frac{\alpha_s-\alpha_t}{1-\alpha_t}x_0
                +
                \frac{1-\alpha_s}{1-\alpha_t}x_t
            \right) \eqsp.
    \end{equation}
    The joint distribution with p.m.f.
    \[
        \fw{s,t\tbar 0}{x_0}{x_s,x_t}[\ell]
        \eqdef
        \fw{s\tbar 0,t}{x_0,x_t}{x_s}[\ell]
        \fw{t\tbar 0}{x_0}{x_t}[\ell]
    \]
    satisfies
    \begin{align}
        \label{eq:max-coupling-pair}
        \fw{s,t\tbar 0}{x_0}{x_s,x_t}[\ell]
        &=
        \alpha_t \indic_{x_s = x_0}\indic_{x_t = x_0}
        \nonumber\\
        &\quad +
        \frac{\alpha_s-\alpha_t}{\vocab}\indic_{x_s = x_0}
        +
        \frac{1-\alpha_s}{\vocab}\indic_{x_t = x_s}.
    \end{align}
    Moreover, it preserves the forward marginals,
    \[
        \sum_{x_s \in \msv}
        \fw{s,t\tbar 0}{x_0}{x_s,x_t}[\ell]
        =
        \fw{t\tbar 0}{x_0}{x_t}[\ell],
        \qquad
        \sum_{x_t \in \msv}
        \fw{s,t\tbar 0}{x_0}{x_s,x_t}[\ell]
        =
        \fw{s\tbar 0}{x_0}{x_s}[\ell],
    \]
    and is a maximal coupling between $\fw{s\tbar 0}{x_0}{}[\ell]$ and $\fw{t\tbar 0}{x_0}{}[\ell]$.
\end{proposition}

\begin{proof}
    If $x_t = x_0$, the bridge in \eqref{eq:frozen-audm-bridge} is deterministic and therefore
    \[
        \fw{s,t\tbar 0}{x_0}{x_s,x_t}[\ell]
        =
        \left(
            \alpha_t + \frac{1-\alpha_t}{\vocab}
        \right)
        \indic_{x_s = x_0}\indic_{x_t = x_0}.
    \]
    If $x_t \neq x_0$, then
    \[
        \fw{t\tbar 0}{x_0}{x_t}[\ell]
        =
        \frac{1-\alpha_t}{\vocab},
    \]
    so multiplying the second line of \eqref{eq:frozen-audm-bridge} by $\fw{t\tbar 0}{x_0}{x_t}[\ell]$ gives
    \[
        \fw{s,t\tbar 0}{x_0}{x_s,x_t}[\ell]
        =
        \frac{\alpha_s-\alpha_t}{\vocab}\indic_{x_s = x_0}
        +
        \frac{1-\alpha_s}{\vocab}\indic_{x_s = x_t}.
    \]
    Combining the two cases yields \eqref{eq:max-coupling-pair}. We next verify the marginals. Summing \eqref{eq:max-coupling-pair} over $x_s$ gives
    \begin{align*}
        \sum_{x_s \in \msv}
        \fw{s,t\tbar 0}{x_0}{x_s,x_t}[\ell]
        &=
        \alpha_t \indic_{x_t = x_0}
        +
        \frac{\alpha_s-\alpha_t}{\vocab}
        +
        \frac{1-\alpha_s}{\vocab}\\
        &=
        \alpha_t \indic_{x_t = x_0}
        +
        \frac{1-\alpha_t}{\vocab}
        =
        \fw{t\tbar 0}{x_0}{x_t}[\ell].
    \end{align*}
    Summing instead over $x_t$ gives
    \begin{align*}
        \sum_{x_t \in \msv}
        \fw{s,t\tbar 0}{x_0}{x_s,x_t}[\ell]
        &=
        \alpha_t \indic_{x_s = x_0}
        +
        \frac{\vocab(\alpha_s-\alpha_t)}{\vocab}\indic_{x_s = x_0}
        +
        \frac{1-\alpha_s}{\vocab}\\
        &=
        \alpha_s \indic_{x_s = x_0}
        +
        \frac{1-\alpha_s}{\vocab}
        =
        \fw{s\tbar 0}{x_0}{x_s}[\ell].
    \end{align*}

    It remains to prove maximality. Under the coupling above,
    \[
        \sum_{x \in \msv}
        \fw{s,t\tbar 0}{x_0}{x,x}[\ell]
        =
        \alpha_t + (1-\alpha_s) + \frac{\alpha_s-\alpha_t}{\vocab}.
    \]
    On the other hand, for the two uniform-diffusion marginals we have
    \[
        \fw{s\tbar 0}{x_0}{x}[\ell]
        =
        \alpha_s \indic_{x = x_0} + \frac{1-\alpha_s}{\vocab},
        \qquad
        \fw{t\tbar 0}{x_0}{x}[\ell]
        =
        \alpha_t \indic_{x = x_0} + \frac{1-\alpha_t}{\vocab},
    \]
    so
    \begin{align*}
        \sum_{x \in \msv}
        \min\!\left\{
            \fw{s\tbar 0}{x_0}{x}[\ell],
            \fw{t\tbar 0}{x_0}{x}[\ell]
        \right\}
        &=
        \alpha_t + \frac{1-\alpha_t}{\vocab}
        +
        (\vocab-1)\frac{1-\alpha_s}{\vocab}\\
        &=
        \alpha_t + (1-\alpha_s) + \frac{\alpha_s-\alpha_t}{\vocab}.
    \end{align*}
    Hence the probability of equality reaches the maximal-coupling upper bound.
\end{proof}

\begin{proposition}
    \label{prop:uniform-max-coupling-rate}
    Let $\beta_t \eqdef -\alpha'_t / \alpha_t$. The token-wise backward generator conditioned on $X^\ell_0 = x_0$ and associated with the bridge \eqref{eq:frozen-audm-bridge} is, for $\tilde x \neq x$,
    \[
        \rratemat{t}{x, x_0}{\tilde x}[\mathrm{mc}]
        =
        \frac{\alpha_t\beta_t}{1-\alpha_t}
        \indic_{x \neq x_0}
        \indic_{\tilde x = x_0}.
    \]
    The diagonal term is given by row-sum zero. Consequently, the exact reverse generator obtained after averaging against the denoising posterior is
    \[
        \rratemat{t}{\bx}{\tilde\bx}[\mathrm{mc}]
        =
        \begin{cases}
            \frac{\alpha_t\beta_t}{1-\alpha_t}
            \pdata{0\tbar t}{\bx}{\tilde\bx^\ell}[\ell]
            & \text{if $\bx^{-\ell} = \tilde\bx^{-\ell}$ and $\bx^\ell \neq \tilde\bx^\ell$},\\
            0 & \text{otherwise}.
        \end{cases}
    \]
\end{proposition}

\begin{proof}
    Fix $x_0$ and let $h > 0$ be such that $t-h \ge 0$. If $x = x_0$, then the first line of \eqref{eq:frozen-audm-bridge} gives
    \[
        \fw{t-h\tbar 0,t}{x_0,x}{\tilde x}[\ell]
        =
        \indic_{\tilde x = x},
    \]
    so there is no off-diagonal jump. If $x \neq x_0$, the bridge gives
    \[
        \fw{t-h\tbar 0,t}{x_0,x}{\tilde x}[\ell]
        =
        \frac{\alpha_{t-h}-\alpha_t}{1-\alpha_t}\indic_{\tilde x = x_0}
        +
        \frac{1-\alpha_{t-h}}{1-\alpha_t}\indic_{\tilde x = x}.
    \]
    Since $\alpha_{t-h} = \alpha_t + h\beta_t\alpha_t + o(h)$, for every $\tilde x \neq x$ we obtain
    \[
        \lim_{h \downarrow 0}
        \frac{
            \fw{t-h\tbar 0,t}{x_0,x}{\tilde x}[\ell]
        }{h}
        =
        \frac{\alpha_t\beta_t}{1-\alpha_t}
        \indic_{\tilde x = x_0}.
    \]
    This proves the conditioned rate formula.

    For the exact reverse generator, let $\bx,\tilde\bx \in \msx$ differ only at position $\ell$, with $\bx^\ell \neq \tilde\bx^\ell$. By construction,
    \[
        \pdata{t-h\tbar t}{\bx}{\tilde\bx}
        =
        \sum_{\bx_0}
        \fw{t-h\tbar 0,t}{\bx_0,\bx}{\tilde\bx}
        \pdata{0\tbar t}{\bx}{\bx_0}.
    \]
    Because the bridge factorizes across positions and $\tilde\bx$ differs from $\bx$ only at $\ell$, the only nonzero off-diagonal contribution comes from clean tokens such that $\bx^\ell_0 = \tilde\bx^\ell$. Therefore
    \[
        \lim_{h \downarrow 0}
        \frac{\pdata{t-h\tbar t}{\bx}{\tilde\bx}}{h}
        =
        \frac{\alpha_t\beta_t}{1-\alpha_t}
        \sum_{\bx_0:\, \bx^\ell_0 = \tilde\bx^\ell}
        \pdata{0\tbar t}{\bx}{\bx_0}
        =
        \frac{\alpha_t\beta_t}{1-\alpha_t}
        \pdata{0\tbar t}{\bx}{\tilde\bx^\ell}[\ell],
    \]
    which is the claimed formula.
\end{proof}

Interestingly, this backward rate matrix coincides with the one obtained in \cite{campbell2024generativeflowsdiscretestatespaces}. Even though it arises from a different perspective, in both cases, the effect is to reduce the stochasticity of the reverse process by making the token-wise coupling maximal while preserving the correct forward marginals.

We now derive the corresponding continuous-time ELBO. We parameterize the reverse generator by a denoiser model $\denoiser{t}{}{\bx}[\param] \in (\Simplex{\vocab})^{\len}$ and define
\begin{equation}
    \label{eq:max-coupling-reverse-generator}
    \rratemat{t}{\bx}{\tilde\bx}[\param]
    \eqdef
    \begin{cases}
        \frac{\alpha_t\beta_t}{1-\alpha_t}
        \dotp{\tilde\bx^\ell}{\denoiser{t}{}{\bx}[\param]^\ell}
        & \text{if $\bx^{-\ell} = \tilde\bx^{-\ell}$ and $\bx^\ell \neq \tilde\bx^\ell$},\\
        0 & \text{otherwise.}
    \end{cases}
\end{equation}

\begin{proposition}
    \label{prop:max-coupling-ctmc-loss}
    Up to a $\param$-independent constant, the continuous-time ELBO associated with the maximal-coupling reverse generator \eqref{eq:max-coupling-reverse-generator} is
    \begin{multline}
        \label{eq:max-coupling-ctmc-loss}
        \txts \mathcal L^{\maxcoupling}_{\infty}(\param)
        \eqdef
        - \int_0^1
        \frac{\alpha'_t}{1-\alpha_t}\,
        \pE\big[
            \sum_{\ell=1}^{\len}
            \big\{
                1-\dotp{X_t^\ell}{\denoiser{t}{}{X_t}[\param]^\ell}
                \\
                -
                \indic_{X_t^\ell \neq X_0^\ell}
                \big(
                    1 + \log \dotp{X_0^\ell}{\denoiser{t}{}{X_t}[\param]^\ell}
                \big)
            \big\}
        \big]\rmd t,
    \end{multline}
    where the expectation is \wrt\ $X_0 \sim \pdata{0}{}{}$ and $X_t \sim \fw{t\tbar 0}{X_0}{}$.
\end{proposition}

\begin{proof}
    By \Cref{prop:uniform-max-coupling-rate}, conditioned on $(X_0^\ell, X_t^\ell)$, the token-wise off-diagonal reverse rate is
    \[
        \frac{\alpha_t\beta_t}{1-\alpha_t}
        \indic_{X_t^\ell \neq X_0^\ell}
        \indic_{\tilde x = X_0^\ell}
    \]
    for every $\tilde x \neq X_t^\ell$. Therefore, exactly as in the derivation of \eqref{eq:ctmc-specific-elbo}, the contribution of position $\ell$ at time $t$ to the continuous-time ELBO is
    \[
        \frac{\alpha_t\beta_t}{1-\alpha_t}
        \sum_{\tilde x \neq X_t^\ell}
        \Phi\!\left(
            \indic_{X_t^\ell \neq X_0^\ell}\indic_{\tilde x = X_0^\ell},
            \dotp{\tilde x}{\denoiser{t}{}{X_t}[\param]^\ell}
        \right).
    \]
    If $X_t^\ell = X_0^\ell$, the first argument of $\Phi$ is zero for every $\tilde x \neq X_t^\ell$, hence
    \begin{align*}
        \sum_{\tilde x \neq X_t^\ell}
        \Phi\!\left(
            0,
            \dotp{\tilde x}{\denoiser{t}{}{X_t}[\param]^\ell}
        \right)
        &=
        \sum_{\tilde x \neq X_t^\ell}
        \dotp{\tilde x}{\denoiser{t}{}{X_t}[\param]^\ell}\\
        &=
        1-\dotp{X_t^\ell}{\denoiser{t}{}{X_t}[\param]^\ell}.
    \end{align*}
    If $X_t^\ell \neq X_0^\ell$, only the term $\tilde x = X_0^\ell$ has a nonzero first argument, and therefore
    \begin{align*}
        &\sum_{\tilde x \neq X_t^\ell}
        \Phi\!\left(
            \indic_{\tilde x = X_0^\ell},
            \dotp{\tilde x}{\denoiser{t}{}{X_t}[\param]^\ell}
        \right)\\
        &\qquad=
        \sum_{\tilde x \neq X_t^\ell}
        \dotp{\tilde x}{\denoiser{t}{}{X_t}[\param]^\ell}
        -
        1
        -
        \log \dotp{X_0^\ell}{\denoiser{t}{}{X_t}[\param]^\ell}\\
        &\qquad=
        -\dotp{X_t^\ell}{\denoiser{t}{}{X_t}[\param]^\ell}
        -
        \log \dotp{X_0^\ell}{\denoiser{t}{}{X_t}[\param]^\ell}.
    \end{align*}
    Combining the two cases gives
    \[
        1-\dotp{X_t^\ell}{\denoiser{t}{}{X_t}[\param]^\ell}
        -
        \indic_{X_t^\ell \neq X_0^\ell}
        \big(
            1+\log \dotp{X_0^\ell}{\denoiser{t}{}{X_t}[\param]^\ell}
        \big).
    \]
    Multiplying by $\alpha_t\beta_t/(1-\alpha_t) = -\alpha'_t/(1-\alpha_t)$ and summing over $\ell$ yields exactly \eqref{eq:max-coupling-ctmc-loss}.
\end{proof}


%% file: appendix/loo_sensitivity.tex
\section{Leave-one-out and Hollow Transformers}
\label{app:loo-sensitivity}

As noted in \Cref{sec:loo}, the $\ell$-th component of the leave-one-out posterior is independent of the local observation $\bx_t^\ell$. This is a necessary structural property of the optimum identified in \Cref{prop:loo}. It therefore has two practical consequences. First, it provides a direct diagnostic for suboptimality: after converting a trained model to the leave-one-out representation, any residual dependence of the $\ell$-th output on $\bx_t^\ell$ certifies that the model has not reached the leave-one-out optimum. Second, it suggests an architectural design: instead of hoping that optimization learns this invariance, one may try to enforce it by construction.

\paragraph{LOO sensitivity diagnostic.} We check the sensitivity of the leave-one-out posterior to the local observation $\bx_t^\ell$ on some of the models that we have trained. More precisely, we vary $\bx_t^\ell$ while keeping $\bx_t^{-\ell}$ fixed and measure the induced change in the $\ell$-th leave-one-out prediction. The results are reported in \Cref{fig:loo-sensitivity}.

\begin{figure}[t]
    \centering
    \begin{subfigure}[b]{0.48\textwidth}
        \centering
        \includegraphics[width=\textwidth]{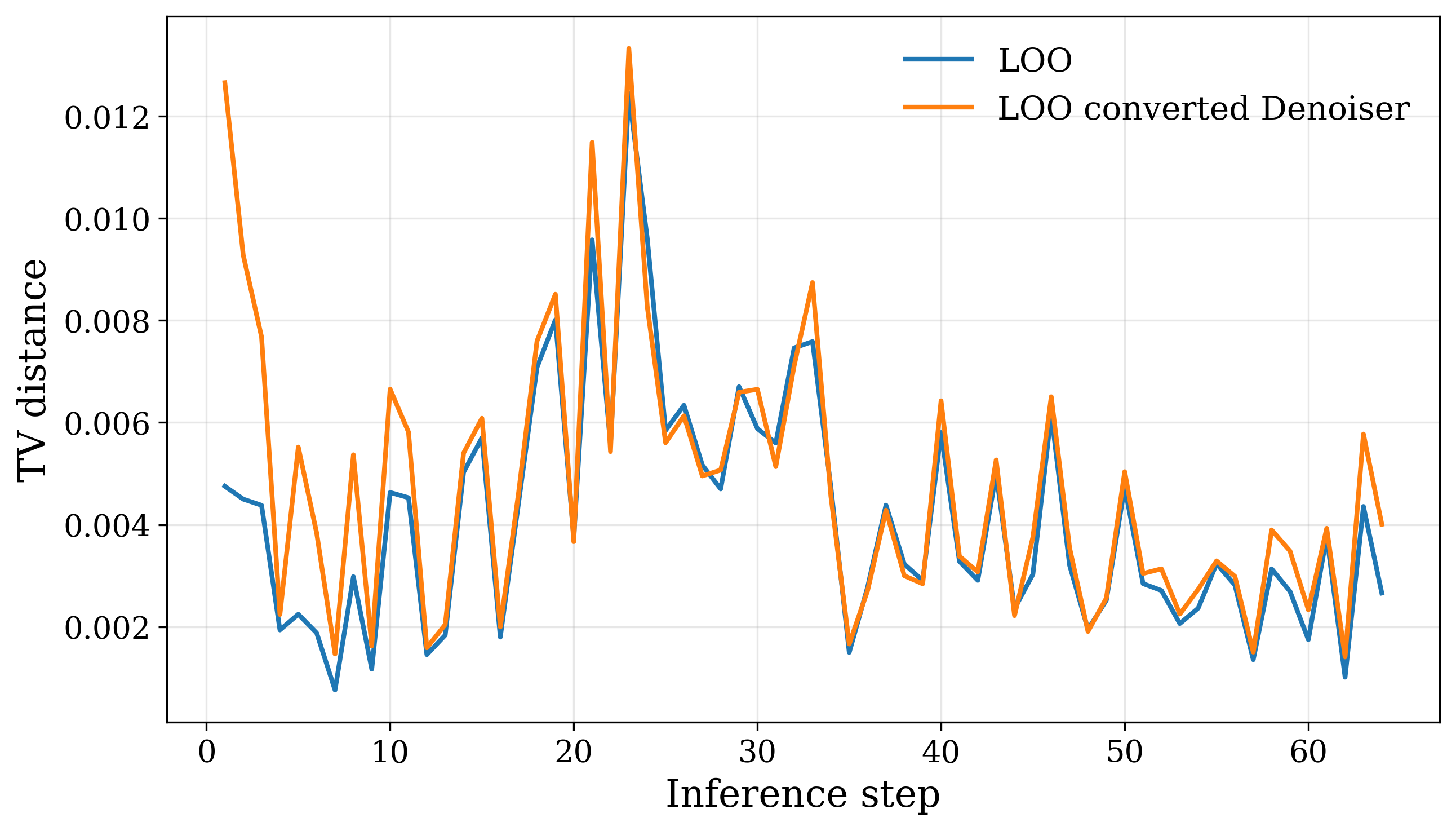}
        \caption{MNIST, converted denoiser}
        \label{fig:loo-sensitivity-mnist-converted}
    \end{subfigure}
    \hfill
    \begin{subfigure}[b]{0.48\textwidth}
        \centering
        \includegraphics[width=\textwidth]{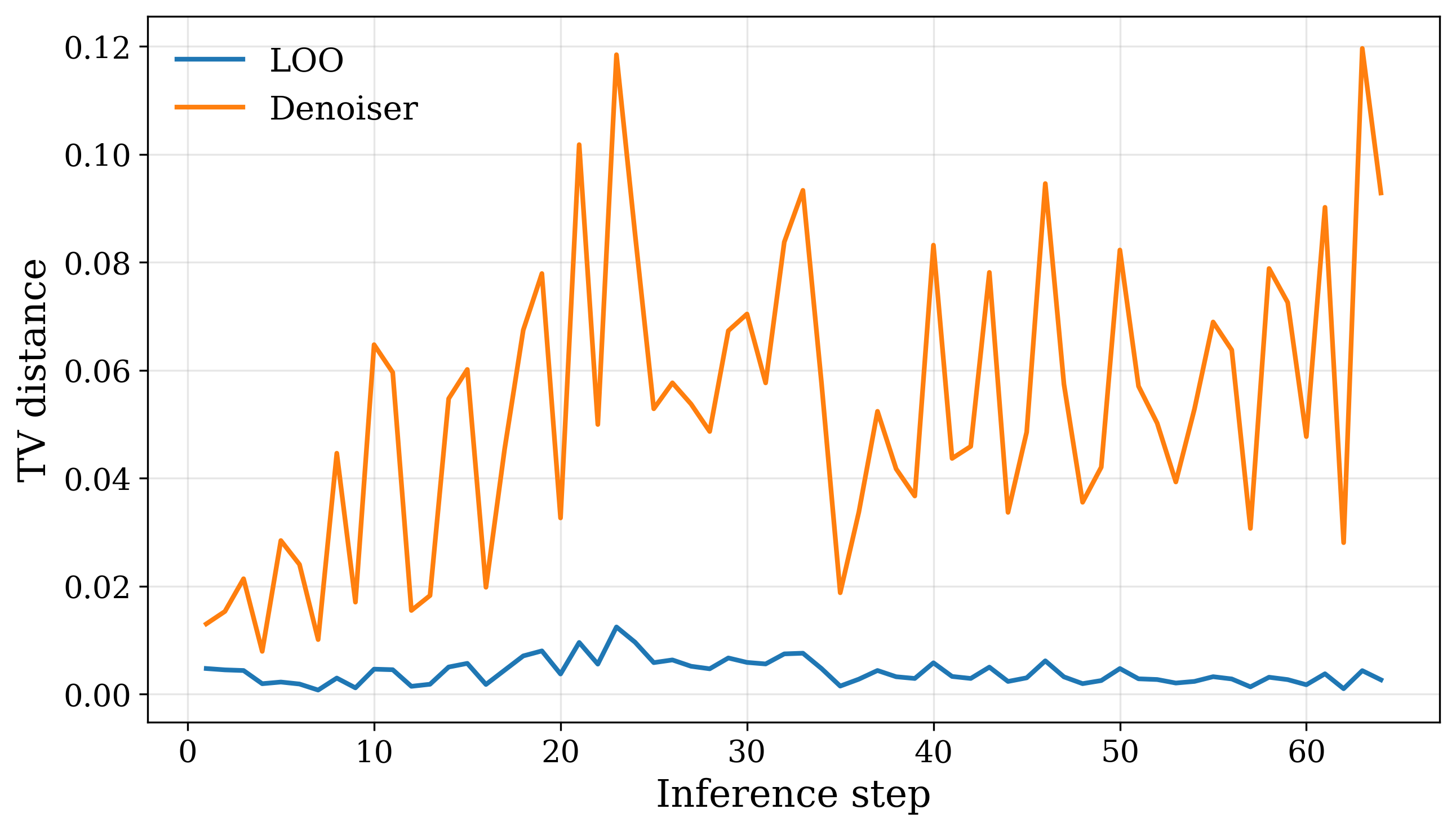}
        \caption{MNIST, direct leave-one-out}
        \label{fig:loo-sensitivity-mnist-loo}
    \end{subfigure}
    
    \vspace{0.5em}
    
    \begin{subfigure}[b]{0.48\textwidth}
        \centering
        \includegraphics[width=\textwidth]{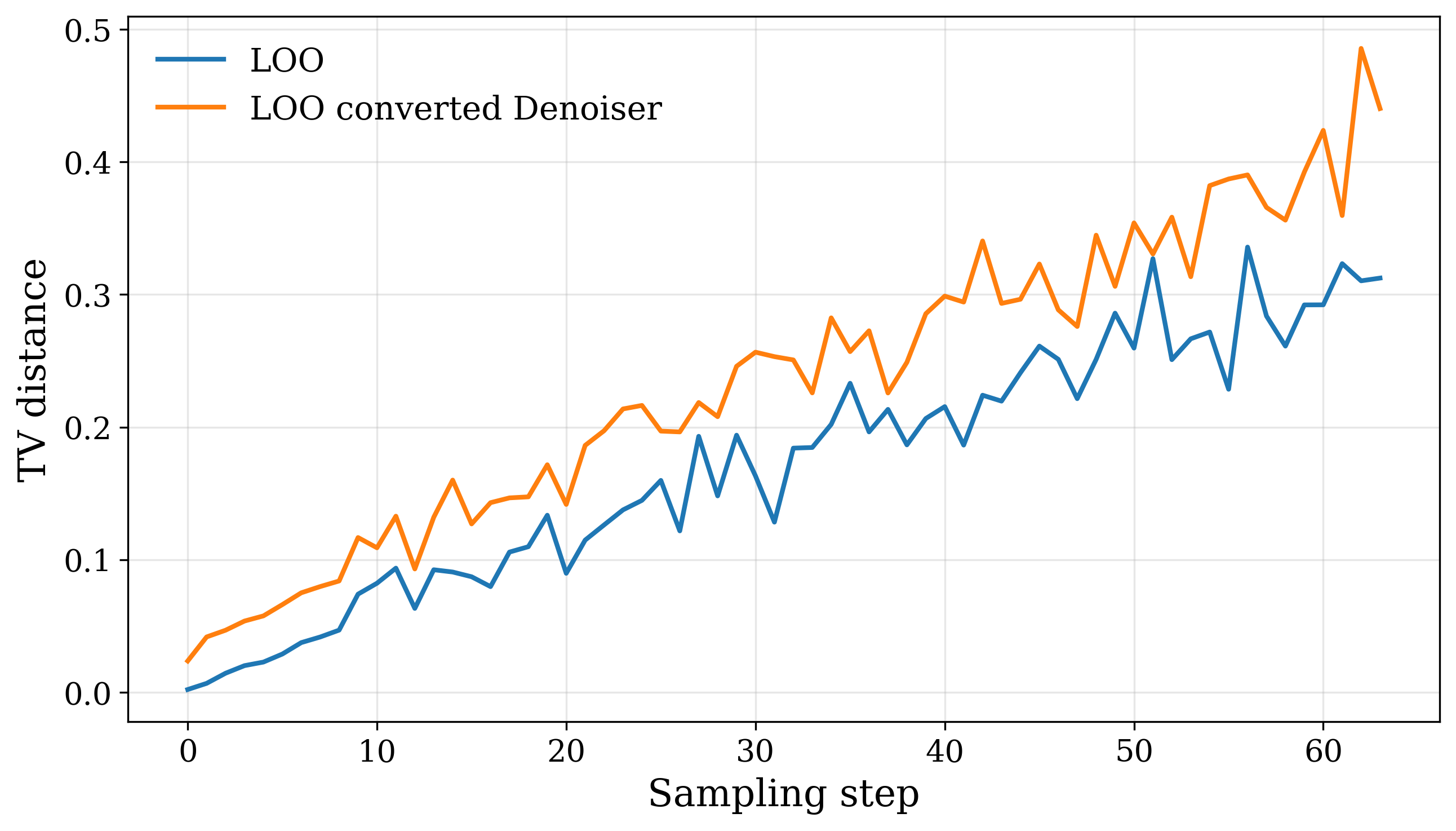}
        \caption{OWT, converted denoiser}
        \label{fig:loo-sensitivity-owt-converted}
    \end{subfigure}
    \hfill
    \begin{subfigure}[b]{0.48\textwidth}
        \centering
        \includegraphics[width=\textwidth]{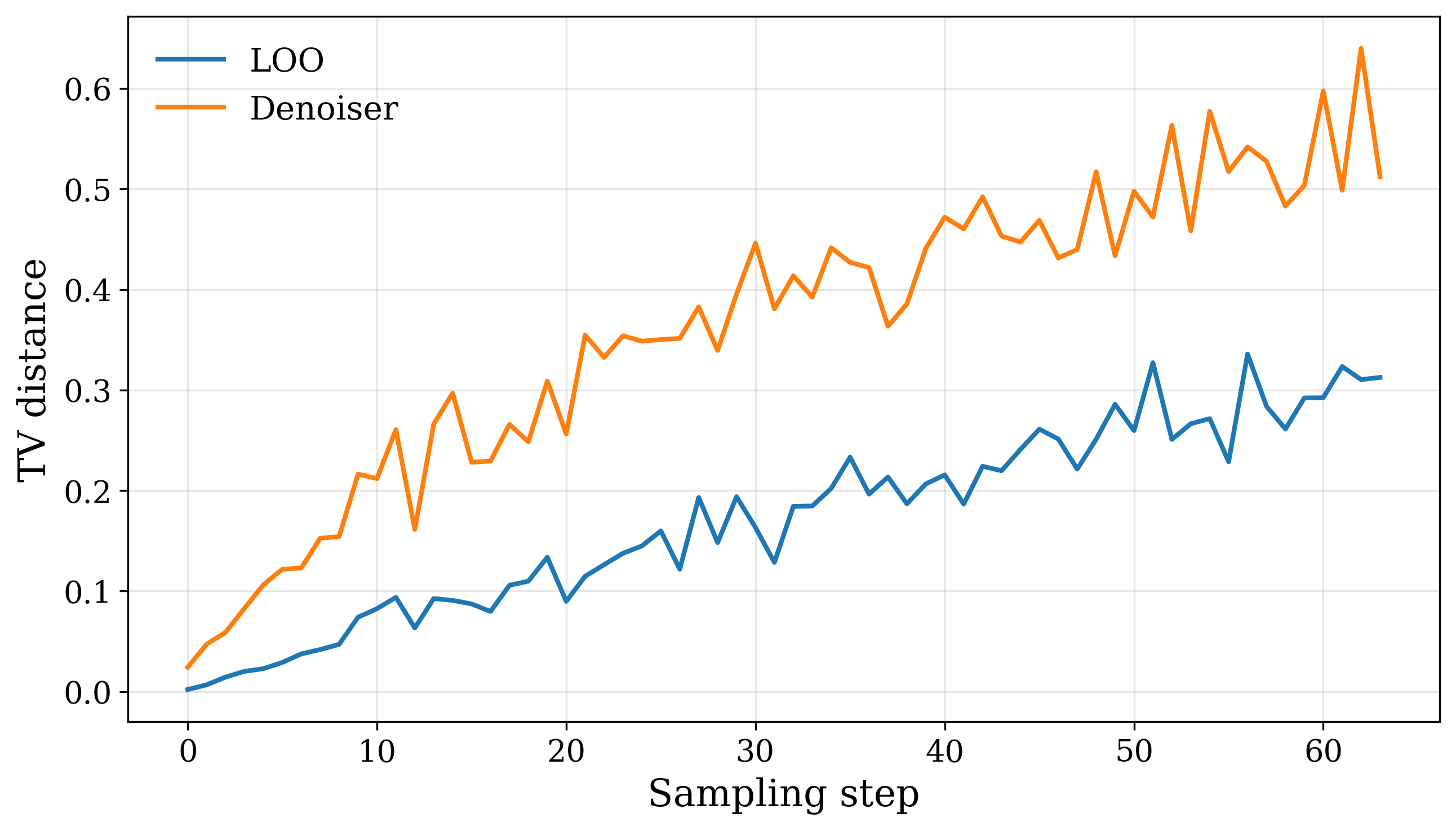}
        \caption{OWT, direct leave-one-out}
        \label{fig:loo-sensitivity-owt-loo}
    \end{subfigure}
    \caption{Sensitivity of the leave-one-out prediction to the local observation $\bx_t^\ell$. On MNIST, the converted denoiser and the directly trained leave-one-out model are both nearly insensitive to $\bx_t^\ell$. On OWT, the sensitivity remains visible for both models, but it is lower for the directly trained leave-one-out predictor than for the converted denoiser. This indicates that the leave-one-out invariance is not automatically recovered in the high-dimensional regime, and that proximity to this invariance is informative about the quality of the attained optimum.}
    \label{fig:loo-sensitivity}
\end{figure}

On low-dimensional examples such as MNIST, the leave-one-out posterior is indeed nearly insensitive to the local observation. Moreover, the denoiser converted into the leave-one-out representation and the leave-one-out predictor trained directly produce almost identical sensitivity curves. This strongly suggests that both models are already close to their respective optima, in which case the conversion formula of \Cref{prop:loo_to_denoiser} is effectively exact.

On larger sequence models, such as the transformer trained on OWT, the sensitivity remains clearly nonzero. This shows that in realistic high-dimensional settings the leave-one-out invariance is not automatically recovered by finite-capacity models trained with standard objectives. More generally, the sensitivity diagnostic appears to be well aligned with downstream performance: models that are closer to the leave-one-out optimum (using the local sensitivity as proxy) also tend to be the ones that perform better generatively.

This observation naturally suggests enforcing the invariance directly in the architecture. For transformers, however, doing so while preserving access to the rest of the context is not entirely straightforward.

\paragraph{Hollow transformer.} A proposal for such an architecture is the hollow transformer introduced by \cite{sun2023scorebased} and later used by \cite{zhao2025informed}. The idea is to combine two autoregressive streams, one left-to-right and one right-to-left, and to offset the representations so that the output at position $\ell$ never attends to the input token at the same position, while still depending on all the other positions. Equivalently, in the continuous relaxation used to describe the architecture, the $\ell$-th output is constrained to satisfy $\partial f(\bx_t, t)^\ell / \partial \bx_t^\ell = 0$. This makes the architecture structurally compatible with the leave-one-out target.

In \cite{zhao2025informed}, this construction is used in the absorbing setting to learn the additional leave-one-out information needed by the informed corrector without modifying the ELBO objective. The same motivation applies here: if the target at position $\ell$ should ignore $\bx_t^\ell$, it is appealing to remove that dependence by design rather than by optimization.

This inductive bias, however, is not free. Besides requiring two autoregressive streams, the hollow transformer also restricts the information flow relative to a standard transformer, since the output at position $\ell$ cannot aggregate the full context in a single attention step. More precisely, when computing the $\ell$-th output, the tokens to the left and to the right of $\ell$ interact within their respective groups, but cross-group interactions can only occur through deeper layers. This weakens the directness of the attention mechanism and makes the prediction problem harder. In our own experiments on LM1B, this translated into worse optimization and worse final performance than the corresponding standard DiT models, in line with the observations of \cite{zhao2025informed}. We therefore do not use the hollow transformer in the main experiments. Our view is that it remains a conceptually clean way to enforce the leave-one-out structure, but in large-scale language modeling this structural benefit can be outweighed by the optimization cost.

%% file: appendix/experiments.tex
\section{Experimental Details}
\label{app:exp-details}
  
\subsection{Setup and hyperparameters}

We report here the details relevant to the language-model experiments. We use the same setting as \cite{lou2024discrete,sahoo2024simple}. All models are trained with the same DDiT backbone and the same optimization scheme, and the only dataset-specific architectural change is the sequence length. Both LM1B and OWT are trained from scratch for $10^6$ steps with global batch size $512$, Adam with learning rate $3 \times 10^{-4}$, $\beta_1 = 0.9$, $\beta_2 = 0.999$, $\epsilon = 10^{-8}$, no weight decay, gradient clipping at $1.0$, and a constant learning-rate schedule with $2500$ warmup steps. Training uses bfloat16, exponential moving average with decay $0.9999$, antithetic time sampling, and sampling floor $\epsilon = 10^{-3}$. Unless stated otherwise, all experiments use seed $1$. For generative evaluation we report GPT-2 Large generative perplexity. The sampling uses higher precision $\mathrm{float64}$. These hyperparameters are summarized in \Cref{tab:exp-training}.

\begin{table}[t]
\centering
\captionsetup{font=small}
\caption{Training setup for the language-model experiments.}
\label{tab:exp-training}
\makebox[\linewidth][c]{%
\resizebox{0.6\linewidth}{!}{%
\begin{tabular}{lll}
\toprule
Setting & LM1B / OWT \\
\midrule
Backbone & DDiT \\
Hidden size & $768$ \\
Conditioning dimension & $128$ \\
Number of blocks & $12$ \\
Number of heads & $12$ \\
Dropout & $0.1$ \\
Training steps & $10^6$ \\
Global batch size & $512$ \\
Optimizer & Adam \\
Learning rate & $3 \times 10^{-4}$ \\
$\beta_1, \beta_2, \epsilon$ & $0.9, 0.999, 10^{-8}$ \\
Weight decay & $0$ \\
Learning-rate schedule & constant with $2500$ warmup steps \\
Precision & bfloat16 \\
Gradient clipping & $1.0$ \\
EMA decay & $0.9999$ \\
Antithetic time sampling & yes \\
Sampling floor & $10^{-3}$ \\
Noise schedule & linear schedule \\
\bottomrule
\end{tabular}
}}
\end{table}

\begin{table}[t]
  \centering
  \captionsetup{font=small}
  \caption{Sampling runtime in seconds per sample for the predictor-only temperature and nucleus frontiers, and the predictor-corrector runs. Nucleus is split into filtered top-$p$ sampling ($p<1$) and the unfiltered case ($p=1$).}
  \label{tab:corrector_runtime_per_sample}
  \resizebox{0.6\linewidth}{!}{%
  \begin{tabular}{lrrrr}
  \toprule
  Method & NFE 128 & NFE 256 & NFE 512 & NFE 1024 \\
  \midrule
  Temperature & 8.512 & 15.895 & 30.602 & 60.316 \\
  Nucleus ($p=1$) & 8.410 & 15.883 & 30.402 & 59.492 \\
  Nucleus ($p<1$) & 27.141 & 52.969 & 104.504 & 207.531 \\
  Predictor-corrector & 7.621 & 13.945 & 26.543 & 53.199 \\
  \bottomrule
  \end{tabular}
  }
\end{table}


LM1B is trained and evaluated on the same dataset family with the BERT tokenizer, with sequence length $L=128$. OWT is trained on the OpenWebText train split and validated on the held-out OpenWebText validation split with the GPT-2 tokenizer, with sequence length $L=1024$. The zero-shot perplexity evaluation of \Cref{tab:ppl-eval-multidataset_ppl} uses the OWT checkpoints and evaluates them on OWT together with AG News \cite{zhang2016characterlevelconvolutionalnetworkstext}, PubMed \cite{cohan2018discourseawareattentionmodelabstractive}, Lambada \cite{paperno2016lambadadatasetwordprediction}, WikiText \cite{merity2016pointersentinelmixturemodels}, PTB \cite{marcus-etal-1993-building}, and LM1B retokenized with GPT-2, all at sequence length $1024$. We use sentence packing for all datasets.

In the absorbing-uniform variants, the noised data and the uniform noise variable $(x_t, u_t)$ are first embedded with the same embedding table with hidden size $d$, concatenated together and with $\indic_{x_t = u_t}$, forming a vector of dimension $2d + 1$, and then passed through an MLP with hidden dimension $4d$ and output dimension $d$. This adds 7M parameters, around $4\%$ of the total parameter count. 



Across all comparisons, we keep the previous training setup fixed and vary only the modeling choices that are central to the paper: the training loss, the prediction target and the associated reverse-process parameterization. The denoiser-based UDM baseline uses the standard denoiser prediction, while the leave-one-out model replaces this with the leave-one-out target. The maximal-coupling models change only the path coupling. The masked baseline uses the standard absorbing-state parameterization of \cite{sahoo2024simple} with carry-over unmasking and zero-masking probability. We provide a summary of the available modeling choices in \Cref{tab:exp-algos}.


The sampling frontiers are obtained either with temperature sampling or top-$p$ sampling \citep{Holtzman2020The} at fixed checkpoint, and plotting generative perplexity against the resulting entropy.
For top-$p$ frontiers we use $p \in \{0.80,0.85,0.90,0.92,0.94,0.96,0.98,1.00\}$ and $\mathrm{NFE} \in \{8,16,32,64,128,256,512,1024\}$. For temperature frontiers we use $T \in \{0.80,0.82,\dots,1.10\}$ over the same NFE grid. The predictor-corrector experiments are run on OWT with the confidence-based corrector described in the paper, sweeping the number of corrector steps $M \in \intset{1}{5}$ and the number of parallel updates $k \in \intset{1}{16}$ for $\mathrm{NFE} \in \{64,128,256,512,1024\}$.

The training time for each OWT run is approximately 1000 H100 hours. For LM1B, it is approximately 250 H100 hours. Each generative frontier evaluations compute time is approximately 1 hour on a H100 GPU.

\begin{table}[t]
\centering
\captionsetup{font=small}

\begin{minipage}[t]{0.48\linewidth}
\centering
\caption{Modeling choices for considered methods.}
\label{tab:exp-algos}
\resizebox{\linewidth}{!}{%
\begin{tabular}{lll}
\toprule
Method family & Training loss & Prediction target \\
\midrule
UDM  & cross-entropy or ELBO & denoiser or leave-one-out \\
Max-coupling UDM & ELBO & denoiser or leave-one-out \\
MDM & cross-entropy or ELBO & denoiser \\
AUDM & ELBO & denoiser \\
ReAUDM & ELBO & denoiser \\
\bottomrule
\end{tabular}
}
\end{minipage}
\hfill
\begin{minipage}[t]{0.48\linewidth}
\centering
\caption{Sudoku experiment setup and sweep ranges. 
}
\label{tab:sudoku-exp}
\resizebox{\linewidth}{!}{%
\begin{tabular}{ll}
\toprule
Quantity & Value \\
\midrule
Task & $9 \times 9$ Sudoku completion \\
Training split & 50k puzzles \\
Model & DiT, about 6M parameters \\
Learning-rate sweep & $\{10^{-4}, 3{\times}10^{-4}, 5{\times}10^{-4}, 7{\times}10^{-4}, 10^{-3}\}$ \\
NFE sweep & $\{16, 32, 64, 128\}$ \\
Evaluation & $100$ repeats with $1000$ validation puzzles \\
Metric & solve rate, reported as Sudoku accuracy \\
\bottomrule
\end{tabular}
}
\end{minipage}

\end{table}

\subsection{Additional results}

\paragraph{Temperature and top-$p$ sampling} We report the temperature frontier across NFEs $\{16, 128, 512\}$ in \Cref{fig:parameterization_temperature}. The full top-$p$ frontiers across all NFEs are reported in \Cref{fig:parameterization-nucleus-appendix}. They confirm the same trend as in the main text: for uniform diffusion, the leave-one-out parameterization is consistently preferable in the low-entropy regime. For the maximal-coupling variants the difference between the two parameterizations is smaller but shows a similar trend.

\begin{figure}[t]
\centering
    \includegraphics[width=0.7\linewidth]{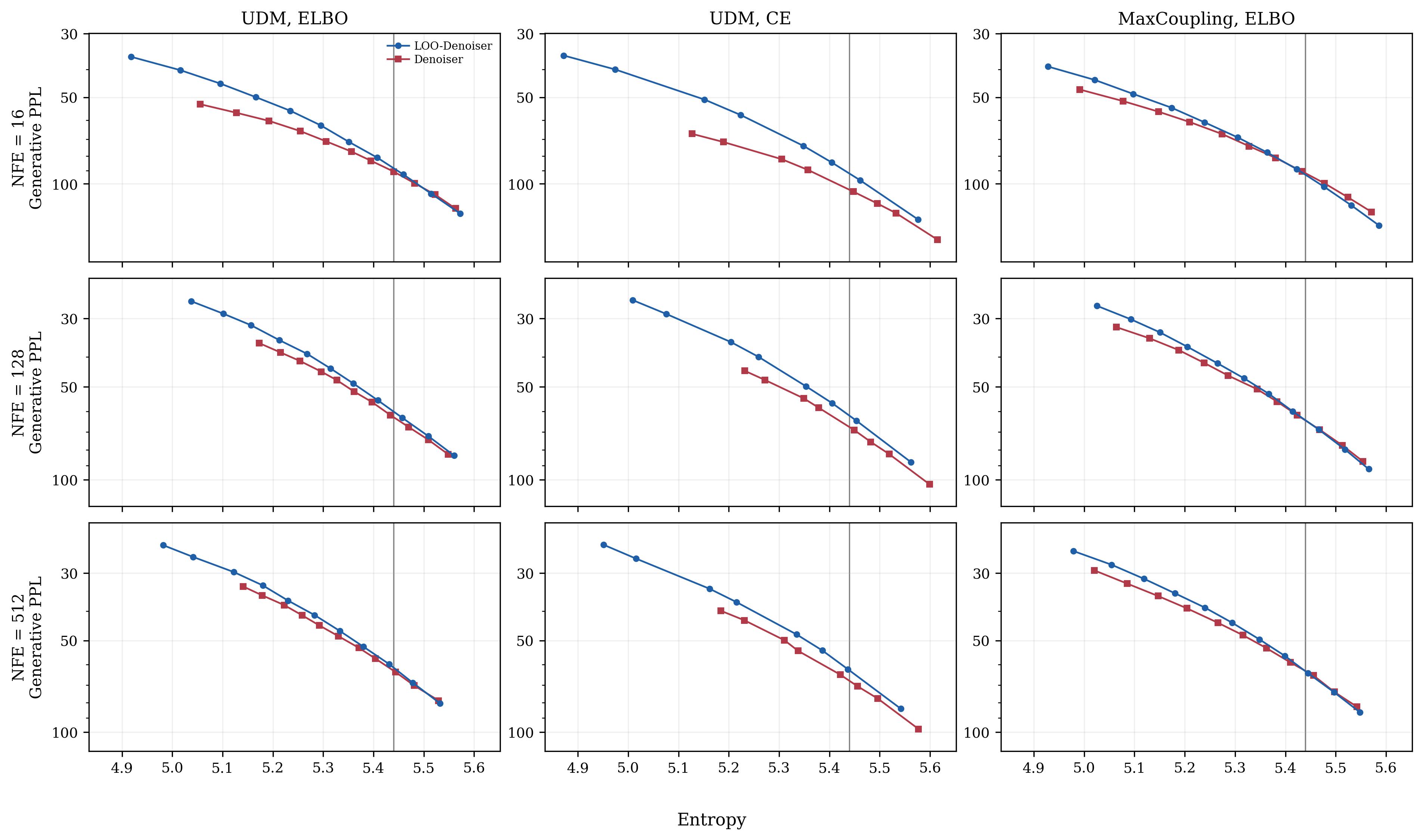}
\caption{Generative frontiers with varying temperature across NFEs for the denoiser and leave-one-out parameterizations.}
\label{fig:parameterization_temperature}
\end{figure}

\begin{figure}[t]
    \centering
    \includegraphics[width=1.0\textwidth]{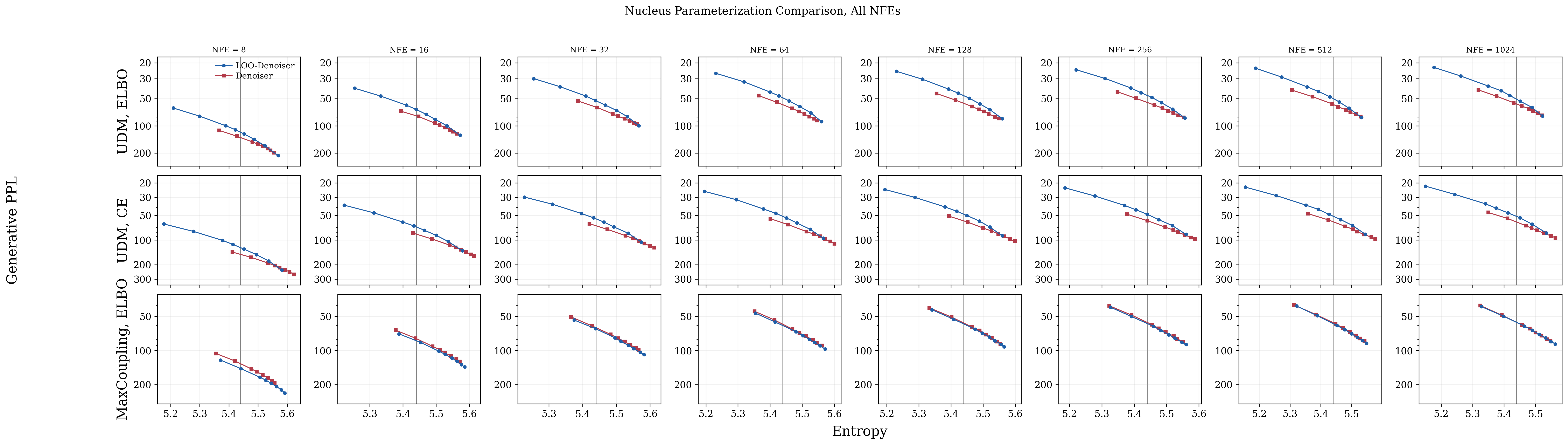}
    \caption{Top-$p$ frontiers across all NFEs for the denoiser and leave-one-out parameterizations.}
    \label{fig:parameterization-nucleus-appendix}
\end{figure}

Applying nucleus sampling after converting a denoiser prediction to the leave-one-out representation also improves the frontier, as shown in \Cref{fig:nucleus-application-frontier}. The converted denoiser closes a substantial part of the gap to the model trained directly with the leave-one-out target, while remaining clearly better than applying top-$p$ to the denoiser itself.

The predictor-corrector comparison with filtered top-$p$ sampling is shown in \Cref{fig:corrector-vs-nucleus}. At 128 and 256 NFEs, the two approaches are broadly competitive, while at 512 and 1024 NFEs several predictor-corrector configurations move to a strictly better region of the frontier.

\paragraph{Performance comparison with AUDM}
Finally, \Cref{fig:audm-udlm-frontier-appendix} reports the full frontier comparison for AUDM, resampled AUDM, UDM, and MDM. The two AUDM variants track each other closely and remain competitive throughout the NFE range. They outperform MDM and the denoiser parameterization variants, which is consistent with the main-text conclusion that the absorbing-uniform construction preserves strong sampling behavior while improving likelihood.


\begin{figure}[ht!]
    \centering
    \includegraphics[width=0.95\textwidth]{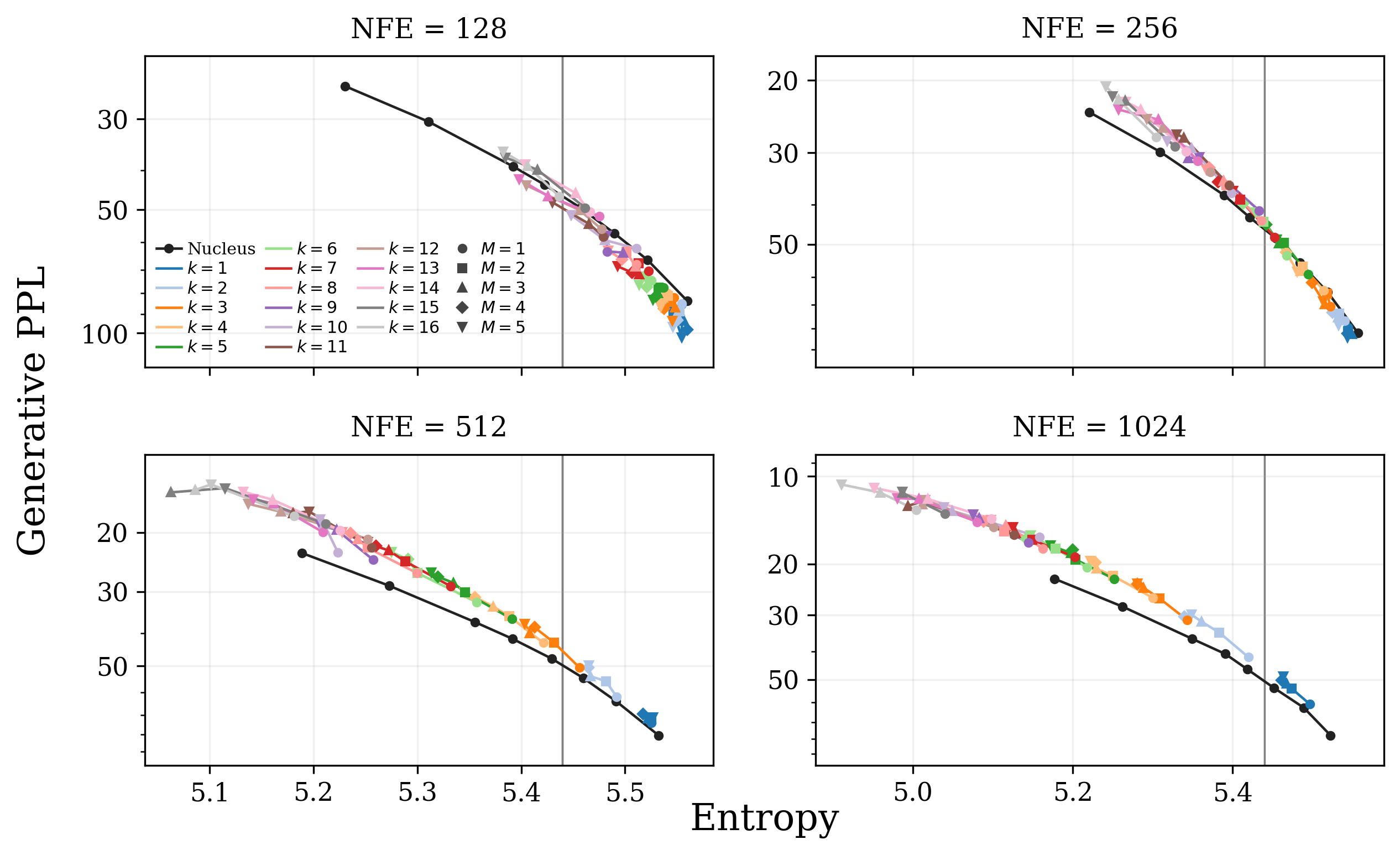}
    \caption{Comparison between predictor-corrector sampling and top-$p$ sampling.}
    \label{fig:corrector-vs-nucleus}
\end{figure}

\begin{figure}[ht!]
    \centering
    \includegraphics[width=0.95\textwidth]{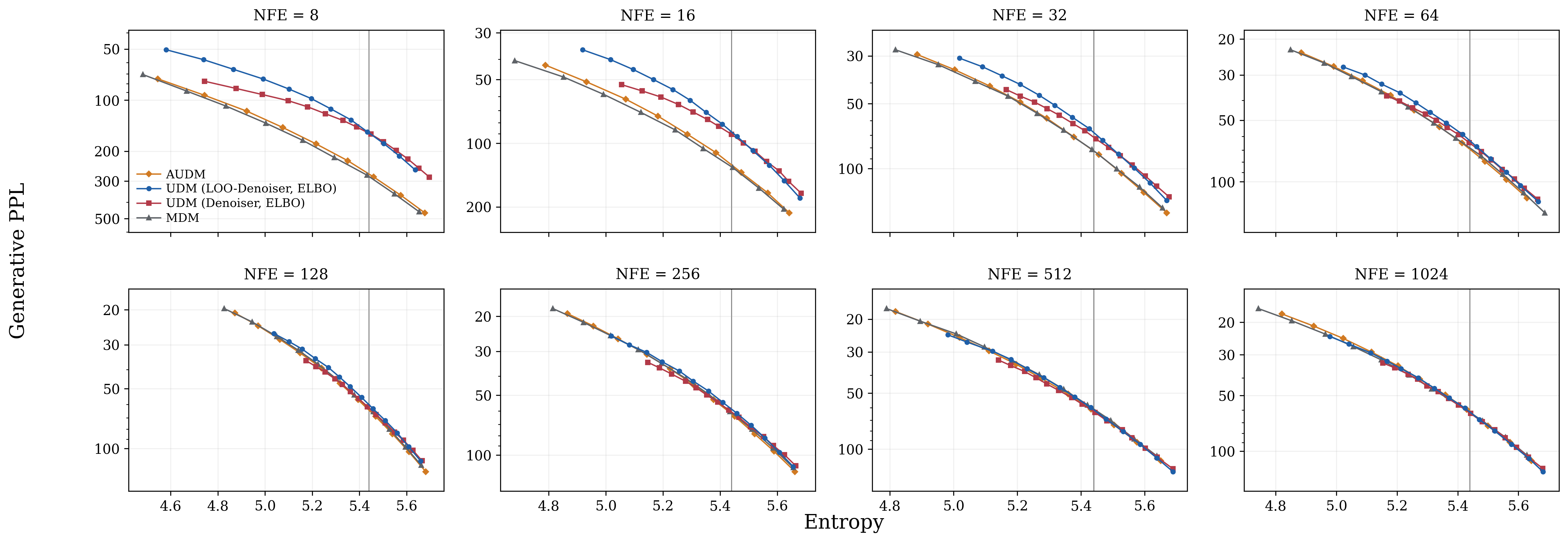}
    \caption{Full frontier comparison for AUDM, resampled AUDM, UDM, and MDM across all NFEs.}
    \label{fig:audm-udlm-frontier-appendix}
\end{figure}

\newpage

\subsection{Generated samples}


We show representative generated text from the OWT top-$p$ sweep below. For each method, we select the $\mathrm{NFE}=64$ sample whose empirical entropy is closest to $5.4$, matching the data entropy estimate used in the main text. We also compare top-$p$ sampling and predictor-corrector for the UDM ELBO / LOO-Denoiser model at $\mathrm{NFE}=512$, selecting the closest-entropy sample in each sweep.
\label{app:owt-generated-samples}

\input{appendix/generated_samples}
\clearpage

%% file: appendix/generated_samples.tex
\noindent\textbf{Top-$p$ sampling examples ($\mathrm{NFE}=64$).}

\vspace{0.35em}

\noindent\begin{mdframed}[linecolor=black!20,backgroundcolor=gray!4,linewidth=0.3pt,innerleftmargin=6pt,innerrightmargin=6pt,innertopmargin=5pt,innerbottommargin=5pt]
\textbf{UDM CE / LOO-Denoiser}\hfill {\small Entropy $=5.418$, Gen-PPL $=47.0$}

\smallskip
{\small\itshape
because they still play on the same team, rather than each other. Having met the slower contenders, those who think they are the only the fastest have fast. Miguel from the start to start of the race looked good with a decent first finish in the field but struggles in the first 10, as he crosses \#13, trailing at ten, and starts off at \#13 limping for about 21 seconds.

\medskip
By 1:20, he is out. He eventually breaks back, however, with a string of misses throughout the race before slamming home on laps 19 to knock out yet another winner, Miguel. In a video provided by Ahmed, Victor Sane of the five-man pack also stated that he has opted not to meet the \#13 Chachea after running only 492 with him in lap 2.}
\end{mdframed}

\vspace{0.35em}

\noindent\begin{mdframed}[linecolor=black!20,backgroundcolor=gray!4,linewidth=0.3pt,innerleftmargin=6pt,innerrightmargin=6pt,innertopmargin=5pt,innerbottommargin=5pt]
\textbf{UDM CE / Denoiser}\hfill {\small Entropy $=5.401$, Gen-PPL $=54.8$}

\smallskip
{\small\itshape
the current installment of the series, Macrossi and showmaster would give us the two Init-List volumes, with the fundamental decision in place: collecting the six on the first issue and those other five in. Update: in the next time, he no longer acknowledges Nina and his planet, Nina, barely mentioning the source of ``contract...'' stories.

\medskip
Those insights from the side that were all straightened for me. TWI's...I think And...yeah yeah probably that means a lot. Now, let's take for a look what we see the comics by Robert Sottarzio and Blake Nagle, and as well as comic book ``from the crew'' Miles Gerhart and Mike Campbell. Speaking of, those guys: Pat Bubuman and Fluffy Rainplie.}
\end{mdframed}

\vspace{0.35em}

\noindent\begin{mdframed}[linecolor=black!20,backgroundcolor=gray!4,linewidth=0.3pt,innerleftmargin=6pt,innerrightmargin=6pt,innertopmargin=5pt,innerbottommargin=5pt]
\textbf{UDM ELBO / Denoiser}\hfill {\small Entropy $=5.421$, Gen-PPL $=54.8$}

\smallskip
{\small\itshape
the current directions of what to expect out of the show. It would get much larger bigger. And then we would have to make sure of it. Things to expect would be forced Alan to begin with puberty. Update: at the next episode, he no longer seems to be finally starting with puberty, but at the next remaining episode of one. What follows is an intriguing glimpse into basic goals that they all had in the past.

\medskip
The key are those who feel keen that watching these shows will probably reveal what is good for them, such as pimping for its sponsor or profiting from money, which is good for themselves. And than that though, might as well keep the people out of this.}
\end{mdframed}

\vspace{0.35em}

\noindent\begin{mdframed}[linecolor=black!20,backgroundcolor=gray!4,linewidth=0.3pt,innerleftmargin=6pt,innerrightmargin=6pt,innertopmargin=5pt,innerbottommargin=5pt]
\textbf{UDM ELBO / LOO-Denoiser}\hfill {\small Entropy $=5.400$, Gen-PPL $=42.2$}

\smallskip
{\small\itshape
the current state of the series as one of my favorites. It would be much easier to see this story evolve to get to 2021, and it shows Pats' POV on 9/11 and the other five timelines. Update: Alan tells it to himself he will ``move to his planet with a black cube when the sequence is set to one.'' Ok, so happy to say this is something that makes all the jokes just disappear.

\medskip
Yes, I feel bad that the relatively short storytelling has taken too far in the direction of the dust-covering action ending. I love all the tension and action of this show, and, for some reason, might as well watch the hell out of the ending because it's much deeper than that.}
\end{mdframed}

\vspace{0.35em}

\noindent\begin{mdframed}[linecolor=black!20,backgroundcolor=gray!4,linewidth=0.3pt,innerleftmargin=6pt,innerrightmargin=6pt,innertopmargin=5pt,innerbottommargin=5pt]
\textbf{AUDM}\hfill {\small Entropy $=5.439$, Gen-PPL $=53.9$}

\smallskip
{\small\itshape
the current condition of the series as to its ratings. Some critics also point to the legacy of Michaels in Michaels hands, and question his decision to leave the still on-air show running ``the other way around.'' Sills, however, have no recent backlash against Michaels. Whether with the network.

\medskip
Whether the failure of the show, the artistic intent of the show, or their reputation is something that they all share, we should not expect such a backlash. It bears noting that the failure of the Murgons franchise was one of the goals of Michaels earlier. Sports-fors often say ``It's good for the teeth,'' but that's true when key figures in the sports enterprise are replaced. We see it this way: Viewers assume they will have to part ways.}
\end{mdframed}

\vspace{0.35em}

\noindent\begin{mdframed}[linecolor=black!20,backgroundcolor=gray!4,linewidth=0.3pt,innerleftmargin=6pt,innerrightmargin=6pt,innertopmargin=5pt,innerbottommargin=5pt]
\textbf{Resampled AUDM}\hfill {\small Entropy $=5.446$, Gen-PPL $=53.6$}

\smallskip
{\small\itshape
Simply put, the author notes and feedback tend to always be anonymous, not multiauthored. The higher your feedback, the less it will be ``no reading more'' and ``how do I have more time,'' etc. Being anonymous does not play into the reader character and feedback, however. I've spent more time reading what reviews say than watching what I read.

\medskip
I believe it's not bad, because the text fits into itself, so that's why I can't read all. Get the text card Every web site has a text card: the main page, mailing list, talkback list, the comments, thumbs/crumbs up, reviews etc. I generally do more with 50 cards, no one reads them when you want post reviews, but they're doing that most of the time, so I rely on it.}
\end{mdframed}

\noindent\textbf{UDM ELBO / LOO-Denoiser comparison ($\mathrm{NFE}=512$).}

\vspace{0.35em}

\noindent\begin{mdframed}[linecolor=black!20,backgroundcolor=gray!4,linewidth=0.3pt,innerleftmargin=6pt,innerrightmargin=6pt,innertopmargin=5pt,innerbottommargin=5pt]
\textbf{Top-$p$ sampling}\hfill {\small Entropy $=5.392$, Gen-PPL $=41.5$}

\smallskip
{\small\itshape
average 7.8 assists. That's where he went from getting high to major in the '90s. When it was season's end, though, Westbrook brought Seth Jones in. He didn't buy himself a chance to run with that. ``I think we had a lot of matchups, and his matchup looked good,'' he says. ``He also made a lot of adjustments in his playing style. I think Boris (Green) and Kemba Walker are getting better and they both make a lot of plays. But he's also made his game look good. We had guys like David West, but my team is losing guys. I'd be afraid to run into that.''

\medskip
Oh, it's midseason, but Westbrook will have to lead this new roster with Hassan Whiteside, whom the Thunder played with this past year as his Big Sam. Aldridge talks about a three-team rotation based on his Thibs. ``They're from San Jose, they play at Portland, San Antonio, it's a rivalry there, I think,'' he says. ``That's the biggest. Their team here is from California. They scored 110 field goals at the 12,000 throw line, they converted on 90. They're learning a new style. That's the big one. Part of it is just being able to incorporate the play and the tendencies and last year, the wins. You want to give the win a building block for a rotation.''}
\end{mdframed}

\vspace{0.35em}

\noindent\begin{mdframed}[linecolor=black!20,backgroundcolor=gray!4,linewidth=0.3pt,innerleftmargin=6pt,innerrightmargin=6pt,innertopmargin=5pt,innerbottommargin=5pt]
\textbf{Predictor-corrector}\hfill {\small Entropy $=5.403$, Gen-PPL $=37.5$}

\smallskip
{\small\itshape
But the Windows operating system is quickly falling into relative obscurity and/or complete obscurity. In fact it has come to be called WinTV Meets Windows. I was shocked at that in my eyes. And the WinTV Framework? Jonathan (Jackson's nickname) is a developer himself. He created it. He wrote, all the architecture and components along with all of the servers and apps running it on Windows. Moreover, Jonathan is a senior developer of the C++, so he is responsible for providing the Windows Store API for the software.

\medskip
4: 1/1\% Speedtest list is to help it's team experience the software as it is. Moreover, it ensures that it has everything it needs to build a new system, during development. Even if you want to try different components and projects, the whole code works. 3: Please write tests every time again. As with building the system, a simple, simple user interface, it needs very many tests. If you have more than 1\% of such a test (CLS), then the speedtest list is also very useful. It shows how far the code needs to go if you include penetration testing (TPT). So in case you'd like to make a quick test and even then 1-2 in the long run, the new code is only possible.}
\end{mdframed}